%% file: main.tex
\pdfoutput = 1
\documentclass[11pt]{article}
\usepackage{fullpage}
\usepackage[english]{babel}
\usepackage[utf8x]{inputenc}
\usepackage[T1]{fontenc}
\usepackage{amsmath}
\usepackage{amssymb}
\usepackage{amsthm}
\usepackage{bbm}
\usepackage{bbold}
\usepackage{parskip}
\usepackage{thm-restate}
\usepackage{booktabs}
\usepackage{array}
\usepackage{graphicx}
\usepackage{subcaption}
\usepackage{placeins}

\usepackage{hyperref}
\hypersetup{
	colorlinks=true,
	linkcolor=blue!70!black,
	citecolor=blue!70!black,
	urlcolor=blue!70!black
}
\usepackage{cleveref}
\usepackage{graphicx}
\input{macros.tex}
\SetArgSty{textup}

\usepackage{mathtools}

\title{Revisiting the Provable-Auditable Privacy Gap of DP-SGD}
\author{
Saloni Modi\thanks{University of Texas at Austin, \texttt{saloni.a.modi@utexas.edu}}
\and
Srivi Balaji\thanks{University of Texas at Austin, \texttt{srivibalaji@utexas.edu}. Srivi and Yusong contributed equally.}
\and
Yusong Zhu\thanks{University of Texas at Austin, \texttt{zhuys@utexas.edu}. Srivi and Yusong contributed equally.}
\and
Gautam Kamath\thanks{University of Waterloo and Vector Institute, \texttt{g@csail.mit.edu}. Supported by a Canada CIFAR AI Chair, an NSERC Discovery Grant, and an Ontario Early Researcher Award.
}
\and
Kevin Tian\thanks{University of Texas at Austin, \texttt{kjtian@cs.utexas.edu}}}
\date{}

\allowdisplaybreaks
\begin{document}
\pagenumbering{gobble}
\maketitle
\begin{abstract}
Differential privacy (DP) has traditionally been used to provide theoretical upper bounds on an algorithm's stability to changing its training data. In modern private machine learning applications, achieving strong tradeoffs between utility and theoretical privacy is challenging, and thus one may optimistically hope that existing theoretical privacy analyses are loose. Recent work on \emph{privacy auditing} has adopted a dual viewpoint, instead lower bounding the true privacy of an algorithm by constructing empirical distinguishing events. The auditing literature has thus far yielded a pessimistic outlook on the looseness of theoretical privacy bounds for DP-SGD, the de facto private training method in modern ML, as nearly-matching empirical lower bounds have been achieved under various threat models \cite{NasrHSBTJCT23, AnnamalaiC24, CebereBP25}. 

In this work, we propose the empirical privacy lower bound of an algorithm as a concrete metric to optimize for, complementary to the theoretical upper bound. We give a lightweight defense framework that generically augments optimization methods in the ML pipeline to have significantly-improved empirical privacy on standard benchmarks. Moreover, we show that our framework comes at \emph{no theoretical privacy cost} when augmenting DP-SGD, unlike
previously-proposed defenses against membership inference attacks. We evaluate our defense against a broad range of audit constructions, models, and datasets to demonstrate its flexibility.
\end{abstract}

\thispagestyle{empty}
\newpage
\tableofcontents
\thispagestyle{empty}
\newpage
\pagenumbering{arabic}

\input{intro.tex}
\input{overview.tex}
\input{proof.tex}
\input{private.tex}

\section*{Acknowledgments}

KT thanks Jonathan Ullman and Florian Tram\`er for helpful conversations during this project's conception. We thank Milad Nasr for clarifying conversations regarding the prior work \cite{NasrHSBTJCT23}, and both Milad Nasr and Thomas Steinke for providing feedback on drafts of this paper. We are grateful to the Texas Advanced Computing Center (TACC) and the UT Austin Center for Generative AI for providing the computing resources used in this project.

\bibliographystyle{alpha}
\bibliography{ref}

\newpage
\appendix

\input{lbmethod.tex}
\input{appendix_intro}
\input{disparate_impact}
\input{ablations.tex}

\end{document}

%% file: macros.tex
\usepackage[lined,boxed,ruled,norelsize,linesnumbered]{algorithm2e}
\hypersetup{
	colorlinks=true,
	linkcolor=blue!70!black,
	citecolor=blue!70!black,
	urlcolor=blue!70!black
}

\usepackage{xcolor}
\definecolor{burntorange}{rgb}{0.8, 0.33, 0.0}

\newtheorem{theorem}{Theorem}
\newtheorem{lemma}{Lemma}

\newtheorem{definition}{Definition}
\newtheorem{corollary}{Corollary}

\newcommand{\defeq}{:=}

\newcommand{\norm}[1]{\left\lVert#1\right\rVert}

\newcommand{\eps}{\epsilon}
\newcommand{\lam}{\lambda}
\newcommand{\sig}{\sigma}

\newcommand{\R}{\mathbb{R}}

\newcommand{\N}{\mathbb{N}}

\newcommand{\half}{\frac{1}{2}}

\newcommand{\ind}{\mathbb{I}}
\newcommand{\E}{\mathbb{E}}

\newcommand{\Nor}{\mathcal{N}}

\newcommand{\dd}{\textup{d}}

\newcommand{\Par}[1]{\left(#1\right)}
\newcommand{\Brack}[1]{\left[#1\right]}
\newcommand{\Brace}[1]{\left\{#1\right\}}

\newcommand{\alg}{\mathcal{A}}

\newcommand{\event}{\calE}

\newcommand{\calA}{\mathcal{A}}
\newcommand{\calB}{\mathcal{B}}

\newcommand{\calD}{\mathcal{D}}
\newcommand{\calE}{\mathcal{E}}
\newcommand{\calF}{\mathcal{F}}

\newcommand{\calM}{\mathcal{M}}
\newcommand{\calN}{\mathcal{N}}

\newcommand{\calS}{\mathcal{S}}
\newcommand{\calT}{\mathcal{T}}
\newcommand{\calU}{\mathcal{U}}

\newcommand{\calX}{\mathcal{X}}

\newcommand{\codeStyle}[1]{{\bfseries #1} }

\newcommand{\codeReturn}{\codeStyle{Return:}}

\newcommand{\epslb}{\eps_{\textup{lb}}}
\newcommand{\epsub}{\eps_{\textup{ub}}}
\newcommand{\elb}{\mathsf{EmpiricalDPLB}}

\newcommand{\halpha}{\hat{\alpha}}
\newcommand{\hbeta}{\hat{\beta}}
\newcommand{\simiid}{\sim_{\textup{i.i.d.}}}

\newcommand{\Beta}{\textup{Beta}}

\newcommand{\bbN}{{\mathbb{N}}}

\newcommand{\brac}[1]{\left\{#1\right\}}
\newcommand{\mech}{\calA}
\newcommand{\DPSGDP}{\mathsf{DPSGDPoisson}}
\newcommand{\DPSGDS}{\mathsf{DPSGDShuffle}}
\newcommand{\FDPSGDS}{\mathsf{FilteredDPSGDShuffle}}
\newcommand{\FDPSGDP}{\mathsf{FilteredDPSGDPoisson}}
\newcommand{\SelectTop}{\mathsf{SelectTop}}
\newcommand{\clip}{\textup{clip}}
\newcommand{\Bern}{\textup{Bern}}
\renewcommand{\epsilon}{\varepsilon}
\renewcommand{\eps}{\varepsilon}

%% file: intro.tex
\section{Introduction}

It is well-documented that machine learning (ML) models are vulnerable to a variety of privacy threats \cite{HomerSRDTMPSNC08,ShokriSSS17,CarliniTWJHLRBSEOR21}. To combat these threats,
differential privacy (DP)~\cite{DworkMNS06} has seen wide adoption as a rigorous notion of data privacy.
This notion is typically phrased as an upper bound (Definition~\ref{def:dp}): for an algorithm $\calA: \calS^* \to \Omega$, DP bounds the value of
\begin{equation}\label{eq:true_privacy}\eps^\star(\delta) \defeq \log\Par{\sup_{\substack{\calD, \calD' \in \calS^* \\ \text{neighboring}}} \sup_{\calE \subseteq \Omega} \frac{\Pr[\calA(\calD) \in \calE] - \delta}{\Pr[\calA(\calD') \in \calE]}}.\end{equation}
The quantity $\eps^\star(\delta)$ can be viewed as the ``true privacy'' parameter at failure probability $\delta$ of $\calA$, as it is realized by (or is the limit of) concrete neighboring $\calD, \calD'$,\footnote{Different works in the privacy auditing literature use slightly different definitions of neighboring. In this work, we use the ``add/remove'' definition as in \cite{MironovTZ19, NasrHSBTJCT23, CebereBP25}; see Section~\ref{ssec:prelims} for more discussion.} and an ``audit'' $\calE \in \Omega$. Various techniques have been developed to upper bound $\eps^\star$. For example, modern DP-SGD~\cite{SongCS13,BassilyST14,AbadiCGMMTZ16} privacy analyses proceed using (advanced) composition \cite{DworkRV10, DworkR14}, sometimes on an alternative metric such as R\'enyi differential privacy (RDP) \cite{AbadiCGMMTZ16, Mironov17}. We informally use $\epsub(\delta)$ to denote the tightest upper bound on \eqref{eq:true_privacy} yielded by existing accounting techniques.

There are various ways that privacy accounting can be loose. A simple example is lossiness in converting between privacy definitions. A more subtle example, concerning the \emph{threat model}, asks: what should we view as the output $\Omega$ of an algorithm, such as DP-SGD, that (adaptively) generates a sequence of private models $M_1, M_2, \ldots, M_T$? One way to define $\Omega$ is as a product over $T$ model realizations. This is implicit in standard (composition-based) DP analyses, for which upper bounds on $\eps^\star(\delta)$ hold even if $\calE$ depends on all intermediate models. However, it is often more realistic to assume that only the final $M_T$ is published (a.k.a.\ the ``hidden state model'' \cite{YeS22, CebereBP25}), or even that our access to $M_T$ is limited (e.g., query access rather than full model weights), in which case the feasible audits $\calE$ in \eqref{eq:true_privacy} may be severely restricted. Indeed, recent theory has illustrated scenarios where changing the threat model can (significantly) amplify privacy \cite{YeS22, AltschulerT22}.

\subsection{Motivation}\label{ssec:motivation_audit}

There is good reason to hope that current DP accounting significantly overestimates the true privacy $\eps^\star$. It is well-documented that, to retain acceptable utility, current DP optimizers must compromise by yielding fairly large privacy parameters. 
For example, Figure 1(a) of~\cite{DeBHSB22} summarizes recent state-of-the-art accuracy of private training on CIFAR-10. 
With a strong provable privacy $\epsub \approx 1$, the best accuracy was only $60\%$~\cite{TramerB21}, compared with a non-private state of the art of $99\%$+.
On the other hand achieving a modest accuracy of $> 80\%$ required a much larger $\epsub \approx 8$. 
The worst-case implications of such a large $\epsub$ are dubious, as $\exp(8) \approx 3000$. 
Optimistically, could it be that the true privacy parameter is much smaller than current theory predicts?

A closely-related question is: how can a practitioner better reason about the true privacy $\eps^\star$, taking the threat model into account? This is the key conceptual question that our work addresses. Motivated by work on privacy auditing  \cite{JagielskiUO20, NasrSTPC21}, we propose using the following, more operational, ``empirical privacy'' quantity as an explicit metric for algorithm design:
\begin{equation}\label{eq:epslb}\epslb(\delta) \defeq \log\Par{\sup_{\calD \in \calS^n} \sup_{\textup{canary } c} \sup_{\calE \in \calT} \max\Par{\frac{\Pr[\calA(\calD^c) \in \calE] - \delta}{\Pr[\calA(\calD) \in \calE]}, \frac{\Pr[\calA(\calD) \in \calE] - \delta}{\Pr[\calA(\calD^c) \in \calE]}}}.\end{equation}
In \eqref{eq:epslb}, $c \in \calS$ is a ``canary'' sample, $\calT$ denotes a family of audits, and the goal of $\calE$ is to detect the presence of $c$ in a dataset $\calD$, where $\calD^c \defeq \calD \cup \{c\}$. Under the strongest possible $\calD$, $c$, and $\calE$, the quantities \eqref{eq:true_privacy} and \eqref{eq:epslb} are the same. In practice, we cannot always find this optimal triple (e.g., the range $\Omega$ is typically infinite), and hence our estimated $\epslb$ is only a lower bound for $\eps^\star$. Nonetheless, by optimizing over a comprehensive family of $(\calD, c, \calE)$, one can hope to obtain nearly-tight \emph{certifiable lower bounds} on $\eps^\star(\delta)$. We review the standard auditing methodology in Section~\ref{ssec:measure_lb}. A typical setup involves a mislabeled or out-of-distribution canary $c$, and an audit that thresholds the loss value of the final model when labeling $c$. Intuitively, if $c$ was memorized, then it will be correctly classified and hence will incur much lower loss than had the model never seen $c$.

Existing work on auditing DP-SGD, arguably the most well-studied private training algorithm in ML, is surprisingly pessimistic: nearly-tight audits (e.g., certified $\epslb$ within a 30\% factor of $\epsub$) are achievable in a broad range of threat models.\footnote{In Section~\ref{ssec:threat}, we formally define and review different threat models that we consider from  the literature.}  This was first demonstrated by \cite{NasrSTPC21, NasrHSBTJCT23} in the strong ``gradient space canary'' threat model, where the attacker can arbitrarily control \emph{gradients} associated with a sample. Existing $\epsub$ bounds hold in this threat model, but potentially overestimate the true \eqref{eq:true_privacy}, in real-world settings where gradient space attacks are infeasible \cite{CebereBP25, BoglioniLIW25}. Later, \cite{AnnamalaiC24, CebereBP25} obtained similar nearly-tight audits under ``hidden state'' threat models where only the final model is released; particularly, \cite{AnnamalaiC24} focuses on a more realistic ``input space canary'' threat model, where the canary is a fixed planted sample image.

\subsection{Our contributions}\label{ssec:results}

Our work makes two main contributions towards more accurately measuring the true privacy \eqref{eq:true_privacy}, and designing algorithms with improved empirical privacy under realistic threat models.

\begin{itemize}
    \item We initiate the study of the \emph{auditable privacy lower bound} $\epslb$ as a formal metric for algorithms to target. This line of research should be viewed as complementary to existing algorithmic work in DP, which focuses on designing methods that enjoy a smaller $\epsub$. 
    \item We develop an algorithmic framework for significantly improving $\epslb$, as estimated by state-of-the-art audits in the hidden state, input space threat model (cf.\ Section~\ref{ssec:threat}). Our framework is a lightweight filtering-based wrapper (cf.\ Section~\ref{ssec:framework}), carefully designed so that, when applied to a wide range of popular DP-SGD variants, it retains the same provable $\epsub$.
\end{itemize}

Empirical privacy measures are brittle \cite{SongM21, CarliniCNSTT21, CarliniJZPTT22, aerni2024evaluations}, so reasoning about their guarantees takes care. In Section~\ref{ssec:discussion}, we discuss the broader implications of $\epslb$ \eqref{eq:epslb} as a definition. We observe here that targeting \eqref{eq:epslb} opens an algorithmic design space that must go beyond DP-SGD. Indeed, despite a decade of research, no algorithm has dethroned the $\epsub$-utility tradeoff of (Poisson subsampled) DP-SGD. This algorithm is also tightly auditable \cite{NasrSTPC21, NasrHSBTJCT23, AnnamalaiC24, CebereBP25}, with the notable caveat that existing audits have almost entirely targeted DP-SGD. In Section~\ref{sec:private}, we demonstrate that auditing our filtered variant with a comprehensive suite of tests from the literature yields $\epslb \approx 0$, at a negligible utility drop from its unfiltered counterpart. 

Although other strategies have been proposed for decreasing the auditability of ML algorithms (discussed in Section~\ref{ssec:prior}), our framework has a key qualitative advantage: it comes at \emph{no loss} to the provable $\epsub$ of standard private optimizers. This is the focus of Section~\ref{sec:proof}, and making our proof flexible to a wide range of filters and optimizers is a main technical contribution of our work. For example, while our framework is related to robust statistics-inspired filters from the data poisoning literature \cite{TranLM18, HayaseKSO21}, these defenses break down when the fraction of poisoned points is too small (e.g., $< 1\%$). This is at odds with the auditing regime, where even a single canary can yield nearly-tight audits. Moreover, the filtering rules of \cite{TranLM18, HayaseKSO21} were based on aggregate statistics (e.g., PCA). This makes it challenging to reason about privacy in datasets undergoing adaptive filtering, which intuitively could break the neighboring property over time.

Our privacy proof bypasses both issues, by restricting our filtering rules to be based on what we call \emph{sample signatures}, defined in Section~\ref{ssec:framework}. Intuitively, these rules score individual samples in a way that only depends on other samples through a joint privately-trained model (so aggregate statistics such as the sample covariance of gradients cannot be used). We give a simple yet flexible proof that neighboring datasets that are filtered using sample signatures remain neighboring (Lemma~\ref{lem:induct_neighbor_bound}), and that adaptive composition proofs are robust to this interleaved filtering (Lemma~\ref{lem:interleave_rdp}). Putting together these observations, we prove a privacy guarantee for our filtered DP-SGD method in Corollary~\ref{cor:privacy_filter} that matches Poisson-subsampled DP-SGD (Theorem~\ref{thm:dp-sgd-rdp-account}). Section~\ref{ssec:input} then demonstrates that sample signatures are powerful enough to reliably filter input space canaries, substantially improving the Pareto frontier of utility-$\epslb$ tradeoffs against this test suite.

Of course, the metric \eqref{eq:epslb} is only meaningful if the audit suite is sufficiently representative of real attacks. Section~\ref{sec:private} and Appendix~\ref{app:alt_audits} provide systematic evaluations against all proposed input space canaries we are aware of, and Section~\ref{ssec:defense_aware} constructs canaries deliberately targeting our framework. Across all constructions, our framework yields improved, and sometimes negligible, $\epslb$. 

In Section~\ref{ssec:gradient}, we conclude by critically examining potential vulnerabilities of our framework under a stronger (and often, unrealistic) \emph{gradient space} threat model. Notably, the trainer can often enforce a particular threat model: for example, by computing their own gradients, and not publishing checkpoints, one can restrict training to the hidden state, input space setting.

\subsection{Discussion}\label{ssec:discussion}

A wide body of work \cite{SongM21, CarliniCNSTT21, CarliniJZPTT22, aerni2024evaluations} cautions against the potential of empirical privacy evaluations to be misleading. With this context in mind, here we provide some perspective on the value of $\epslb$ as an evaluation metric, and other considerations for the community.

\textbf{The implications of $\epslb$ as a metric.} Our $\epslb$ is not a formal security guarantee against \emph{all possible privacy attacks}, and should be interpreted by users of our framework as measuring the vulnerability of an ML model against a \emph{current suite of representative membership inference attacks}. These attacks' efficacy is measured by the same statistical distinguishing task as in \eqref{eq:epslb}, and thus we believe it is reasonable to base our $\epslb$ test suite $\calT$ off of state-of-the-art audits. Our empirical privacy definition has the added benefit of being threat model-aware, i.e., the test suite in \eqref{eq:epslb} can be restricted depending on the attacks that are information-theoretically feasible.

As the test suite $\calT$ in our definition \eqref{eq:epslb} expands with future work, $\epslb$ becomes a ``moving target'' by nature. Notably, any newly-developed auditing strategies can only strengthen the value of this target. Our hope is that formalizing $\epslb$ as an evaluation metric initiates a back-and-forth between auditing researchers aiming to expose privacy vulnerabilities (by expanding $\calT$), and algorithm designers aiming to withstand existing audits. Indeed, our evaluations in Section~\ref{sec:private} reveal a blind spot of current input space audits: either our filtered algorithms truly have smaller $\eps^\star$ than the theory predicts, or new tests must be developed to more accurately measure $\eps^\star$. Without proposing $\epslb$ as an evaluation metric, or assessing audits beyond applying them to a single algorithm (DP-SGD), there would be no impetus to develop these stronger tests. 

Finally, we believe our work has the potential to motivate advances in DP theory. If a phenomenon is robustly identified  (e.g., the community cannot develop strong audits against an algorithm in a particular threat model), this creates an opportunity for theorists to formalize this guarantee.

\textbf{Impact on non-canary samples.} Prior works also caution against other potential pitfalls of empirically-private methods, such as their exposure of other samples beyond the planted canary (the ``privacy onion'' effect \cite{CarliniJZPTT22}), and their potential for disparate impact on minority ``benign outlier'' groups. Regarding the former issue, Appendix~\ref{app:non_canary_audit} demonstrates that even if we ignore the effects of multiple discovery, our filtered defense does not expose any non-canary samples at a rate comparable to the standard canary audit on non-filtered DP-SGD. Regarding the latter, the general tension between privacy and fairness is a known phenomenon \cite{BagdasaryanPS19, UniyalNKSKMT21} that is out of our scope to address in full; nevertheless, Appendix~\ref{app:disparate_impact} shows that the disparate impact of our framework is mild when minority groups are sufficiently large, or relatively few points are filtered.

\subsection{Prior work}\label{ssec:prior}

To our knowledge, no work has previously proposed the goal of lowering the auditable $\epslb$ \eqref{eq:epslb} without compromising either the provable $\epsub$ or utility of a private learning algorithm. Here, we survey two lines of works that have aimed to balance two of these three criteria.

\textbf{Empirically-private algorithms.} Empirical privacy is often framed through membership inference, and indeed, a line of recent work from the security literature designs heuristics meant to protect against membership inference attacks (MIAs) \cite{NasrSH18, JiaSBZG19, ShejwalkarH21, ChenYF22, TangMSSNHM22, ChenP24}. These defenses aim at directly improving tradeoffs between the auditable $\epslb$ and the model utility, but do not consider provable privacy. Several of these works modify the training procedure or loss directly in a highly dataset-dependent way \cite{NasrSH18, JiaSBZG19, ChenYF22, TangMSSNHM22, ChenP24}, which makes it difficult to perform theoretical privacy accounting. Other strategies are based on distillation from a teacher model \cite{ShejwalkarH21} and are therefore less flexible to setups where such a model is not available. We complement this literature by designing a heuristic defense with a privacy proof (Section \ref{sec:proof}), and show that it maintains a (lossless) $\epsub$-utility tradeoff, competitive with the state-of-the-art. This modification makes measuring gaps between $\epslb$ and $\epsub$ meaningful.

\textbf{Provable alternatives to DP-SGD.} The literature has long considered tradeoffs between provable $\epsub$ and utility to be the gold standard for evaluating private learning algorithms. Alternatives to DP-SGD have been proposed, including PATE \cite{PapernotAEGT17} (which uses public data), and augmentations of DP-SGD \cite{AndrewTMR21, DuLCCH21, TangSL24} that adaptively estimate hyperparameters. 

For simplicity, we focus on (subsampled variants of) DP-SGD as our baseline private optimizer of choice, as it remains an industry gold standard in privately training modern ML models \cite{DeBHSB22}. While other private optimizers have gained traction recently as competitive alternatives, such as DP-FTRL \cite{KairouzMSTTX21,PillutlaUCDGHKMMRST25}, work on the best way to audit models trained via alternative optimizers, a prerequisite to producing meaningful $\epslb$ estimates, is comparatively sparse. We do note that our framework does not currently apply to DP-FTRL, because that algorithm is not permutation-invariant (see discussion in Section~\ref{sec:proof}). We leave constructing a broader, provably private, filtering framework beyond permutation-invariant algorithms as an interesting open problem.

We also mention that \cite{FeldmanZ21} proposed a conceptually similar use of filtering to design improved provably-private algorithms. Their approach (and in particular, their privacy proof) is quite different from ours, as they require filtering decisions to be based solely on individual statistics, whereas our framework uses rank-based filtering depending on outlier scores. Moreover, their goal is to improve utility-privacy tradeoffs, whereas our work is largely motivated by the third axis of empirical privacy.
Finally, while we are not aware of prior works whose goal is explicitly to improve $\epsub$-$\epslb$-utility tradeoffs, \cite{JagielskiUO20} observed that certain modifications to training (e.g., randomized initialization) result in better $\epslb$, in a similar spirit to our motivation.

%% file: overview.tex
\section{Overview}

\subsection{Preliminaries}\label{ssec:prelims}

\textbf{Privacy.} For a sample space $\calS$, we use $\calS^* \defeq \bigcup_{n \in \N} \calS^n$ to denote an arbitrary-sized dataset from $\calS$. We provide the following standard definition of differential privacy.

\begin{definition}[($\eps, \delta$)-DP]\label{def:dp}
Let $(\eps, \delta) \in \R_{\ge 0} \times [0, 1]$. We say algorithm $\calA: \calS^* \to \Omega$ satisfies $(\eps, \delta)$-differential privacy (or, is $(\eps, \delta)$-DP) if for all events $\event \subseteq \Omega$, and all neighboring $\calD, \calD' \in \calS^*$,
\[\Pr\Brack{\calA(\calD) \in \event} \le \exp(\eps) \Pr\Brack{\calA(\calD') \in \event} + \delta.\]
\end{definition}

One subtlety in Definition~\ref{def:dp} is that the literature uses slightly different definitions of \emph{neighboring}; these notions of DP often imply each other, but only up to a constant factor loss. As our paper's methods are based on Poisson-subsampled DP-SGD, the de facto method for training provably-private ML models,\footnote{For example, privacy accounting via Poisson subsampling is the standard implementation in popular privacy packages such as Pytorch Opacus \cite{YousefpourSSTPMNGBZCM21} and Tensorflow Privacy \cite{tensorflow_privacy}.} we use the definition admitting the tightest analyses of this algorithm, for fair comparison. Concretely, we use the add/remove definition of DP, where datasets $\calD \in \calS^n$, $\calD' \in \calS^{m}$ are \emph{neighboring} if $|m - n| = 1$, and $|\calD \setminus \calD'| + |\calD' \setminus \calD| = 1$, i.e., they differ in one element. 

The privacy analysis of DP-SGD is often performed using the following alternative definition.

\begin{definition}[RDP]
    \label{def:rdp}
    Let $n\in \bbN$, $\alpha > 1$, and $\rho \ge 0$. We say that $\calA: \calS^* \to \Omega$ satisfies $(\alpha, \rho)$-RDP if for all neighboring datasets $\calD, \calD' \in \calS^*$,
    \begin{equation}
        \label{eq:def_rdp}
        D_\alpha(\calA(\calD) \Vert \calA(\calD^\prime)) \le \rho,
    \end{equation}
    where $D_\alpha$ is the R\'enyi divergence and defined as $D_\alpha(P\|Q) = \frac 1 {\alpha - 1}\log \E_{\omega\sim P}[(\frac{P(\omega)}{Q(\omega)})^{\alpha - 1}]$.
\end{definition}
Lemma~\ref{lem:rdp_to_dp} gives a conversion from RDP to DP. For more properties of RDP, we refer to \cite{Mironov17}.
\begin{lemma}[\cite{Mironov17}]
\label{lem:rdp_to_dp}
If $\calM$ satisfies $(\alpha, \eps)$-RDP, it satisfies $(\eps + \frac{1}{\alpha -1}\log \frac{1}{\delta}, \delta)$-DP for all $\delta \in (0, 1)$.
\end{lemma}
\textbf{DP-SGD.} Algorithm~\ref{alg:dp-sgd-subsample} outlines a typical Poisson-subsampled DP-SGD framework for training a model with parameters $\theta$ by minimizing the empirical loss $\frac 1 n \sum_{(x, y) \in D} \ell(\theta; x, y)$, where $\calD$ is the dataset of $n$ examples. At each step of DP-SGD, we first (randomly) subsample a minibatch; then we compute the per-sample gradient $\nabla_\theta \ell(\theta; x, y)$ for each $(x, y)$ within the batch $\calB$ and clip their $\ell_2$ norms; lastly, we average the clipped batch and add Gaussian noise to ensure privacy.

\begin{algorithm}[ht]
\DontPrintSemicolon
    \caption{$\DPSGDP(\calD, \theta_0, q, \eta, C, \sig, T)$}
    \label{alg:dp-sgd-subsample}
\KwIn{Dataset $\calD$, model initialization $\theta_0$, sample rate $q$, step size $\eta$, 
clip bound $C$, noise multiplier $\sigma$, step count $T$}
\KwOut{Final model parameters $\theta$}
$\theta \leftarrow \theta_0$\;
\For{$t\in [T]$}{
    $\calB \gets \brac{}$\;
    \For{$(x, y) \in \calD$}{
    $r \gets \textup{Bern}(q)$\;
    \If{$r = 1$}{$\calB \gets \calB \cup (x, y)$}
    }
    $\xi \sim \mathcal{N}\left(0,\left(\frac{\sigma C}{qn}\right)^2 I\right)$ \tcp*{Spherical Gaussian with standard deviation $\frac{\sigma C}{qn}$}\;
    $\theta \leftarrow \theta - \eta\left(
    \frac{1}{qn}\sum_{(x, y)\in\mathcal{B}}\text{clip}(\nabla_\theta \ell(\theta;x, y), C) + \xi\right)$ \tcp*{$\text{clip}(v, C) = v \cdot \min(1, \frac{C}{\norm{v}_2})$}\;
}
\Return{$\theta$}
\end{algorithm}

Computing the tightest upper bound of overall privacy cost of training a model is extensively studied in DP. Various accounting methods have been proposed to tackle this problem, e.g., moments accountant in \cite{AbadiCGMMTZ16}, RDP in \cite{MironovTZ19, zhu2019poission}, and PLD in \cite{koskela2019computingtightdifferentialprivacy}. We recall the accounting method based on RDP in Theorem~\ref{thm:dp-sgd-rdp-account}, which gives a clean theoretical bound, implemented by popular privacy packages \cite{YousefpourSSTPMNGBZCM21, tensorflow_privacy}. We mention that PLD \cite{koskela2019computingtightdifferentialprivacy} can achieve tighter $\epsub$ estimates in some cases than Theorem~\ref{thm:dp-sgd-rdp-account}, but uses a numerical method that does not yield an explicit bound; our defense would enjoy an equally-tighter upper bound, as discussed in Section~\ref{sec:proof}.

\begin{theorem}[\cite{MironovTZ19}]
    \label{thm:dp-sgd-rdp-account}
    If $q < \frac{1}{5}, \sigma > 4$, and $\alpha \le  \frac{\frac{1}{2}\sigma^2L^2 - \ln 5 - 2\ln \sigma}{L + \ln(q\alpha) + (2\sigma^2)^{-1}}$ where $L \defeq \log(1 + \frac 1 {q(\alpha - 1)})$,
    Algorithm~\ref{alg:dp-sgd-subsample} is $(\alpha, T\cdot \frac{2q^2\alpha}{\sigma})$-RDP. It is also $(T\cdot \frac{2q^2\alpha}{\sigma} + \frac{1}{\alpha - 1}\log \frac{1}{\delta}, \delta)$-DP for all $\delta \in (0, 1)$.
\end{theorem}

Due to practical considerations, alternatives to Algorithm~\ref{alg:dp-sgd-subsample} are sometimes used, despite their weaker provable privacy-utility tradeoffs. For example, a common heuristic is to divide the dataset into randomly-shuffled minibatches of size $m$ (see Algorithm~\ref{alg:dp-sgd-shuffle}) for $N$ epochs; due to implementation details, this strategy often results in improved practical performance. Despite Theorem~\ref{thm:dp-sgd-rdp-account} not formally applying to this alternative algorithm, it is a common practice to apply its bound with the ``in-expectation'' equivalent parameters $q = \frac m n$, $T = N \cdot \lceil \frac n m \rceil$ to conclude a heuristic bound \cite{DeBHSB22, Ponomareva_2023, ChuaGKKMSZ24, LebedaRKS25}. We sidestep this issue by directly using Algorithm~\ref{alg:dp-sgd-subsample} in our experiments, with its provable guarantee in Theorem~\ref{thm:dp-sgd-rdp-account}. In Section~\ref{sec:proof}, we discuss how our privacy accounting for our filtered optimizers applies to Algorithm~\ref{alg:dp-sgd-shuffle} as well. We also explain how our analysis framework extends to other recently-proposed alternatives, such as \cite{ChuaGHKKLMSZ25}.

\subsection{Threat models}\label{ssec:threat}

The threat model specifies what information is known to the attacker/auditor and what strategies are allowed when constructing attacks. The literature on privacy auditing often uses ``white-box'' and ``black-box'' inconsistently as descriptors of threat models. 

To sidestep this ambiguity, in this paper, we describe threat models along two main axes.
The first axis is auditor access: in a \emph{full state} threat model, the auditor has access to internal training information such as model parameters throughout training (e.g., intermediate checkpoints), while in a \emph{hidden state} setting, intermediate training states are not revealed and the auditor only observes the final trained model (i.e., its weights). The second axis is what the auditor is allowed to perturb when designing an attack: \emph{input space} attacks construct canaries by modifying the raw input (e.g., pixel values), while \emph{gradient space} attacks operate by targeting the gradient signal a canary would produce during training (rather than directly optimizing the input itself). 

Recent works \cite{CebereBP25, BoglioniLIW25} suggest that the hidden state, input space threat model most accurately reflects the practice of private ML. While this setting is our focus, in Section~\ref{sec:private}, we critically examine our framework's vulnerability to gradient space attacks accessing model history (Section~\ref{ssec:gradient}), and the reproducibility of these attacks in input space settings (Section~\ref{ssec:defense_aware}).

\subsection{Framework}\label{ssec:framework}

As discussed in Section~\ref{ssec:results}, our defense is a wrapper that periodically filters training samples based on certain scoring rules. We provide pseudocode for the Poisson subsampled DP-SGD variant of our defense in Algorithm~\ref{alg:dp_sgd_with_defense_ss}; we show how to bound its privacy in Corollary~\ref{cor:privacy_filter}.

\begin{algorithm}[htbp]
\DontPrintSemicolon
\caption{$\FDPSGDP(\calD, \theta_0, q, \eta, C, \sigma, T, m, f, k, \texttt{scope})$}
\label{alg:dp_sgd_with_defense_ss}
\KwIn{Dataset $\calD$, model initialization $\theta_0$, sample rate $q$, step size $\eta$, clip bound $C$, noise multiplier $\sigma$, step count $T$, epoch parameter $m$, sample signature $f$, selection parameter $k$, flag $\texttt{scope} \in \{\texttt{global}, \texttt{class}\}$}
\KwOut{Final model parameters $\theta$}
$(\calF, S, \theta) \gets (\emptyset, \{0\}^{|\calD|}, \theta_0)$\;
\For{$t\in [T]$}{
    $\calB \gets \{\}$\;
    \For{$(x, y) \in \calD$}{
    $r \gets \Bern(q)$\;
    \If{$r = 1$}{
    $\calB \gets \calB \cup (x, y)$
    }
    }
    \ForEach{$(x_i, y_i) \in \calB$}{
    $g_i \gets \begin{cases}
    \clip(\nabla_\theta \ell(\theta; x_i, y_i), C) & i \not\in \calF  \\
    0 & i \in \calF
    \end{cases}$\;\label{line:ascent_poisson}
    $S[i] \gets f_\theta(x_i, y_i)$\;\label{line:score_poisson}
    }
    $\xi \sim \calN(0,(\frac{\sig C}{qn})^2 I)$\;
    $\theta \leftarrow \theta - \eta(
    \frac{1}{qn}\sum_{(x_i, y_i)\in\calB}g_i + \xi)$\;
    \If{$m \mid t$}{\tcp*{Wait $m$ minibatch iterations before filtering.}
    $\calF \gets \calF \cup \SelectTop(S, \calD, \calF, k, \texttt{scope})$\;
    }
}
\Return{$\theta$}
\end{algorithm}

\textbf{Sample signatures.}
A \textit{sample signature} (or \textit{scoring function}) maps samples to real numbers, based solely on information derived from that sample and privatized model parameters. Crucially, a sample signature does not depend on any other samples, except through the learned private model, in contrast to previous ``batch'' filtering rules proposed in the data poisoning literature \cite{TranLM18, HayaseKSO21}, which depended on e.g., the mean or covariance of the dataset gradients (a batch statistic). Use of sample signatures is essential for preserving privacy, an argument we formalize in Section~\ref{sec:proof} via a generalized composition theorem robust to sample signature filtering.

% =========================
% Algorithm 3
% =========================
\begin{algorithm}[htbp]
\caption{$\SelectTop(S, \calD, \calF, k, \texttt{scope})$}
\label{alg:top_select}
\DontPrintSemicolon
\KwIn{Computed sample signatures $S$, dataset $\calD$, filtered set $\mathcal{F} \subseteq \calD$, count $k \in \N$, mode $\texttt{scope} \in \{\texttt{global}, \texttt{class}\}$}
\KwOut{Selected samples $\mathcal{T}$}
$\mathcal{U} \gets \calD \setminus \calF$\;

\eIf{$\texttt{scope} = \texttt{global}$}{
  \Return{$k$ indices $i \in \calU$ with largest $S[i]$}\;
}{
  \ForEach{class $c$ represented in $\calU$}{\label{line:class_start}
    $\calT_c \gets $ $k$ indices $i \in \calU_c \defeq \{i \in \mathcal{U} \mid y_i = c\}$ with largest $S[i]$\;
  }
  \Return{$\bigcup_{\text{class } c \text{ represented in } \calU} \calT_c$}\label{line:class_end}
}
\end{algorithm}

We denote a sample signature by $f_\theta(x, y)$ when scoring a sample $(x, y)$, where the subscript indicates implicit dependence on the model. 
We consider several choices of $f_\theta$; a few examples follow.

\begin{itemize}
    \item Gradient norm: $f_\theta(x, y) \defeq \norm{\nabla_\theta \ell(\theta; x, y)}_p$ where, e.g., $p \in \{1, 2, \infty\}$. High scores suggest an influential sample in the trajectory of the model updates.
    \item Prediction margin: Let the (softmaxed) logits of example $(x, y)$ with respect to model $\theta$ be denoted $p_\theta(x)[c]$ for a class $c$. Then we define $f_\theta(x, y) \defeq -(p_\theta(x)[y] - \max_{c \neq y} p_\theta(x)[c])$. High scores suggest a sample with an unusual or misclassified label.
\end{itemize}

In Appendix~\ref{app:score}, we introduce several additional sample signatures, and provide an ablation study on the ability of these various scores to  detect samples with high risk of memorization.

One major benefit of our outlier detection strategy is that it can detect outliers with significantly less signal than the ``batch'' update rules used by \cite{TranLM18, HayaseKSO21}, which state they require a constant fraction (e.g., $1\%$) of poisoned data to be effective. In DP auditing settings, this distinction is crucial, as memorization of even a single canary can nullify privacy bounds.

\textbf{Design rationale.} We make two key design choices in our implementation of Algorithm~\ref{alg:dp_sgd_with_defense_ss}. The first  is to only allow for sample signature-based filtering (as opposed to ``batch'' scoring rules). This enables using our generalized composition theorems from Section~\ref{sec:proof} to argue that our modifications do not affect the privacy of Algorithm~\ref{alg:dp_sgd_with_defense_ss} compared to Algorithm~\ref{alg:dp-sgd-subsample}. 
Regarding the ``\texttt{scope}'' parameter in Algorithm~\ref{alg:dp_sgd_with_defense_ss}, gradient statistics vary between classes. Global filtering risks flagging representative samples from classes with naturally higher gradient norms as outliers. Per-class filtering (Algorithm~\ref{alg:top_select}, Lines~\ref{line:class_start} to~\ref{line:class_end}) standardizes the detection threshold by comparing each sample only against others in its class, ensuring that outlier detection is relative to class-specific distributions.

\subsection{Measuring empirical privacy}\label{ssec:measure_lb}

We next review LiRA-style membership inference attacks for DP auditing \cite{CarliniCNSTT21, NasrHSBTJCT23}.

\textbf{Trade-off functions.} Consider two arbitrary distributions $P, Q$ supported on the same space $\Omega$, and a \emph{decision rule} $\phi: \Omega \to [0, 1]$. Intuitively, $\phi$ acts as a distinguisher between $P$ and $Q$, where $\phi(\omega) = 0$ can be taken as guessing that $\omega$ was sampled from $P$, and $\phi(\omega) = 1$ indicates a sample from $Q$ (fractional values thus represent uncertainty). We accordingly define
\begin{align*}
\alpha_{P, Q}(\phi) \defeq \E_{\omega \sim P}[\phi(\omega)] ,\quad \beta_{P, Q}(\phi) \defeq \E_{\omega \sim Q}[1 - \phi(\omega)]
\end{align*}
to be the type I (FPR) and type II (FNR) error rates of $\phi$. 
For each FPR $\alpha \in [0, 1]$, we can consider the optimal test which minimizes the FNR $\beta \in [0, 1]$ at the specified $\alpha$ level, formalized as follows.

\begin{definition}[Trade-off function]\label{def:tradeoff}
For two distributions $P, Q$ supported on $\Omega$, we define the \emph{trade-off function} $T_{P, Q}: [0, 1] \to [0, 1]$ by $T_{P, Q}(\alpha) \defeq \inf_{\substack{\phi: \Omega \to [0, 1], \alpha_{P, Q}(\phi) \le \alpha}} \beta_{P, Q}(\phi)$.

More generally, we say that $f: [0, 1] \to [0, 1]$ is a trade-off function if $f = T_{P, Q}$ for some $P, Q$.
\end{definition}

The Neyman-Pearson lemma characterizes $\phi$ achieving the infimum in $T_{P, Q}$ as thresholded likelihood ratio tests (Theorem A.1, \cite{DongRS22}). Next, we state some basic properties of trade-off functions.

\begin{lemma}[Proposition 2.2, \cite{DongRS22}]\label{lem:tradeoff}
$f: [0, 1] \to [0, 1]$ is a trade-off function iff it is convex, non-increasing, and satisfies $f(\alpha) \le 1 - \alpha$ for all $\alpha \in [0, 1]$.
\end{lemma}

The first property in Lemma~\ref{lem:tradeoff} follows because given tests $\phi, \phi'$ with FPRs $\alpha, \alpha'$, and mixing parameter $\lam \in [0, 1]$, we can always define a test that outputs $\phi$ with probability $\lam$ and $\phi'$ with probability $1-\lam$. The second property is because the constraint set on $\phi$ is only larger as $\alpha$ grows. 

Finally, the last property is because we can always simultaneously obtain FPR $\alpha$ and FNR $1 - \alpha$, for any $\alpha \in [0, 1]$, by ignoring the sample and uniformly setting $\phi = \alpha$. If $f(\alpha) = 1 - \alpha$ uniformly, then $P$ and $Q$ are undistinguishable by any test, i.e., $P = Q$ almost everywhere.

\textbf{$f$-DP and $(\eps, \delta)$-DP.} The definition of $f$-DP is naturally motivated by Definitions~\ref{def:dp} and~\ref{def:tradeoff}.

\begin{definition}[$f$-DP]\label{def:fdp}
Let $f: [0, 1] \to [0, 1]$ be a trade-off function. We say that a mechanism $\calM: \calS^n \to \Omega$ satisfies $f$-differential privacy (or, is $f$-DP) if for all neighboring $\calD, \calD' \in \calS^n$,
\begin{equation}\label{eq:fdp}T_{\alg(\calD), \alg(\calD')}(\alpha) \ge f(\alpha) \text{ for all } \alpha \in [0, 1].\end{equation}
\end{definition}

A higher trade-off function corresponds to distributions being difficult to distinguish. Definition~\ref{def:fdp} captures this intuition by positing that every pair $\alg(\calD)$, $\alg(\calD')$ induced by neighboring datasets is at least hard to distinguish as specified by $f$. For example, if $f(\alpha) = 1 - \alpha$ uniformly, then any $f$-DP $\alg$ has $\alg(\calD) = \alg(\calD')$ almost everywhere, so its outcome is independent from the dataset it is trained on with probability $1$. We call $\alg$ that satisfies such an $f$-DP \emph{perfectly private}.

Importantly, \cite{WassermanZ10, DongRS22} observed the following equivalence between Definitions~\ref{def:dp} and~\ref{def:fdp}.

\begin{lemma}[Proposition 2.12, \cite{DongRS22}]\label{lem:equiv_dp}
Let $f: [0, 1] \to [0, 1]$ be a symmetric trade-off function. Then a randomized algorithm $\alg$ is $f$-DP iff for all $\delta \in (0, 1)$, $\alg$ is $(\eps(\delta), \delta)$-DP, where
\begin{equation}\label{eq:eps_convert}\eps(\delta) \defeq \max\Brace{0, \max_{\alpha \in [0, 1]} \log\Par{\frac{1 - f(\alpha) - \delta}{\alpha}}},\end{equation}
\end{lemma}
Note that if we increase $f$ pointwise, applying \eqref{eq:eps_convert} leads to smaller privacy parameters.

\textbf{Pipeline.} Following \cite{NasrSTPC21, NasrHSBTJCT23}, we can certify privacy lower bounds via Lemma~\ref{lem:equiv_dp}. Recall that the true DP function $\eps^\star$ \eqref{eq:true_privacy} of an algorithm $\alg$ are the largest parameters attainable via a set of tests (corresponding to a dataset-canary-audit triple \eqref{eq:epslb} at each $\delta \in (0, 1)$). Analogously, one can define the true $f$-DP curve to be the pointwise largest $f$ so that $\alg$ is $f$-DP. Recalling \eqref{eq:fdp}, each point on this true curve is again realized by a dataset, canary, and audit $\phi: \Omega \to [0, 1]$.

In practice we can only realize a subset of all possible dataset-canary-audit triples, and thus the best we can do is to upper bound the true $f$-DP curve as tightly as possible. 
Each test is parameterized by neighboring datasets $(\calD, \calD')$ and a candidate test $\phi: \Omega \to [0, 1]$ where $\Omega$ is the output space of $\alg$. Suppose we ascertain that for this particular triple,
\[\E_{\omega \sim \alg(\calD)}[\phi(\omega)] \le \hat{\alpha},\quad \E_{\omega \sim \alg(\calD')}[1 - \phi(\omega)] \le \hat{\beta}.\]
A standard way to obtain $(\hat{\alpha}, \hat{\beta})$ is to generate holdout copies of $\alg(\calD)$ and $\alg(\calD')$, estimate the two expectations above from empirical averages, and apply a Clopper-Pearson correction.
This certifies $f^\star(\alpha) \le \hat{\beta}$ for all $\alpha \ge \hat{\alpha}$, where $f^\star$ denotes the true $f$-DP curve for $\alg$. Any such triple $(\phi, \calD, \calD')$ yields an empirical upper bound on $f^\star$, which translates to a lower bound on $\eps^\star$ via \eqref{eq:eps_convert}. 

A well-documented difficulty in carrying out this methodology is that unless the number of holdout copies used is (very) large, the Clopper-Pearson correction can dramatically loosen the audit. As a result, prior works have adopted several heuristic approximations (multiple hypothesis testing, and use of Gaussian DP curves) that can inflate the estimated $\epslb$, while sometimes sacrificing formal correctness. We discuss these heuristics at length in Appendix~\ref{app:lbmethod}, and provide  counterexamples flagging situations where their guarantees do not formally apply. 

We make explicit that in the rest of the paper (up until Appendix~\ref{app:other_eps}), all of our reported $\epslb$ use \emph{both} the multiple hypothesis testing and GDP extrapolation heuristics described in Appendix~\ref{app:lbmethod}. This reporting is consistent with prior privacy auditing works advertising tight or nearly-tight audits, which also used these heuristics \cite{NasrHSBTJCT23, AnnamalaiC24, CebereBP25}. Notably, by using both heuristics, the reported $\epslb$ metrics of our algorithm (as well as all other defenses) is as unfavorable as possible.

For completeness, in Appendix~\ref{app:other_eps} we report \emph{all} of the estimated $\epslb$ for DP-SGD and our filtered variants, both with and without these heuristics, and using different sample splits. Our filtered DP-SGD variant consistently achieves a lower $\epslb$ than the unfiltered implementation against input space, hidden state audits across all reporting methods (formal or not).

%% file: proof.tex
\section{Privacy of Filtering via Sample Signatures}\label{sec:proof}

Here we provide a framework to analyze the privacy of augmenting a DP algorithm with filtering. The unfiltered variants of our algorithms, Algorithm~\ref{alg:dp-sgd-subsample} and~\ref{alg:dp-sgd-shuffle}, have two useful properties. First, their privacy is based on composition of individual steps' (R\'enyi) privacy. Second, the algorithms are permutation-invariant, in that the same sequence of random operations (in distribution) is performed even if we arbitrarily shuffle the data. These properties hold for many DP algorithms; the latter specifically is natural if assumptions placed on dataset samples are uniform. Our framework for arguing about privacy with interleaved filtering steps carefully uses these two properties. 

\subsection{Analysis}\label{ssec:privproof}

The first piece of our framework is a generalized privacy composition theorem. Intuitively, it states that if before we apply privacy composition, we can arbitrarily replace a pair of datasets with new datasets with the same number of neighbors, this cannot change the final privacy bound.

\begin{lemma}[Interleaved R\'enyi composition]\label{lem:interleave_rdp}
Let $\calA_1: \calS^* \to \Omega_1$ be $(\alpha, \rho_1)$-RDP, and let $\calA_2: \calS^* \times \Omega_1 \to \Omega_2$ be such that $\calA_2(\cdot, \omega_1)$ is $(\alpha, \rho_2)$-RDP for any $\omega_1 \in \Omega_1$. Also, let $\calF: \calS^* \times \Omega_1 \to \calS^*$ be a \emph{filter} such that for any neighboring $\calD, \calD' \in \calS^*$, we always have that $\calF(\calD, \omega_1)$ and $\calF(\calD', \omega_1)$ are neighboring for any $\omega_1 \in \Omega_1$. Then, the mechanism $\calA$ that first samples $\omega_1 \gets \calA_1(\calD)$ and then samples $\omega_2 \gets \calA_2(\calF(\calD, \omega_1), \omega_1)$ (i.e., applies $\calF$ to $\calD$ before $\calA_2$) is $(\alpha, \rho_1 + \rho_2)$-RDP.
\end{lemma}
\begin{proof}
We closely follow the proof of Proposition 1, \cite{Mironov17}. Fix a pair of outcomes $\omega = (\omega_1, \omega_2) \in \Omega_1 \times \Omega_2$. For simplicity, we overload notation so that $\mech_1(\calD)(\omega_1)$ refers to the density of outcomes according to $\calA_1(\calD)$, and define $\mech_2(\calD, \omega_1)(\omega_2)$, $\mech(\calD)(\omega)$ similarly. Finally, we use the shorthand $\calD_{\omega_1}$ to mean $\calF(\calD, \omega_1)$ for all $\omega_1 \in \Omega_1$. We have the desired
\begin{gather*}
\exp\Par{(\alpha - 1)D_\alpha\Par{\mech(\calD) \| \mech(\calD')}} = \int_{\omega \in \Omega_1 \times \Omega_2} \Par{\frac{\mech(\calD)(\omega)}{\mech(\calD')(\omega)}}^\alpha \mech(\calD')(\omega)\dd \omega \\
= \int_{\omega_1 \in \Omega_1} \Par{\frac{\calA_1(\calD)(\omega_1)}{\calA_1(\calD')(\omega_1)}}^\alpha  \Par{\int_{\omega_2 \in \Omega_2} \Par{\frac{\calA_2(\calD_{\omega_1}, \omega_1)(\omega_2)}{\calA_2(\calD'_{\omega_1}, \omega_1)(\omega_2)}}^\alpha \mech_2(\calD'_{\omega_1}, \omega_1)(\omega_2)\dd \omega_2} \calA_1(\calD')(\omega_1) \dd \omega_1 \\
\le \exp((\alpha-1)\rho_2)\int_{\omega_1 \in \Omega_1} \Par{\frac{\calA_1(\calD)(\omega_1)}{\calA_1(\calD')(\omega_1)}}^\alpha  \calA_1(\calD')(\omega_1) \dd \omega_1 \le \exp((\alpha - 1)(\rho_1 + \rho_2)).
\end{gather*}
The first inequality used the privacy assumption on $\mech_2(\cdot, \omega_1)$, and that $\calD_{\omega_1}$ and $\calD_{\omega_1}'$ are neighboring, for any choice of $\omega_1$. The second inequality used the privacy assumption on $\calA_1$.
\end{proof}

The strategy in Lemma~\ref{lem:interleave_rdp} is quite general, and we expect it to apply to many if not all settings where privacy composition holds. For example, it is straightforward to apply the strategy to prove generalized variants of basic and advanced composition under filtering (Theorems 3.16, 3.20, \cite{DworkR14}). This strategy also works for composition of PLD curves, see e.g., Lemma C.3 of \cite{DongRS22}, which also first conditions on the shared outcome of $\calA_1$ before applying composition with $\calA_2$.

Importantly in our applications to our filtered defense, if we are applying Lemma~\ref{lem:interleave_rdp} sequentially (e.g., $T$ times to compose $T$ DP-SGD steps), $\omega_1$ in the lemma statement encompasses the entire history of $T - 1$ prior operations, so the operations $\calF(\cdot, \omega_1)$ can depend on all previous iterations. 

The second piece of our framework is an observation that taking two neighboring datasets, scoring samples based on a coupled rule, and dropping the $k$ largest-scored samples each, cannot increase the Hamming distance between the datasets. This observation treats dropped samples as equivalent (e.g., replaced with a standardized dummy sample), regardless of their indices. This is where we require permutation-invariance: otherwise (if indices are accounted for) the distance can grow.\footnote{An illustrative example is neighboring size-$2$ datasets with scores $\{1, 2\}$ and $\{3, 2\}$. Replacing the largest score with a common dummy causes $2$-neighboring datasets, but they remain $1$-neighboring if indices are ignored.} Our proof in fact handles both main definitions of neighboring used in the DP literature, namely, the add/remove notion used in our work, as well as the replacement notion.

\begin{lemma}
    \label{lem:induct_neighbor_bound}
    Let $D_1, D_2$ be multisets of real numbers, and let $k \le \min\{|D_1|, |D_2|\}$. Let $F_1$ be $D_1$ with its $k$ largest elements removed, and define $F_2$ similarly with respect to $D_2$. Then,
    \begin{equation}\label{eq:neighbor_preserve}|D_1 \setminus D_2| + |D_2 \setminus D_1| \ge |F_1 \setminus F_2| + |F_2 \setminus F_1|. \end{equation}
\end{lemma}
\begin{proof}
Denote $s \defeq |D_1 \cap D_2|$ and $t \defeq |F_1 \cap F_2|$. Because
\[|D_1 \setminus D_2| + |D_2 \setminus D_1| = \Par{|D_1| - s} + \Par{|D_2| - s} = |D_1| + |D_2| - 2s, \]
and similarly the right-hand side of \eqref{eq:neighbor_preserve} is $|D_1| + |D_2| - 2t - 2k$, our goal in \eqref{eq:neighbor_preserve} is equivalent to showing that $t \ge s - k$. If $s \le k$ then this is clear. Assume $s > k$ henceforth.

Sort $D_1 \cap D_2$ in order: $x_1 \ge x_2 \ge \ldots \ge x_s$. Since $s > k$, then for any $j > k$, we must have $x_j \in F_1 \cap F_2$ because we remove the $k$ highest scores of each set only (coupling any copies of duplicated items removed). Thus at least $s - k$ common elements survive and $t \ge s - k$ as claimed.
\end{proof}

Notice that the condition \eqref{eq:neighbor_preserve} captures neighboring-preservation under filtering, up to assigning indices to the elements (i.e., the locations of the removed elements do not matter). This last caveat explains why our framework requires permutation-invariance of the algorithm. Indeed, under the add/remove definition of neighboring, the two sides of \eqref{eq:neighbor_preserve} respectively equal the number of add/remove operations needed to change one set into the other, before and after filtering. Similarly, under the replacement notion of DP (Definition 1, \cite{DworkMNS06}), the two sides of \eqref{eq:neighbor_preserve} are both twice the number of replacement operations needed for this transformation.

By combining Lemma~\ref{lem:interleave_rdp} and~\ref{lem:induct_neighbor_bound}, we can conclude that any permutation-invariant algorithm whose privacy analysis is based on RDP composition retains the same privacy bound under arbitrary sample signature-based filtering. The fact that filtering is performed based on sample signatures (so that the score of each sample is the same as long as the model is coupled) is critical in applying Lemma~\ref{lem:induct_neighbor_bound}. As an example application, we analyze the privacy of Algorithm~\ref{alg:dp_sgd_with_defense_ss}.

\begin{corollary}\label{cor:privacy_filter}
Algorithm~\ref{alg:dp_sgd_with_defense_ss} satisfies the same privacy bound as stated in Theorem~\ref{thm:dp-sgd-rdp-account} for Algorithm~\ref{alg:dp-sgd-subsample}.
\end{corollary}
\begin{proof}
Theorem~\ref{thm:dp-sgd-rdp-account} proceeds by applying the RDP composition theorem to individual DP-SGD steps, whose RDP is analyzed directly. We apply Lemma~\ref{lem:interleave_rdp} in place of standard RDP composition. Lemma~\ref{lem:induct_neighbor_bound} shows that every time we wish to apply Lemma~\ref{lem:interleave_rdp}, the datasets indeed satisfy the precondition of being neighboring, with one subtlety: the indices of deleted samples (which are effectively replaced by a dummy sample with zero gradient) are not guaranteed to be coupled. However, Algorithm~\ref{alg:dp_sgd_with_defense_ss} remains unchanged in distribution if after every filtering step, we randomly permute the dataset (dummy samples included). We can thus couple the random permutations between runs of Algorithm~\ref{alg:dp_sgd_with_defense_ss} on neighboring datasets, so that the dummy indices are always aligned. This coupling does not affect the privacy proof, as it preserves marginal distributions. 
\end{proof}

\subsection{Variants of DP-SGD}\label{ssec:variant}

In this section, we briefly comment on the applicability of the framework in Section~\ref{ssec:privproof} to variants of DP-SGD that see use in practice, beyond the Poisson-subsampled variant in Algorithm~\ref{alg:dp-sgd-subsample}.

First, consider the variant using randomly-shuffled minibatches (as discussed in Section~\ref{ssec:prelims}, with pseudocode in Algorithm~\ref{alg:dp-sgd-shuffle}). It is straightforward to check that, due to the random permutation at the start of every epoch (Line~\ref{line:permute}), the privacy framework of Section~\ref{ssec:privproof} applies if filtering is performed after any number of gradient steps are taken. This is true whether the privacy accounting is done using the standard (formal) accounting based on RDP of the Gaussian mechanism, or
the (heuristic) ``in-expectation equivalent'' variant of Theorem~\ref{thm:dp-sgd-rdp-account} discussed in Section~\ref{ssec:privproof}. In Algorithm~\ref{alg:dp_sgd_with_defense}, we give pseudocode of one potential implementation of a filtered variant of Algorithm~\ref{alg:dp-sgd-shuffle}.

\begin{algorithm}[ht]
    \caption{$\DPSGDS(\calD, \theta_0, m, \eta, C, \sig, N)$}
    \label{alg:dp-sgd-shuffle}
    \DontPrintSemicolon
\KwIn{Dataset $\calD$, model initialization $\theta_0$, batch size $m$, step size $\eta$, 
clip bound $C$, noise multiplier $\sigma$, epoch count $N$}
\KwOut{Final model parameters $\theta$}
$\theta \leftarrow \theta_0$\;
\For{$t\in [N]$}{
    $\{(x_i, y_i)\}_{i \in [n]} \gets$ randomly permuted ordering of elements of $\calD$\;\label{line:permute}
    $\calB_j \gets \brac{(x_i, y_i)}_{i \in [n] \mid j = \lceil \frac i m\rceil}$ for all $j \in [\lceil \frac{n}{m}\rceil]$\;
    \For{$j \in [\lceil \frac{n}{m}\rceil]$}{
    $\xi \sim \calN\Big(0,\Big(\frac{\sig C}{|\calB_j|}\Big)^2 I\Big)$\;
    $\theta \leftarrow \theta - \eta(
    \frac{1}{|\calB_j|}\sum_{(x_i, y_i)\in\calB_j}\text{clip}(\nabla_\theta \ell(\theta;x_i,y_i), C) + \xi)$\;
    }

}
\Return{$\theta$}
\end{algorithm}

\begin{algorithm}[ht]
\DontPrintSemicolon
\caption{$\FDPSGDS(\calD, \theta_0, m, \eta, C, \sigma, N, f, k, \texttt{scope})$}
\label{alg:dp_sgd_with_defense}
\KwIn{Dataset $\calD$, model initialization $\theta_0$, batch size $m$, step size $\eta$, clip bound $C$, noise multiplier $\sigma$, epoch count $N$, sample signature $f$, selection parameter $k$,  flag $\texttt{scope} \in \{\texttt{global}, \texttt{class}\}$}
\KwOut{Final model parameters $\theta$}
$(\calF, S, \theta) \gets (\emptyset, \{0\}^{|\calD|}, \theta_0)$\;
\For{$t\in [N]$}{
    $\{(x_i, y_i)\}_{i \in [n]} \gets$ randomly permuted ordering of elements of $\calD$\;
    $\calB_j \gets \brac{(x_i, y_i)}_{i \in [n] \mid j = \lceil \frac i m\rceil}$ for all $j \in [\lceil \frac{n}{m}\rceil]$\;
    \For{$j \in [\lceil \frac{n}{m}\rceil]$}{
    \ForEach{$(x_i, y_i) \in \calB_j$}{
    $g_i \gets \begin{cases}
    \clip(\nabla_\theta \ell(\theta; x_i, y_i), C) & i \not\in \calF  \\
    0 & i \in \calF
    \end{cases}$\;\label{line:ascent}
    $S[i] \gets f_\theta(x_i, y_i)$\;\label{line:score}
    }
    $\xi \sim \calN(0,(\frac{\sig C}{|\calB_j|})^2 I)$\;
    $\theta \leftarrow \theta - \eta(
    \frac{1}{|\calB_j|}\sum_{(x_i, y_i)\in\calB_j}g_i + \xi)$\;
    }
    $\calF \gets \calF \cup \SelectTop(S, \calD, \calF, k, \texttt{scope})$\;
}
\Return{$\theta$}

\end{algorithm}

Second, we observe that unless provable privacy accounting methods change substantially, our framework is likely to extend to future (permutation-invariant) variants of DP-SGD that are developed. Indeed, any algorithm that is based on an iterative optimization method, whose individual steps are history-dependent and data-dependent, presumably requires an adaptive composition property. In essentially all existing privacy frameworks, adaptive composition is proven through an argument compatible with Lemma~\ref{lem:interleave_rdp}: after conditioning on the outcome of the first mechanism, the proof no longer requires that the datasets be the same in the second mechanism.

To give an example of this phenomenon, consider the recent algorithm of \cite{ChuaGHKKLMSZ25}. This scheme is also permutation-invariant, and its privacy (as a \emph{random allocation scheme}) is established using RDP and PLD \cite{FeldmanS25, FeldmanS26}. These accounting frameworks (particularly, their proofs of adaptive composition) are compatible with Lemma~\ref{lem:interleave_rdp}, and thus a variant of \cite{ChuaGHKKLMSZ25} using sample signature-based filtering would inherit the privacy analyses of \cite{FeldmanS25, FeldmanS26}.

%% file: private.tex
\section{Privacy-Preserving Empirically Private Learning}\label{sec:private}

In this section, we overview our evaluations of our framework in the hidden state, input space threat model. We also provide an evaluation in the gradient space setting in Section~\ref{ssec:gradient}, and additional experiments in the appendix.
Code for all experiments can be found
\href{https://github.com/pineappleEnthusiast/empirical-privacy-defense}{here}.

\subsection{Experimental and audit setup}\label{ssec:setup}

In this section, we describe the experimental setup used to evaluate our defense. Then, we evaluate our defense framework against a suite of privacy audits. %Our focus in this section is specifically on examining the $\epslb$-utility tradeoffs achieved by our defense, so that we can compare against an existing defense from the MIA literature (see discussion in Section~\ref{ssec:prior}).

%\gautam{We had some of this earlier. Redundant to have it twice? }

%\textbf{Datasets.} We evaluate our defense on three standard benchmarks: MNIST (grayscale handwritten digits, 10 classes, 28$\times$28 images), CIFAR-10 (natural images, 10 classes, 32$\times$32 RGB images), and Purchase100 (tabular customer purchase histories represented as binary feature vectors, 100 classes).

\textbf{Datasets.} 
We evaluate audits and defenses on three standard supervised learning benchmarks spanning both vision and tabular domains. For image classification, we use MNIST~\cite{LeCunBBH98}, a grayscale handwritten digit dataset with 10 classes and 28×28 inputs, and CIFAR-10~\cite{Krizhevsky09}, a natural image dataset with 10 classes and 32×32 RGB inputs. We also use Purchase100~\cite{ShokriSSS17}, a tabular dataset of customer purchase histories. Each example is represented as a binary feature vector indicating the presence/absence of purchased items, with 100 class labels.

\textbf{Initializations.} 
We reuse the same model initialization across all training runs for each experiment. Fixed initialization reduces run-to-run variability, which makes it easier to attribute differences in the final model to the data (e.g., whether a canary was included) rather than to randomness in initialization. We use Xavier initialization~\cite{GlorotB10} for CNN-based models and Kaiming initialization~\cite{HeZRS15} for WideResNet models~\cite{ZagoruykoK16}. All models are trained from scratch.

\textbf{Model architectures.}
For MNIST, we use a convolutional neural network (CNN). For CIFAR-10, we test both a CNN and a WideResNet architecture to demonstrate robustness across different model families.\footnote{Our WideResNet configuration is inspired by the WRN-16-4 model of~\cite{DeBHSB22}. Their results are near SOTA, in particular, achieving 79.5\% test accuracy for CIFAR-10 at $\varepsilon=8$. Our pipeline is slightly modified from theirs (e.g., smaller batch size, no augmentation multiplicity), largely due to computational constraints. Consequently, our model achieves a slightly lower accuracy for the same setting, but we anticipate qualitatively similar findings would hold if our audit were run in their exact setup.} For Purchase100, we use a 3-layer multilayer perceptron (MLP).

\textbf{Training configuration.}
For all experiments, unless otherwise specified, we use privacy parameters $\eps$ = 10 and $\delta$ = $10^{-5}$ for DP-SGD and our composed defense. We train for $N = 100$ epochs using Poisson subsampling. We perform hyperparameter grid search to identify learning rates that maximize test accuracy: $\eta$ = 3 for MNIST and CIFAR-10, and $\eta$ = 10 for Purchase100. To disentangle improvements due to the proposed defense from improvements due to tuning, we use the same hyperparameter configurations when auditing with and without our defense.

\textbf{Defense configuration.}
For all experiments, we run our defense using the $L_\infty$ norm as the scoring function (Algorithm~\ref{alg:dp_sgd_with_defense_ss}, Line~\ref{line:score_poisson}) unless otherwise specified. Ablation studies (Appendix~\ref{app:score}) show that alternative scoring functions perform comparably. We use per-class (local) filtering rather than global filtering to align with prior work \cite{TranLM18}. We find that for blank canaries, global filtering also works if we change the scoring function (Appendix~\ref{app:global_filter}). For MNIST and CIFAR10, we filter 5 samples from each class per epoch, amounting to dropping $8.3\%$ and $10\%$ of each training set, respectively. For Purchase100, we filter 1 sample from each class per epoch, amounting to dropping $6.5\%$ of the training set. The choice of $\#$samples to drop per class per epoch was arbitrary, as long as the total $\#$samples dropped was reasonable ($\leq 10\%$).

We find that for blank canaries, our defense is robust to filtering frequency (Appendix~\ref{app:filter_frequency_ablation}) and ``bandwidth'' (the number of samples discarded, Appendix~\ref{app:bandwidth_ablation}).

\textbf{Augmentation multiplicity.}
Practitioners sometimes generate multiple augmented versions of the same sample within a single batch and average their gradients before clipping, a technique known as augmentation multiplicity. This approach is used to improve the privacy-utility tradeoff by reducing the variance of the gradient estimate for each sample. While we do not employ augmentation multiplicity in all of our experiments, our evaluation of our defense includes this setting (cf.\ Table~\ref{tab:private_tradeoffs_acc}) to ensure it remains effective against this commonly-used privacy-enhancing modification.

\textbf{Audit models and evaluation.}
We train $k$ = 400 models for each audit: 200 with the canary and 200 without. While some prior works train more models, we find that with the heuristics in Appendix~\ref{app:lbmethod} applied, 400 shadow models offers a sufficient scale to observe a substantial gap in attack success between filtered and unfiltered settings. Our alternative $\epslb$ reporting in Appendix~\ref{app:other_eps} shows that even without these heuristics, a smaller but noticeable gap persists. As validation, we also provide an MNIST audit with 1000 models (Appendix~\ref{app:trial_count}), which shows limited improvement in audit tightness. Our primary focus is the relative gap in empirical privacy with versus without the filter, so we adopt the same methodology for auditing all optimizers.

\textbf{Compute resources.}
All experiments were run on NVIDIA H200 GPUs with 96GB HBM3 memory. Individual audit jobs used between 5 and 20 GPUs and ran for 5–30 hours, with resource requirements scaling according to model architecture and dataset size. Specific job configurations (number of GPUs, hours per job, batch scripts) are available in the scripts directory of our code repository.

\textbf{Multi-canary audits.} 
For attacks involving multiple canaries, we compute each model's score as the maximum loss over all canary samples. Note that the $\epsilon_{\text{lb}}$ value reported for these attacks is a measure of $k$-group privacy, where $k$ denotes the number of canaries; while it is reasonable to expect the $k$-group privacy parameter to scale as $\approx k \times$ the $1$-group privacy parameter, one cannot formally convert $k$-group privacy into $1$-group privacy in this way.

\subsection{Input space attacks}\label{ssec:input}

Input space attacks operate at the data level by injecting carefully crafted training samples into the dataset. 
We first evaluate our defense on a singular canary across varying privacy budgets ($\epsilon \in \{2, 4, 6, 8, 10\}$), datasets, model architectures, and training configurations (``AM4'' refers to training with augmentation multiplicity using 4 augmentations). For MNIST and CIFAR-10, we use a blank (zero) canary; for Purchase100, we use a random dense vector canary, as a blank input to an MLP propagates no signal through the network (the hidden states are driven entirely by the biases), and is therefore not auditable.

We find that our defense generally reduces $\epslb$ with minimal impact on model utility. Notably, our defense reduced the $\epslb$ for MNIST on a CNN and CIFAR-10 on a WRN-16 (the model with the test utility on a more “complex” dataset) to $\approx$ 0 for almost every setting of $\epsub \in \{2, 4, 6, 8, 10\}$ even after applying both heuristics in Appendix~\ref{app:lbmethod}. Moreover, we found that alternative scoring rules detailed in Appendix~\ref{app:score}, including prediction entropy and gradient kurtosis, substantially reduce the auditable $\epslb$ for several of the other dataset / model combinations. 

Our defense may not universally lower $\epslb$; MNIST and CIFAR-10 on WRN-16 at $\epsub = 2$ and Purchase at $\epsub = 6$ observe an increase in $\epslb$ post-defense. We believe that this apparent discrepancy may potentially arise from finite sample error. In particular, the $90\%$ two-sided confidence intervals overlap for MNIST at $\epsub = 2$ and  Purchase at $\epsub = 6$, so this could just be an artifact of the number of models used. The $90\%$ confidence interval for CIFAR-10 at $\epsub = 2$ without the defense was not very informative, as it was $(0.0, 0.0)$, but the $\epslb$ with the defense was also negligible in this case. Notably, the reversal cases identified either had a small $\epslb$ to begin with, or were on the Purchase dataset, previously remarked on as potentially ill-suited for privacy audits in \cite{CarliniCNSTT21}. 

While the accuracy gaps are generally contained to $\pm 1$, notice the anomalous and relatively large accuracy gap for the Purchase dataset post-defense at $\epsub = 10$. In general, we believe that our finding that the accuracy was minimally affected is fairly robust: on all 24 other tasks, both training and test accuracy dropped by $1.5\%$ or less. For the specific Purchase at $\epsub = 10$ task, we believe the steeper accuracy drop may be due to the simplicity of the dataset (as highlighted in \cite{CarliniCNSTT21}), or due to use of a high $\epsub$ level causing more sensitivity to dropping a small amount of data. 

\begin{figure}[htbp]
    \centering
    \includegraphics[width=\linewidth]{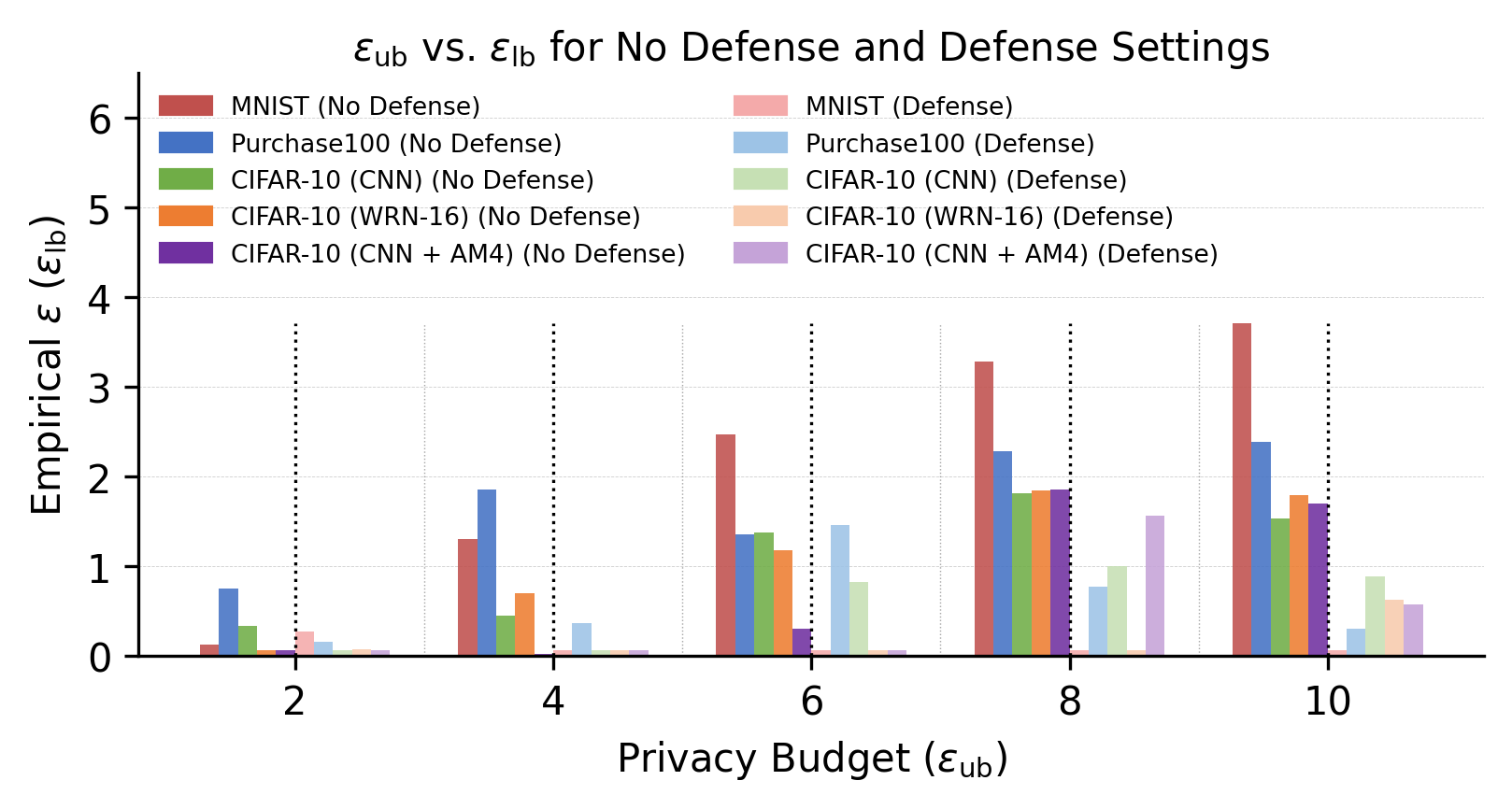}
    \caption{$\epslb$ across varying $\epsub$ budgets. }
    \label{fig:new_private_tradeoff_curves}
\end{figure}

\begin{table}[h]
    \centering
    \small
    \begin{tabular}{llcc}
    \toprule
    \textbf{Dataset (Model)} & \textbf{Theoretical }$\boldsymbol{\varepsilon}$ & \textbf{Train Acc (\%)} & \textbf{Test Acc (\%)} \\
    \midrule
    MNIST (CNN) & 2  & 97.47/96.36 & 97.40/96.57 \\
    MNIST (CNN) & 4  & 98.42/97.23 & 98.37/97.52 \\
    MNIST (CNN) & 6  & 98.65/97.44 & 98.53/97.74 \\
    MNIST (CNN) & 8  & 98.75/97.54 & 98.61/97.86 \\
    MNIST (CNN) & 10 & 98.80/97.59 & 98.64/97.89 \\
    \midrule
    Purchase (MLP) & 2  & 77.82/76.79 & 74.51/73.19 \\
    Purchase (MLP) & 4  & 87.74/87.74 & 84.09/84.09 \\
    Purchase (MLP) & 6  & 89.96/90.08 & 86.92/86.91 \\
    Purchase (MLP) & 8  & 90.40/90.40 & 87.12/87.12 \\
    Purchase (MLP) & 10 & 90.57/85.98 & 87.69/83.11 \\
    \midrule
    CIFAR-10 (CNN) & 2  & 32.89/32.77 & 32.81/32.87 \\
    CIFAR-10 (CNN) & 4  & 61.28/60.26 & 59.33/58.57 \\
    CIFAR-10 (CNN) & 6  & 70.14/68.59 & 66.93/65.55 \\
    CIFAR-10 (CNN) & 8  & 73.97/72.91 & 69.56/69.05 \\
    CIFAR-10 (CNN) & 10 & 76.15/75.35 & 71.18/70.85 \\
    \midrule
    CIFAR-10 (WideResNet-16) & 2  & 53.14/54.07 & 52.53/53.17 \\
    CIFAR-10 (WideResNet-16) & 4  & 62.86/63.60 & 61.75/62.42 \\
    CIFAR-10 (WideResNet-16) & 6  & 68.40/68.63 & 66.57/67.07 \\
    CIFAR-10 (WideResNet-16) & 8  & 72.40/70.99 & 70.13/68.75 \\
    CIFAR-10 (WideResNet-16) & 10 & 73.87/72.97 & 71.36/70.70 \\
    \midrule
    CIFAR-10 (CNN + AM4) & 2  & 31.28/31.45 & 31.72/31.45 \\
    CIFAR-10 (CNN + AM4) & 4  & 57.05/56.36 & 56.36/56.36 \\
    CIFAR-10 (CNN + AM4) & 6  & 66.07/65.26 & 65.24/64.46 \\
    CIFAR-10 (CNN + AM4) & 8  & 69.66/69.16 & 68.36/68.36 \\
    CIFAR-10 (CNN + AM4) & 10 & 71.52/70.94 & 70.49/70.03 \\
    \bottomrule
    \end{tabular}
    \caption{Defense performance on private models across privacy budgets $\varepsilon\in\{2,4,6,8,10\}$. Training and testing accuracies are reported in the table as \textit{No Defense/Defense}.}
    \label{tab:private_tradeoffs_acc}
\end{table}

We also evaluate our defense against harder input space constructions that better reflect membership inference threats. We test three attack variants on CIFAR-10 and MNIST (Table~\ref{tab:other_input_space_attacks}). First, we used adversarial FGSM examples, which iteratively apply small bounded perturbations to a sample until the model predicts a target class \cite{goodfellow2014explaining}. Second, we tested mislabeled in-distribution images, where random training samples are assigned incorrect labels uniformly at random from the remaining classes. Third, we implemented ClipBKD, which identifies the direction of least variance in the training data distribution via singular value decomposition and constructs a canary along the smallest singular vector \cite{JagielskiUO20}. 
Across all attacks, our defense reduced empirical privacy loss.

We note that in this section, we only presented results on single-canary input space audits; however, see Appendix~\ref{app:alt_audits} for experiments on input space canaries under multi-canary audits.

\begin{table}[h]
    \centering
    \small
    \begin{tabular}{lllccc}
        \toprule
        \textbf{Attack} & \textbf{Dataset (Model)} & \textbf{Defense Setting} & \textbf{Train Acc (\%)} & \textbf{Test Acc (\%)} & \textbf{Empirical $\varepsilon$} \\
        \midrule
        FGSM       & CIFAR-10 (CNN) & No Defense & 76.06 & 70.75 & 1.14 \\
        FGSM       & CIFAR-10 (CNN) & Defense    & 74.95 & 70.14 & 0.07 \\
        FGSM       & MNIST (CNN)    & No Defense & 98.78 & 98.57 & 3.03 \\
        FGSM       & MNIST (CNN)    & Defense    & 97.58 & 97.78 & 0 \\
        \midrule
        Mislabeled & CIFAR-10 (CNN) & No Defense & 76.31 & 71.37 & 0.47 \\
        Mislabeled & CIFAR-10 (CNN) & Defense    & 75.05 & 70.61 & 0 \\
        Mislabeled & MNIST (CNN)    & No Defense & 98.80 & 98.64 & 2.45 \\
        Mislabeled & MNIST (CNN)    & Defense    & 97.59 & 97.89 & 0.69 \\
        \midrule
        ClipBKD    & CIFAR-10 (CNN) & No Defense & 76.13 & 71.02 & 2.93 \\
        ClipBKD    & CIFAR-10 (CNN) & Defense    & 75.74 & 70.86 & 0 \\
        ClipBKD    & MNIST (CNN)    & No Defense & 98.88 & 98.60 & 3.46 \\
        ClipBKD    & MNIST (CNN)    & Defense    & 97.59 & 97.89 & 0 \\
        \bottomrule
    \end{tabular}
    \caption{Empirical privacy loss under non-blank input space attacks for $\epsub=10$.}
    \label{tab:other_input_space_attacks}
\end{table}

\subsection{Label-only attack}

Prior work has proposed defenses without provable privacy guarantees against membership inference attacks, and follow-up work has proposed targeted attacks to break them. In this section, we evaluate one such attack proposed by \cite{aerni2024evaluations}, which works in the hidden state, input space setting, and is thus consistent with our threat model. This adaptive, label-only attack is known to break several previously-proposed empirical defenses, including HAMP \cite{ChenP24}. 

HAMP combines a training-time modification with a test-time output defense. During training, the model is trained on soft labels rather than hard labels: the target class is assigned probability determined by an entropy threshold, and the remaining probability mass is spread uniformly over the other classes. The training loss combines a KL-divergence term (between the model's predicted class distribution and the soft label distribution) with a term that rewards higher entropy of the model's predicted class distribution, weighted by a tunable regularization strength. Training uses a batch size of 64, weight decay of 1e-5, and a learning rate of 0.5 with a step schedule (decayed by 10x at epochs 60, 90, and 150), for 200 epochs. At test time, HAMP replaces the model's output logit vector with a logit vector obtained by querying the model on random noise images, sorting those logits, and reordering them to match the rank order of the real prediction’s logits. This preserves the predicted label while destroying all information about the model's confidence.

Follow-up work \cite{aerni2024evaluations} proposes a label-only attack to break defenses, like HAMP, that rely on confidence masking. For a target sample, the attack constructs an 18-dimensional binary feature vector by querying the model on 18 fixed augmentations of the sample (all combinations of horizontal flip {none, flip} and pixel shifts $\{0, -4, +4\}$ along each axis), recording whether the model's predicted label matches the true label under each augmentation. A logistic regression classifier (l2 penalty, C = 1.0) is trained in a leave-one-out fashion across 64 shadow models, with 500 in-distribution mislabeled canary samples: for each target model, the classifier is fit on the binary feature vectors and membership labels of all other shadow models, and evaluated on the held-out target model's vector. The classifier's predicted probability of membership serves as the attack's membership score.
Attack success is measured by \cite{aerni2024evaluations} as TPR at 0.1\% FPR.

We implemented both HAMP and the label-only attack in our codebase and reproduced the results of \cite{aerni2024evaluations}: the label-only attack achieves a TPR of 28.5\% at 0.1\% FPR against HAMP (evaluated on mislabeled canaries), showing that HAMP's confidence-masking can be easily circumvented by a label-only signal, since the defense only erases confidence information while leaving the predicted label (and hence the model's tendency to correctly classify memorized canary samples) fully intact.

We then ran our filtering defense, in the non-private setting, against the same label-only attack, using the same number of canaries (500), shadow models (64), and learning rate (0.5) as the HAMP configuration. We omit HAMP's momentum and weight decay: weight decay causes numerical instability in our custom WideResNet implementation (which replaces batch norm with group norm to support vmap), and we found that omitting momentum yields higher model utility for our defense. 

Under this setup, the attack achieves a TPR of only 0.95\% at 0.1\% FPR against our defense. This suggests that, unlike HAMP, our defense's protection is not reliant on obscuring confidence information that a label-only attack could otherwise exploit; instead, filtering removes the underlying vulnerable canary samples themselves before they are memorized.

\subsection{Defense-aware input space attack}\label{ssec:defense_aware}

Beyond evaluating previously-proposed input space canaries in the literature (in Section~\ref{ssec:input}), here we detail an attempt to learn a defense-aware canary that dodges our filter while remaining auditable. We formulate this as a bi-level optimization problem, where the inner loop simulates the training procedure and the outer loop updates the canary.

Let $x$ denote a candidate canary; we want to map $x$ to $y_w$, a target incorrect label. For a given $x$, the inner loop trains the model using DP-SGD and our gradient-norm filtering defense. Let $K(x)$ denote the epoch at which the defense detects $x$, with $K(x) = -1$ if the canary is never detected. When the canary is detected on epoch $K(x)$, let $\theta_{K(x)}$ denote the model parameters used to compute the defense score on that epoch, and let $\tau_{K(x)}$ denote the corresponding detection threshold i.e. the $k$-th largest gradient score in that epoch. 

\begin{equation}
\mathcal{L}(x)
= \ell\big(f_{\theta_0}(x),\, y_w\big)
+ \lambda \cdot
\begin{cases}
\max\left(0,\;
\left\|\nabla_{\theta} \ell\big(f_{\theta_{K(x)}}(x),\, y_w\big)\right\|_{\infty}
- \tau_{K(x)}\right), & \text{if } K(x) \neq -1, \\
0, & \text{otherwise}
\end{cases}
\label{eq:defense_aware_opt}
\end{equation}

The first term of \eqref{eq:defense_aware_opt} pushes the canary to be compatible with its assigned label under the initial model. We use this initial-model loss as a tractable proxy for auditability because directly optimizing the canary’s loss after training would require differentiating through the entire DP-SGD training trajectory. Using the loss at the detection epoch suffers from the same problem and would also create a dependency on $K(x)$. The model initialization, $\theta_0$, is set independent of the canary, so this offers a simple differentiable objective with respect to $x$. 

The second term of \eqref{eq:defense_aware_opt} penalizes the canary when it has a gradient above the detection threshold, with a penalty proportional to its distance above the detection threshold. 

 The $\lambda$ parameter controls the trade-off between auditability and defense evasion.

The canary is updated using an FGSM-style step. For our experiments, we optimize a canary chosen from class 1 and assigned a target label 0. We use 50 outer optimization steps with a fixed step size of 1e-2 and 3 independent inner training runs per outer step. 

Table \ref{tab:defense_aware_opt} shows that across various $\lambda$, while we increase the average epoch on which the canary is dropped, we cannot push the drop epoch past the end of training (100 epochs). We run a full audit on the learned canary with the largest drop epoch, but we find that the initial canary (pre-optimization) yields a higher $\epslb$ than our learned canary (2.453 vs 1.962). Post defense, the initial canary’s $\epslb$ drops to 0.65, while our learned canary’s $\epslb$ drops to 0.

\begin{table}[htbp]
\centering
\begin{tabular}{c c}
\hline
$\lambda$ & Average drop epoch \\
\hline
0.0   & 5.7  \\
0.0   & 6.0  \\
0.1   & 16.7 \\
0.5   & 14.0 \\
1.0   & 15.0 \\
5.0   & 13.0 \\
10.0  & 13.0 \\
100.0 & 14.0 \\
\hline
\end{tabular}
\caption{Defense-aware canary optimization attempts across different values of $\lambda$.}
\label{tab:defense_aware_opt}
\end{table}

\subsection{Gradient space attacks}\label{ssec:gradient}

In this section, we consider defense-aware canaries that operate under a gradient space threat model. Our goal is to qualitatively examine the effectiveness of our filtered defense when exposed to a stronger, not always realizable, threat model for canary construction.

\textbf{Gradient space attacks.}
For gradient space attacks, we follow the construction of \cite{CebereBP25}, where the attacker uses a holdout model to construct canary gradients as one-hot vectors in the direction of the largest parameter update: $\arg\max_i (\Delta \theta_T - \Delta \theta_0)$, where $\theta_T$ is the final model and $\theta_0$ is the initialization. The auditor then measures the parameter difference $\Delta\theta = \theta_T - \theta_0$ at index $i$, corresponding to the one-hot index of the canary gradient vector.

\textbf{Gradient bandwidth attack.}
In this attack, we allow the attacker to observe intermediate gradient distributions and defense hyperparameters over the course of training for some holdout model. The attacker first tracks $s_t$, the 6th largest gradient norm within class  0 at each epoch $t$. Then, the attacker constructs 500 copies of a one-hot gradient space canary, each with $L_\infty$ norm equal to $\min_t s_t$ (the minimum $6^{\text{th}}$ largest gradient norm observed across all training epochs). This construction exemplifies an attacker with the specific goal of dodging our filtering bandwidth in the gradient space, full state threat model. If we choose to filter 5 samples from each class on every epoch, the gradient space canaries will remain  just below the filtering threshold throughout training. Under this attack, we obtained (group privacy) $\epslb = 27.26$ both with and without the defense. We emphasize that this $\epslb$ is a \emph{group privacy} statistic, so it does not contradict the $\epsub$ setting.

\textbf{Gradient canceling attack.} We were able to break our defense more dramatically by using a more complex \emph{gradient canceling} attack~\cite{KohSL22,LuKY23}, designed with knowledge of our defense's hyperparameter choices. 
This attack exploits our filter by constructing two groups of gradient space canaries with norms that sum to zero. Group A contains 500 canaries with $L_\infty$ norm 1, while Group B contains 600 canaries with $L_\infty$ norm $\frac 5 6$ in the opposite direction. Without our defense, the two groups cancel each other out during training, yielding an $\epslb$ near zero. However, when the defense is active and the defense hyperparameters are chosen so that we only filter Group A, it inadvertently exposes Group B, resulting in non-zero privacy loss. We evaluate privacy using the group $\epslb$  (Table~\ref{tab:gradient_canceling}).

\begin{table}[h]
    \centering
    \small
    \begin{tabular}{lc}
        \toprule
        Defense Setting & Group-Privacy $\varepsilon_{\mathrm{lb}}$ \\
        \midrule
        No Defense & 4.43 \\
        Defense    & 27.26 \\
        \bottomrule
    \end{tabular}
    \caption{Group-privacy $\varepsilon_{\mathrm{lb}}$ under the gradient canceling attack on MNIST.}
    \label{tab:gradient_canceling}
\end{table}

This attack assumes a stronger adversary with knowledge of the defense hyperparameters to choose the canary group sizes. For example, if we instead construct Group A with 10 canaries and Group B with 5 canaries, even if we only allow filtering 1 canary across the whole dataset per epoch, we can still discard both canary groups from the training data, yielding $\epslb = 0$.

We also attempt to reproduce this gradient canceling attack within our intended hidden state, input space threat model by leveraging the structure of a single-layer MLP, where sample gradients are aligned with sample features by the chain rule. This suggests that, in principle, one could construct input space canaries whose gradients mimic the desired cancellation behavior; the key challenge is to simulate the desired magnitude, not just the direction, of a gradient space attack in input space.

To this end, we explore several input space constructions for Group A: (i) one-hot vectors with $L_\infty$ norm exceeding the maximum per-sample gradient norm observed across all epochs on a holdout set, (ii) dense random vectors with $L_2$ norm similarly exceeding this maximum, and (iii) adversarially optimized inputs obtained via FGSM to explicitly induce large gradients.
However, these input space constructions fail to replicate the gradient space attack. The core issue is that we cannot reliably control the per-sample gradient magnitudes across training epochs. As a result, only a small fraction of canaries ($1\%$) are consistently filtered, which is insufficient to disrupt the cancellation effect.

%% file: lbmethod.tex
\section{Empirical Privacy Reporting Details}\label{app:lbmethod}

\subsection{Heuristic conventions and their (potential) pitfalls}\label{ssec:heuristics}

In this section, we detail two practices often used in the privacy auditing literature, and point out when their guarantees formally apply (contrasted with when they should be treated as a heuristic). While this discussion is sometimes present in prior works, we believe making it explicit, with specific counterexamples to the validity of heuristics, will prove valuable to future auditing research, as well as practitioners reasoning about the guarantees of these reporting metrics.

Following Section~\ref{ssec:measure_lb}, the formal way to certify a privacy lower bound on an algorithm $\alg$ is as follows. First, a test $\phi$ however the auditor desires (e.g., the most effective loss threshold is selected using a set of models, which are discarded). Next, fresh holdout models are trained using $\alg$ and evaluated according to the test $\phi$, producing an empirical FPR-FNR tradeoff $(\alpha, \beta)$. Then, a Clopper-Pearson correction with failure probability $\gamma$ is applied to obtain $(\hat{\alpha}, \hat{\beta})$ lying above the $f$-DP curve with probability $1 - \gamma$. Finally, \eqref{eq:eps_convert} converts this point into an $\epslb$ estimate.
The two main heuristics we discuss in this section are \emph{multiple hypothesis testing} and \emph{GDP extrapolation}. We describe these two heuristics in turn, and then examine their potential pitfalls.

\textbf{Multiple hypothesis testing.} Algorithm~\ref{alg:fdp_audit} provides pseudocode for the \emph{multiple hypothesis testing} heuristic, where the same models are used for selecting the test $\phi$, and evaluating it to produce a point $(\halpha, \hbeta)$ above the $f$-DP curve. Observe that the same $k$ i.i.d.\ copies are used in all of the tests $\{\phi_i\}_{i \in [m]}$, and then the best $\epslb$ estimate amongst these tests is reported using \eqref{eq:eps_convert}. It is formally incorrect to say that the selected pair $(\halpha_i, \hbeta_i)$ lies above the true $f$-DP curve with probability $1 - \gamma$, due to the potential for false discovery. The formal fix is to either use a Bonferroni correction (union bound) and set the coverage of each Clopper-Pearson correction in Algorithm~\ref{alg:fdp_audit} to $1 - \frac \gamma m$, or to apply a selected test $\phi_i$ to holdout samples (with Clopper-Pearson coverage $1 - \gamma$). Indeed, the original work of \cite{JagielskiUO20} prescribes using the holdout strategy to ensure statistical validity.

\begin{algorithm}
\DontPrintSemicolon
\caption{$\elb(\alg, \{\calD_i, \calD'_i, \phi_i\}_{i \in [m]}, \delta, \gamma, k)$}
\label{alg:fdp_audit}
\textbf{Input:} Randomized algorithm $\alg: \calX^n \to \Omega$, neighboring datasets $\calD, \calD' \in \calS^*$ and $m$ tests $\phi_i: \Omega \to [0, 1]$, $\delta \in (0, 1)$, $\gamma \in (0, 1)$, $k \in \N$ \;
\textbf{Output:} Empirical privacy lower bound $\epslb(\delta)$\;
%$\epslb \gets 0$\;
$\{\omega_j\}_{j \in [k]} \simiid \alg(\calD)$\;
$\{\omega'_j\}_{j \in [k]} \simiid \alg(\calD')$\;
\For{$i \in [m]$}{
$\alpha_i^{\text{init}} \gets \frac 1 k \sum_{j \in [k]} \phi_i(\omega_j)$\;
$\beta_i^{\text{init}} \gets 1 - \frac 1 k \sum_{j \in [k]} \phi_i(\omega'_j)$\;
$(\hat{\alpha}_i, \hat{\beta}_i) \gets (\Beta^{-1}_{k\alpha_i^{\textup{init}} + 1, k(1 - \alpha_i^{\textup{init}})}(1 - \gamma), \Beta^{-1}_{k\beta_i^{\textup{init}} + 1, k(1 - \beta_i^{\textup{init}})}(1 - \gamma))$ \tcp*{Clopper-Pearson, $\Beta_{a, b}$ is CDF of beta distribution with parameters $a, b$}
}
\codeReturn $\epslb(\delta) \gets \max_{i: \hat{\alpha}_i > 0} \log(\frac{1 - \hat{\beta}_i - \delta}{\hat{\alpha}_i})$\;\label{line:epslb_compute}
\end{algorithm}

Subsequent works in privacy auditing are less careful on this matter, with many (e.g., \cite{MaddockSS23, ZanellaBeguelinWTSRPNKJ23, NasrHSBTJCT23, AnnamalaiC24, CebereBP25}) utilizing a variant of the multiple hypothesis testing method in Algorithm~\ref{alg:fdp_audit}, \emph{without} explicitly using a Bonferroni correction or holdouts. Indeed, Section 5.2 of \cite{NasrHSBTJCT23} explicitly flags this issue, where they state that selecting an audit in this way is ``technically not valid,'' but that ``it has become common'' to report lower bounds in this way. 

Our experiments (cf.\ Appendix~\ref{app:other_eps}) suggest that the Bonferroni correction yields substantially smaller (often negligible) $\epslb$ estimates. On the other hand, the holdout method still results in nontrivial formal $\epslb$ estimates when auditing models trained with DP-SGD, although the tightness of these estimates often worsens by a factor of $3$-$10 \times$. We thus caution against taking reported $\epslb$ wholly at face value if multiple discovery is not accounted for, at least without ablations justifying the heuristic. For completeness, we report estimated $\epslb$ both with and without multiple discovery, as trends between our algorithms and DP-SGD remain consistent across reporting methods.

Another subtlety regarding multiple discovery that arises in the practice of privacy auditing (and private ML more broadly) is \emph{hyperparameter selection}, see e.g., \cite{MohapatraSHKT22,PapernotS22,Ponomareva_2023}. To circumvent this issue, all of our hyperparameter selection is explicitly performed using holdout models that are discarded before selecting our audits and reporting our $\epslb$ estimates.

Finally, we make one observation that is implicit in prior works using Algorithm~\ref{alg:fdp_audit}, but which has not received explicit discussion to our knowledge.
A reasonable modification to Line~\ref{line:epslb_compute} of Algorithm~\ref{alg:fdp_audit} uses Lemma~\ref{lem:tradeoff}, i.e., it first computes the largest convex function $f$ below the points $\{(\hat{\alpha_i}, \beta_i)\}_{i \in [m]}$, before applying the formula \eqref{eq:eps_convert}. This modification could in principle yield tighter audits, as it leverages the ability to use points on the $f$-DP curve outside the $\{(\halpha_i, \hbeta_i)\}_{i \in [m]}$. As shown in Lemma~\ref{lem:no_convexify}, this convexification step in fact does not change the computed $\epslb$.

\begin{lemma}\label{lem:no_convexify}
Let
$\mathcal S=\{(\alpha_i,\beta_i)\}_{i\in [n]} \subset [0,1]^2$
have distinct first coordinates, and let $f$ be the lower convex envelope of
\(\mathcal S\) (the pointwise largest convex function with
$f(\alpha_i)\le \beta_i$ for all $i\in[n]$).\footnote{The lower convex envelope is uniquely defined, see Exercise 3.30 of \cite{BoydV04}.} Also, assume that $\{(0, 1), (1, 0)\} \subset \calS$.
Then for all $\delta \in (0, 1)$,
\[
\sup_{x\in(0,1]} \log\Par{\frac{1-f(x)-\delta}{x}}
=
\max_{i:\alpha_i>0}
\log\Par{\frac{1-\beta_i-\delta}{\alpha_i}},
\]
with the convention that the logarithm of a nonpositive number is $-\infty$.
\end{lemma}
\begin{proof}
Throughout the proof, fix some $\delta \in (0, 1)$ and define 
\[\psi(x) \defeq \log\Par{\frac{1 - f(x) - \delta}{x}},\quad M_1 \defeq \sup_{x \in (0, 1]}\psi(x),\quad M_2 \defeq \max_{i: \alpha_i > 0} \log\Par{\frac{1 - \beta_i - \delta}{\alpha_i}}. \]
It is immediate that $M_2 \le M_1$, because $\beta_i \ge f(\alpha_i)$ and $\alpha_i$ is a candidate argument in the definition of $M_1$. For the reverse inequality, the lower convex envelope of
finitely many points is piecewise affine, and its breakpoints occur at points
\(\alpha_i\) where the constraint is tight, i.e., $f(\alpha_i)=\beta_i$ (see Theorem 19.1, \cite{Rockafellar97}). Therefore, it is enough to show that on an interval $x \in [\alpha_i, \alpha_j]$ such that $f(x) = ax + b$, $f(\alpha_i) = \beta_i$, and $f(\alpha_j) = \beta_j$, $\psi$ achieves its maximum at an endpoint $\alpha_i$ or $\alpha_j$.

To see this last fact, observe that on this interval,
\[\psi(x) = \log\Par{\frac{1 - b - \delta - ax}{x}},\quad \psi'(x) = -\frac{1 - b - \delta}{x(1 - b - \delta - ax)},\quad \text{ for } x < \frac{1 - b - \delta}{a}.\]
Hence $\psi'$ has a fixed sign on this interval (whenever $\psi > -\infty$), so $\psi$ is supremized at an endpoint.
\end{proof}

\textbf{Gaussian DP (GDP) extrapolation.} GDP curves \cite{DongRS22} are a family of $f$-DP curves that are particularly analytically convenient. Informally, these are a one-parameter family of $f$-DP curves that exactly capture the trade-off functions $T_{P, Q}$ (Definition~\ref{def:tradeoff}) when $P = \Nor(0, 1)$ and $Q = \Nor(\mu, 1)$ are equal-variance Gaussians; here, the one parameter is $\mu$. This family is relevant to our setting because DP-SGD utilizes the Gaussian mechanism, and the composition of GDP curves remains a GDP curve (Corollary 3.3, \cite{DongRS22}). GDP curves are notably compatible with sample-efficient auditing: if the true privacy loss curve is GDP, then the auditor can extrapolate from a single tradeoff point $(\alpha, \beta)$ to rule out part of the GDP parametric family, improving statistical efficiency.
Consequently, prior works starting from \cite{NasrHSBTJCT23} (see also \cite{AnnamalaiC24, CebereBP25}) propose using GDP methodology to achieve tighter audits. We call this strategy \emph{GDP extrapolation}, and note that it should be treated as a heuristic whenever the true $f$-DP curve is not a GDP curve.

In this section, we point out two potential pitfalls to using this methodology to formally audit DP-SGD variants. These pitfalls stem from the facts that GDP curves are \emph{not} closed under the following operations: (1) Poisson subsampling, and (2) marginalizing intermediate outputs (i.e., in the hidden state threat model). We formalize these points with explicit counterexamples in Corollaries~\ref{cor:ss_nogdp} and~\ref{cor:hide_nogdp}.
Prior works have also discussed these issues to some extent, e.g., Section 5.5 of \cite{NasrHSBTJCT23}, and an extended discussion of the subsampling issue in \cite{gomez2025gaussian}. 

One common justification for use of GDP-based auditing, when the $f$-DP curve is not formally a GDP curve, is the $f$-DP central limit theorem (Theorem 3.4, \cite{DongRS22}), as essentially all private optimizers in ML are iterative methods that successively apply the same mechanism. We caution that again, the proof of the $f$-DP central limit theorem relies on accessing intermediate outputs (which becomes problematic under hidden state threats, e.g., as in \cite{CebereBP25}), and hence this logic is susceptible to the counterexample in Corollary~\ref{cor:hide_nogdp}. Altogether, GDP-based auditing requires careful justification depending on the setup; for this reason, our experiments report both the GDP-based $\epslb$ and a weaker, but more formal, estimate based on Clopper-Pearson correction.

Let $\Phi$ denote the CDF of a standard Gaussian, and let $\varphi = \Phi'$ denote its PDF. The following lemma is the basis of our counterexamples, and states that the trade-off function between a particular Gaussian and mixture of two Gaussians is not a GDP curve.

\begin{lemma}\label{lem:not_gdp}
    Consider the trade-off function $T_{P, Q}(\alpha)$ for $P = \Nor(0, 1)$ and $Q = \frac{1}{2} \Nor(0, 1) + \frac{1}{2} \Nor(1, 1)$. 
    There does not exist a $\mu > 0$ such that $T_{P, Q} = G_\mu$, where $G_\mu(\alpha) \defeq \Phi(\Phi^{-1}(1-\alpha) - \mu)$.
\end{lemma}
\begin{proof}
By the Neyman--Pearson lemma, optimal tests threshold the likelihood ratio. Here
\[
    \frac{dQ}{dP}(x)
    =
    \frac12+\frac12\exp\!\left(x-\frac12\right),
\]
which is strictly increasing in $x$. Therefore, for false-positive rate
\(\alpha = \Pr[X\ge t]=1-\Phi(t)\), the optimal test is \(\ind_{x\ge t}\), , where $\mathbb{I}_{\calE}$ is the indicator of an event $\calE$, and its false-negative rate is
\[
    T_{P,Q}(\alpha)
    =
    \Pr_{X \sim Q}[X < t]
    =
    \frac12\Phi(t)+\frac12\Phi(t-1),
    \qquad t=\Phi^{-1}(1-\alpha).
\]
If \(T_{P,Q}=G_\mu\), then for all \(t\in\mathbb R\),
$\frac12\Phi(t)+\frac12\Phi(t-1)=\Phi(t-\mu)$.
Differentiating gives
\[
    \frac12\varphi(t)+\frac12\varphi(t-1)=\varphi(t-\mu).
\]
Dividing by \(\varphi(t)\),
\[
    \frac12+\frac12 e^{t-\half}=e^{t\mu-\frac{\mu^2} 2}.
\]
As \(t\to -\infty\), the two sides tend to $\half$ and $0$ respectively for every \(\mu>0\), a contradiction.
\end{proof}

Our claimed counterexamples now follow simply from Lemma~\ref{lem:not_gdp}.

\begin{corollary}\label{cor:ss_nogdp}
    The trade-off function for the Gaussian mechanism with Poisson subsampling does not always correspond to a $\mu$-GDP trade-off function for some $\mu > 0$. 
\end{corollary}
\begin{proof}
Consider a single-element dataset, either $0$ or $1$, and Poisson subsampling with parameter $q = \half$ composed with the Gaussian mechanism (adding $\Nor(0, 1)$ to the empirical mean). If the dataset is $0$, then the resulting distribution is $P = \Nor(0,1)$, while if the dataset is $1$, then the resulting distribution is $Q = \frac12 \Nor(0,1) + \frac12\Nor(1, 1)$, which is exactly the example described in Lemma~\ref{lem:not_gdp}.
\end{proof}

\begin{corollary}\label{cor:hide_nogdp}
    Consider composing several Gaussian mechanisms (adaptively) and outputting only the result of the final mechanism. 
    The trade-off function of this mechanism does not always correspond to a $\mu$-GDP trade-off function for some $\mu > 0$. 
\end{corollary}
\begin{proof}
Consider a two-stage Gaussian mechanism,
with a single-element dataset $d$ that is either $0$ or $1$. 
The first stage is to draw $X \sim \Nor(0, 1)$, independent of the dataset.
We adaptively choose the second stage based on the result of the first: we output a sample from $\Nor(d \cdot \mathbb{I}_{\{X \geq 0\}}, 1)$. 
Considering the various choices of $d$, this again results in exactly the aforementioned distributions $P = \Nor(0, 1)$ or $Q = \frac12 \Nor(0, 1) + \frac12 \Nor(1, 1)$.
\end{proof}

\subsection{Numerical precision}
\label{ssec:precision}
\newcommand{\epsraw}{\eps_{\rm raw}}
\newcommand{\epscp}{\eps_{\rm cp}}

Applying a fully formal empirical privacy estimation methodology, e.g., Clopper--Pearson corrections after holdout, can sometimes report highly-conservative $\epslb$ figures that sometimes even show up as vanishing. 
We numerically examine how the Clopper--Pearson (CP) correction changes empirical privacy lower bounds in the finite-sample regime. For this discussion, we fix the significant level of confidence internal to \(\gamma=0.05\) and ignore the union bound over multiple tests (i.e., the false discovery direction in Appendix~\ref{app:lbmethod}, so that each of the two error rates, \(\alpha\) and \(\beta\), is corrected using a one-sided CP upper confidence bound with failure probability \(\frac \gamma 2=0.025\). For each sample size \(k\in\{200,500,1000,5000, 10000\}\), we enumerate all possible empirical count pairs
\( (x_\alpha,x_\beta)\in[k]\times [k]\), 
where
\(
    \alpha_{\rm emp}=\frac{x_\alpha}{k},
    \beta_{\rm emp}=\frac{x_\beta} k.
\)
For each pair, we compute the raw empirical lower bound and the CP corrected lower bound: 
\[
    \epsilon_{\rm raw}
    =
    \max\left\{0, \log\frac{1-\beta_{\rm emp}-\delta}{\alpha_{\rm emp}}\right\},\quad \epsilon_{\rm CP}
    =
    \max\left\{0,\log\frac{1-\hat\beta-\delta}{\hat\alpha}\right\},
\]
where \(\hat\alpha\) and \(\hat\beta\) are the CP upper confidence bounds, and restrict attention to count pairs with \(\epsilon_{\rm raw}>0\). This isolates empirical observations that initially provide a nonzero privacy signal and shows how many of them are weakened or collapsed to zero after CP correction. We show our result in Figure~\ref{fig:cp-multik-grid}. We make two observations from the numerical results.
\begin{enumerate}
    \item In the experiment, we filter out all \((\alpha,\beta)\) pairs for which \(\epsraw=0\) or \(\epsraw=+\infty\). After this filtering, the range of finite empirical lower bounds that can appear is still limited by the sample size. In particular, the largest possible auditable $\epslb$ grows only logarithmically with \(k\): it is around \(4\) for \(k=200\) in our plotted range and around \(9\) for \(k=10000\).

    \item As \(k\) increases, fewer nonzero raw observations collapse to zero after CP correction, and the average gap between \(\epsraw\) and \(\epscp\) also decreases. However, the worst-case raw value among the collapsed observations does not vanish; in our experiments, it remains around \(2.4\).
\end{enumerate}

Both observations have simple theoretical explanations. For the first
observation, over all finite count pairs with sample size \(k\), the largest possible raw lower bound is attained when \(x_\alpha=1\) and \(x_\beta=0\), giving
\[
    \epsraw^\star
    =
    \log\frac{1-\delta}{1/k}
    =
    \log\bigl((1-\delta)k\bigr).
\]
Thus, the auditable range of finite empirical lower bounds grows only
logarithmically with the number of samples, suggesting that accurately auditing large \(\epsilon\) values requires very large \(k\).

For the second observation, the non-vanishing worst-case collapse is caused by extreme low-count configurations near the boundary of the binomial experiment. Let \(x_\alpha=1,x_\beta=k-c\), where $c$ is some constant that will be specified later. Equivalently \(\alpha_{\rm emp}=\frac{1}{k},1-\beta_{\rm emp}=\frac{c}{k}.\)
Therefore the raw lower bound is
\(\epsraw=\log\frac{c/k-\delta}{1/k}.\)
When \(\delta\) is negligible compared with \(c/k\), this is approximately \(\epsraw\approx\log c\), which is independent of \(k\).
We then consider the CP correction: 
\[
    \hat\alpha
    \approx
    \frac{q_{1-\gamma/2}(\Gamma(2,1))}{k},   
    \quad  
    \hat\beta
    \approx
    1 - \frac{q_{\gamma/2}(\Gamma(c,1))}{k}.
\]
where \(q_p(\Gamma(a,1))\) denotes the \(p\)-quantile of a Gamma distribution. The CP-corrected lower bound collapses to zero when $\hat{\alpha} + \hat{\beta} + \delta \ge 1$. Substituting the approximations above, the collapse condition becomes
\[
    q_{\gamma/2}(\Gamma(c,1)) - k\delta
    \lesssim
    q_{1-\gamma/2}(\Gamma(2,1)).
\]
For \(\gamma=0.05\), the right-hand side is approximately \(5.57\). When
\(k\delta\) is small, the largest integer \(c\) satisfying this condition is about \(c=11\), giving
\[
    \epsraw
    =
    \log(c-k\delta)
    \approx
    \log 11
    \approx
    2.4.
\]
Thus, the non-vanishing maximum collapsed value is a finite-count boundary effect: it persists for large \(k\) as long as \(k\delta\) remains small relative to the relevant true-positive count \(c\). In the pure-DP case \((\delta=0)\), this gap persists for any choice of finite \(k\). In the common auditing setting with \(\delta=10^{-5}\), the gap begins to disappear only when \(k\delta\) becomes comparable to \(c\); for \(c\approx 11\), this requires \(k\gtrsim 10^6\), which is a rather demanding sample size.
\begin{figure*}[h]
    \centering

    % Row 1: k = 200
    \begin{subfigure}[t]{0.32\textwidth}
        \centering
        \includegraphics[width=\linewidth]{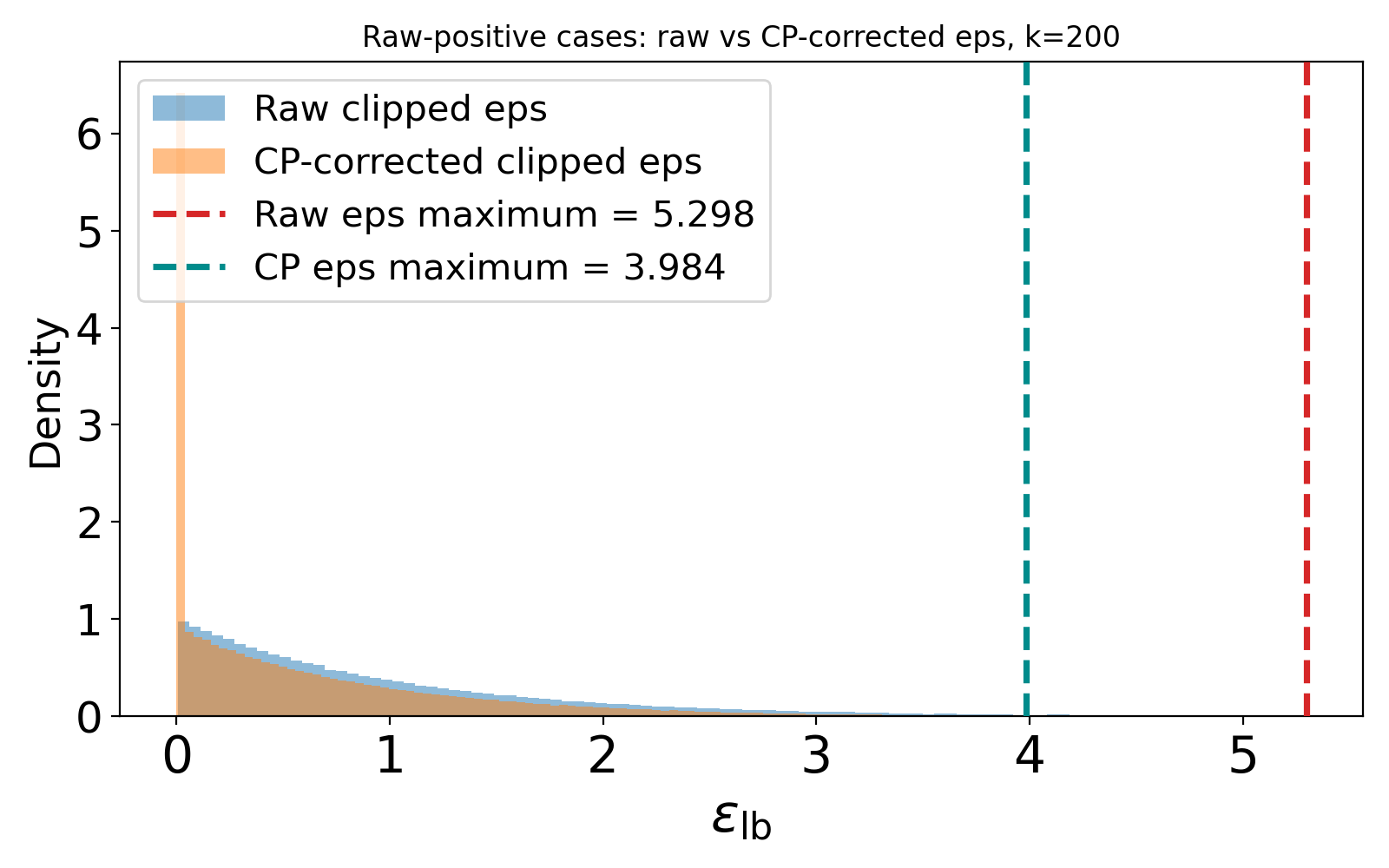}
        \caption*{\(\epsraw\) v.s. \(\epscp\)}
    \end{subfigure}
    \hfill
    \begin{subfigure}[t]{0.32\textwidth}
        \centering
        \includegraphics[width=\linewidth]{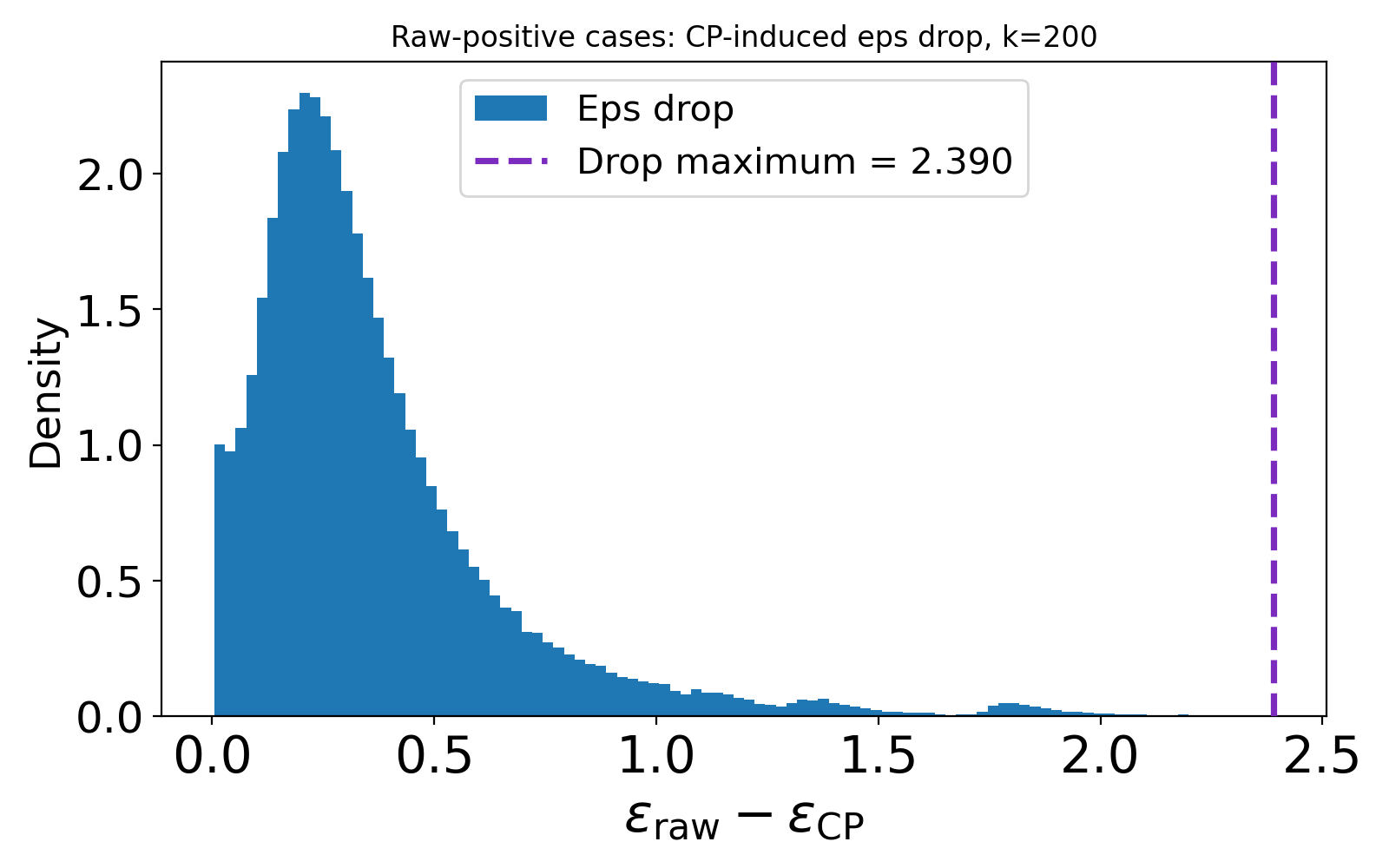}
        \caption*{\(\epsilon_{\mathrm{raw}} -\epsilon_{\mathrm{CP}} \)}
    \end{subfigure}
    \hfill
    \begin{subfigure}[t]{0.32\textwidth}
        \centering
        \includegraphics[width=\linewidth]{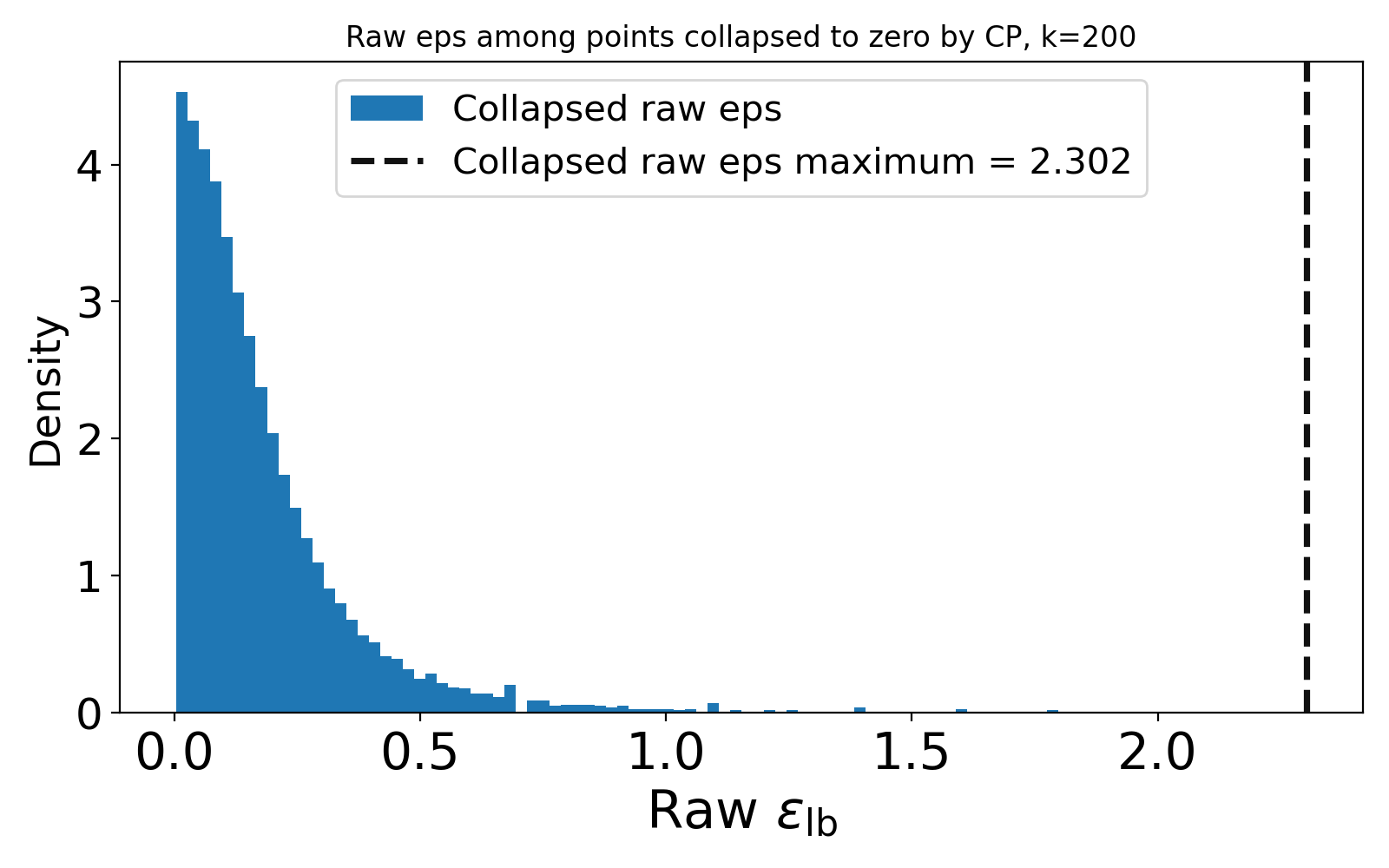}
        \caption*{\(\epsraw\) among collapse points}
    \end{subfigure}

    \vspace{0.5em}
    {\small (a) \(k=200\)}

    \vspace{1em}

    % Row 2: k = 1000
    \begin{subfigure}[t]{0.32\textwidth}
        \centering
        \includegraphics[width=\linewidth]{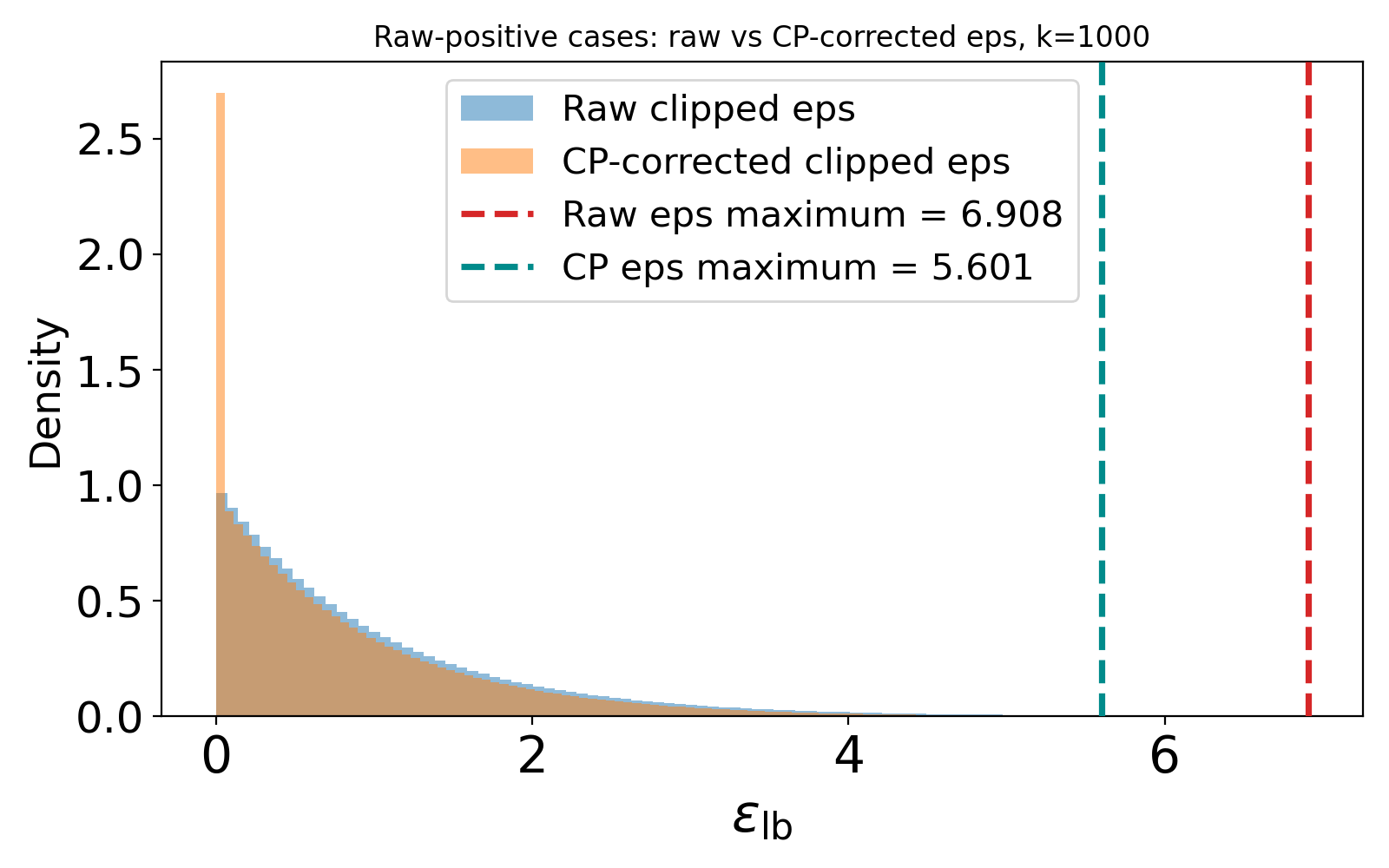}
        \caption*{\(\epsraw\) v.s. \(\epscp\)}
    \end{subfigure}
    \hfill
    \begin{subfigure}[t]{0.32\textwidth}
        \centering
        \includegraphics[width=\linewidth]{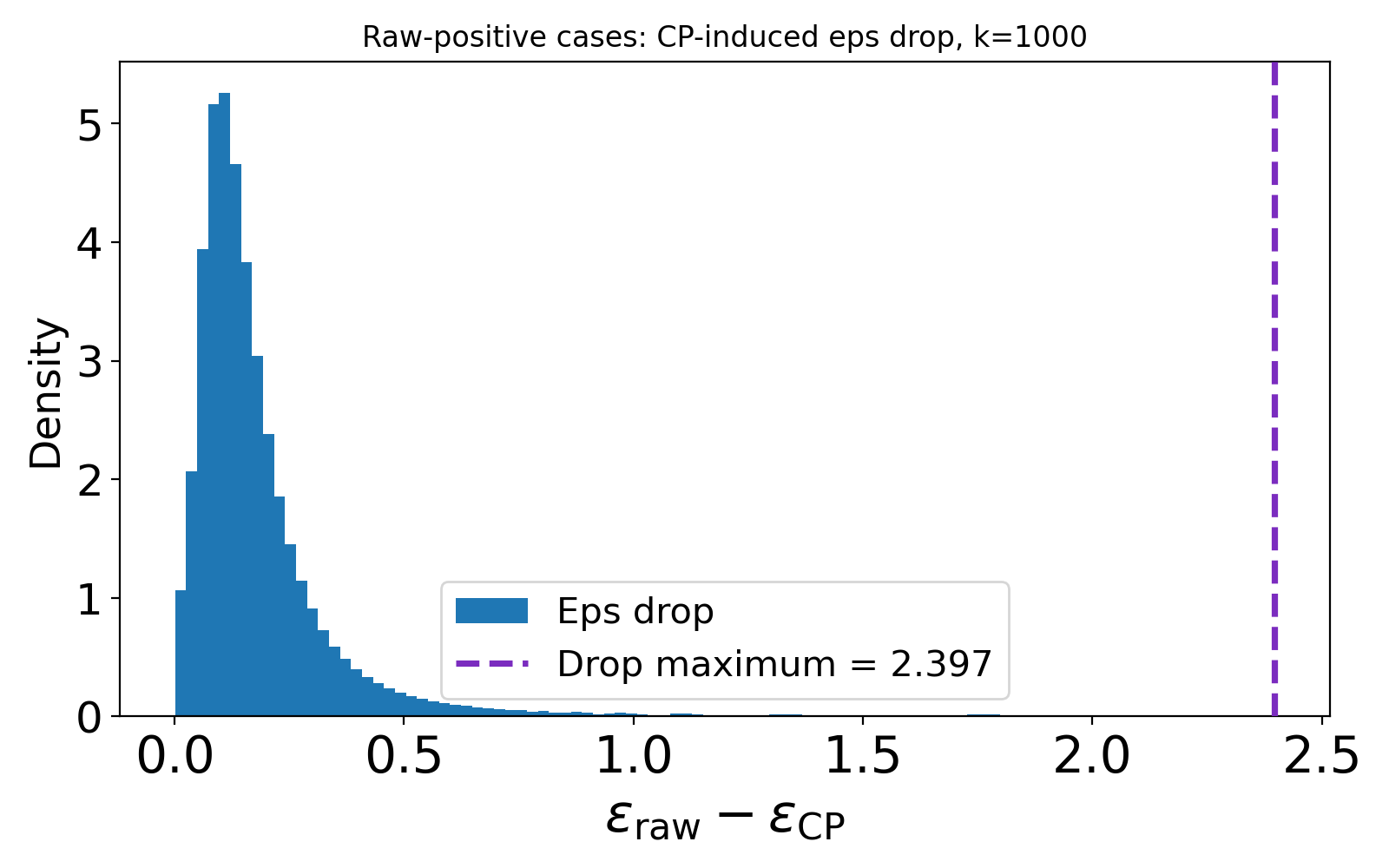}
        \caption*{\(\epsilon_{\mathrm{raw}} -\epsilon_{\mathrm{CP}} \)}
    \end{subfigure}
    \hfill
    \begin{subfigure}[t]{0.32\textwidth}
        \centering
        \includegraphics[width=\linewidth]{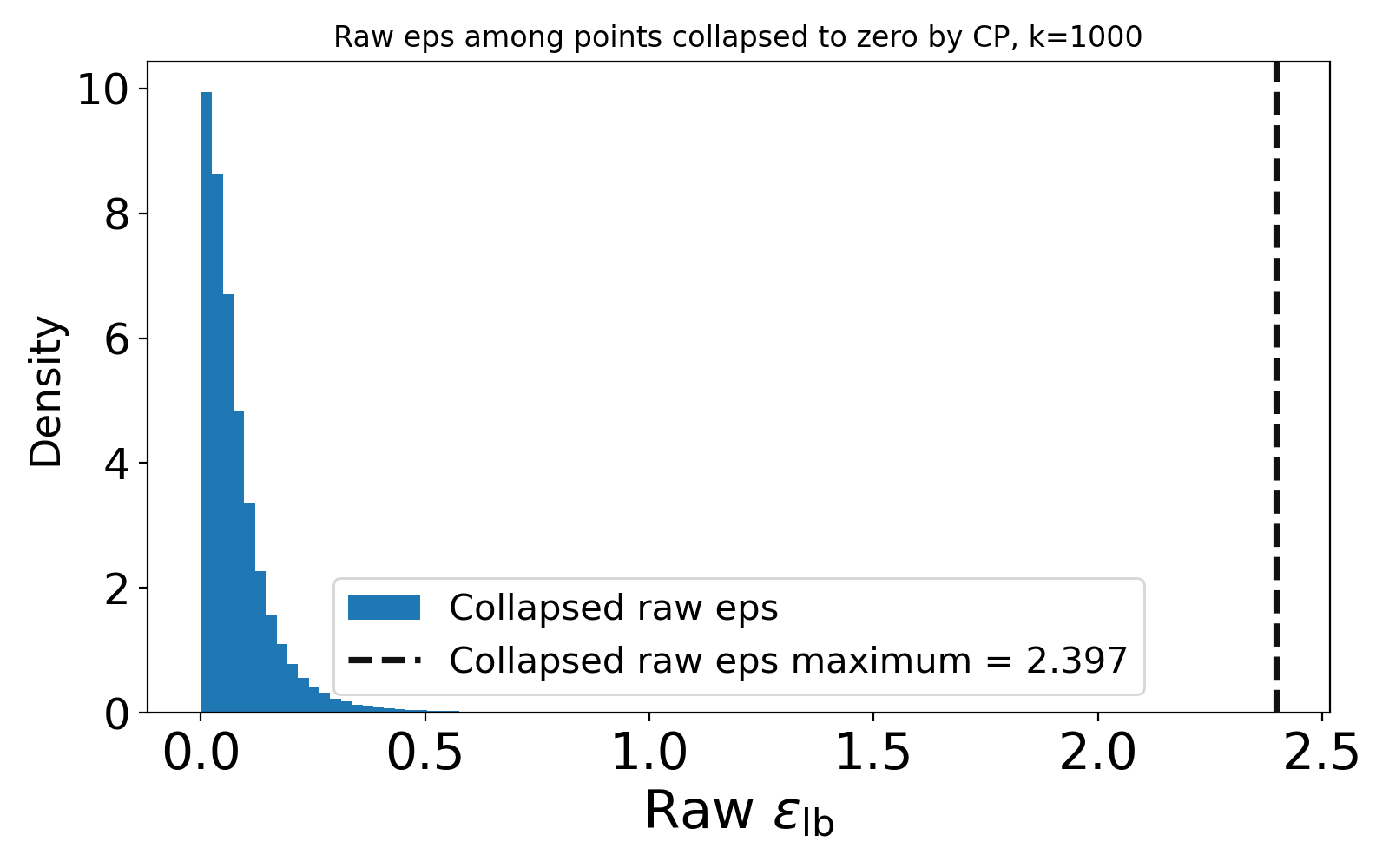}
        \caption*{\(\epsraw\) among collapse points}
    \end{subfigure}

    \vspace{0.5em}
    {\small (b) \(k=1000\)}

    \vspace{1em}

    % Row 3: k = 5000
    \begin{subfigure}[t]{0.32\textwidth}
        \centering
        \includegraphics[width=\linewidth]{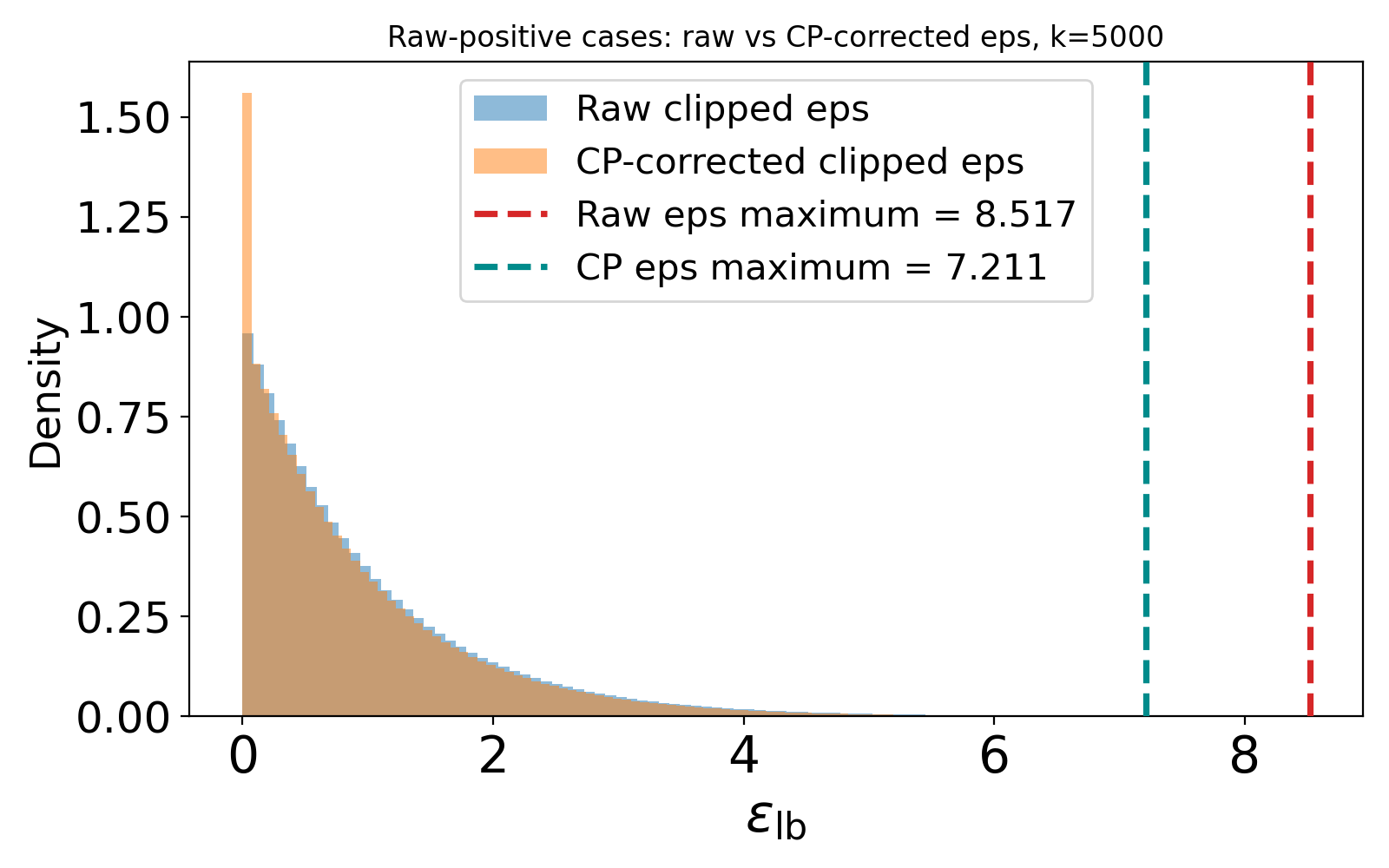}
        \caption*{\(\epsraw\) v.s. \(\epscp\)}
    \end{subfigure}
    \hfill
    \begin{subfigure}[t]{0.32\textwidth}
        \centering
        \includegraphics[width=\linewidth]{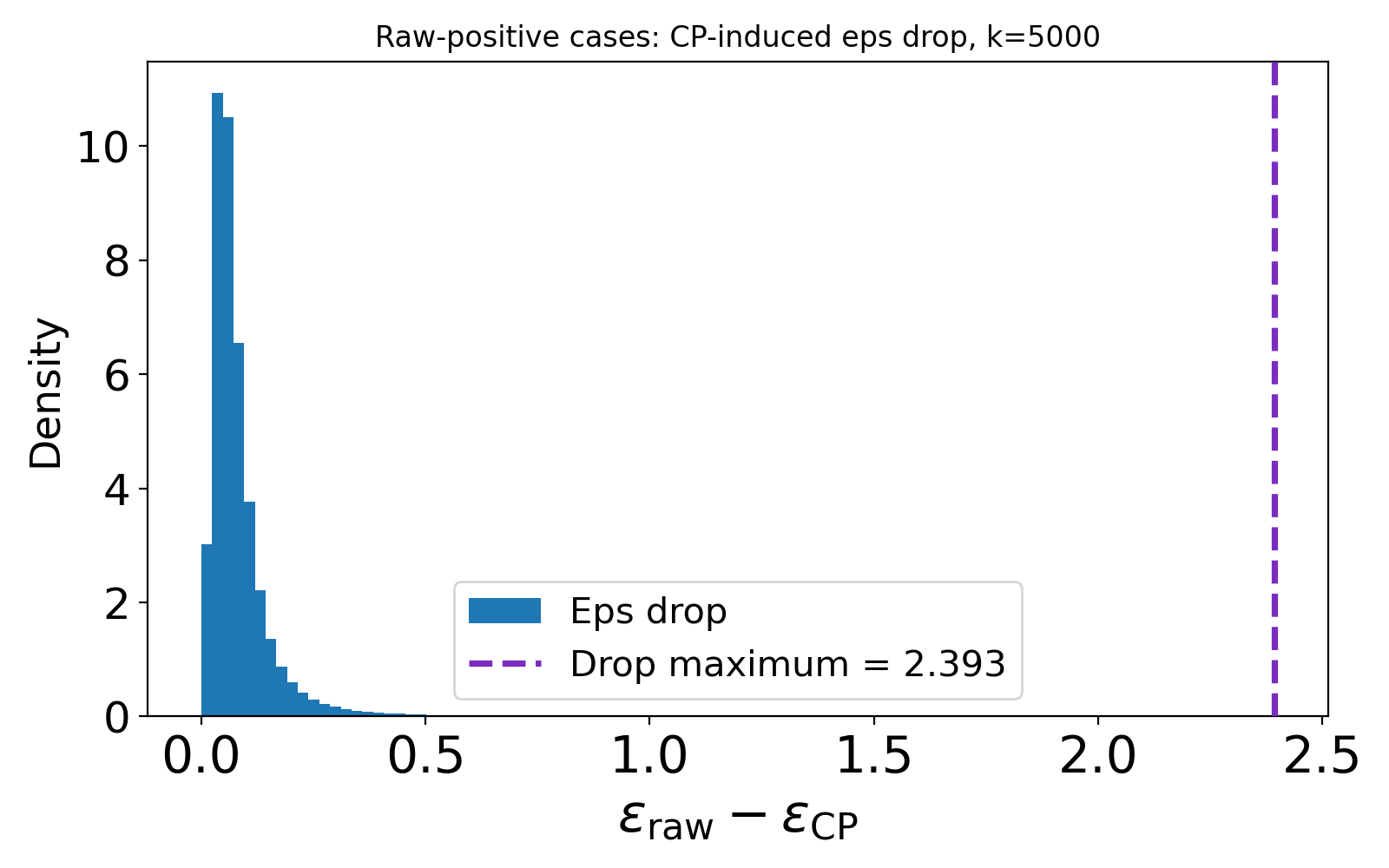}
        \caption*{\(\epsilon_{\mathrm{raw}} -\epsilon_{\mathrm{CP}} \)}
    \end{subfigure}
    \hfill
    \begin{subfigure}[t]{0.32\textwidth}
        \centering
        \includegraphics[width=\linewidth]{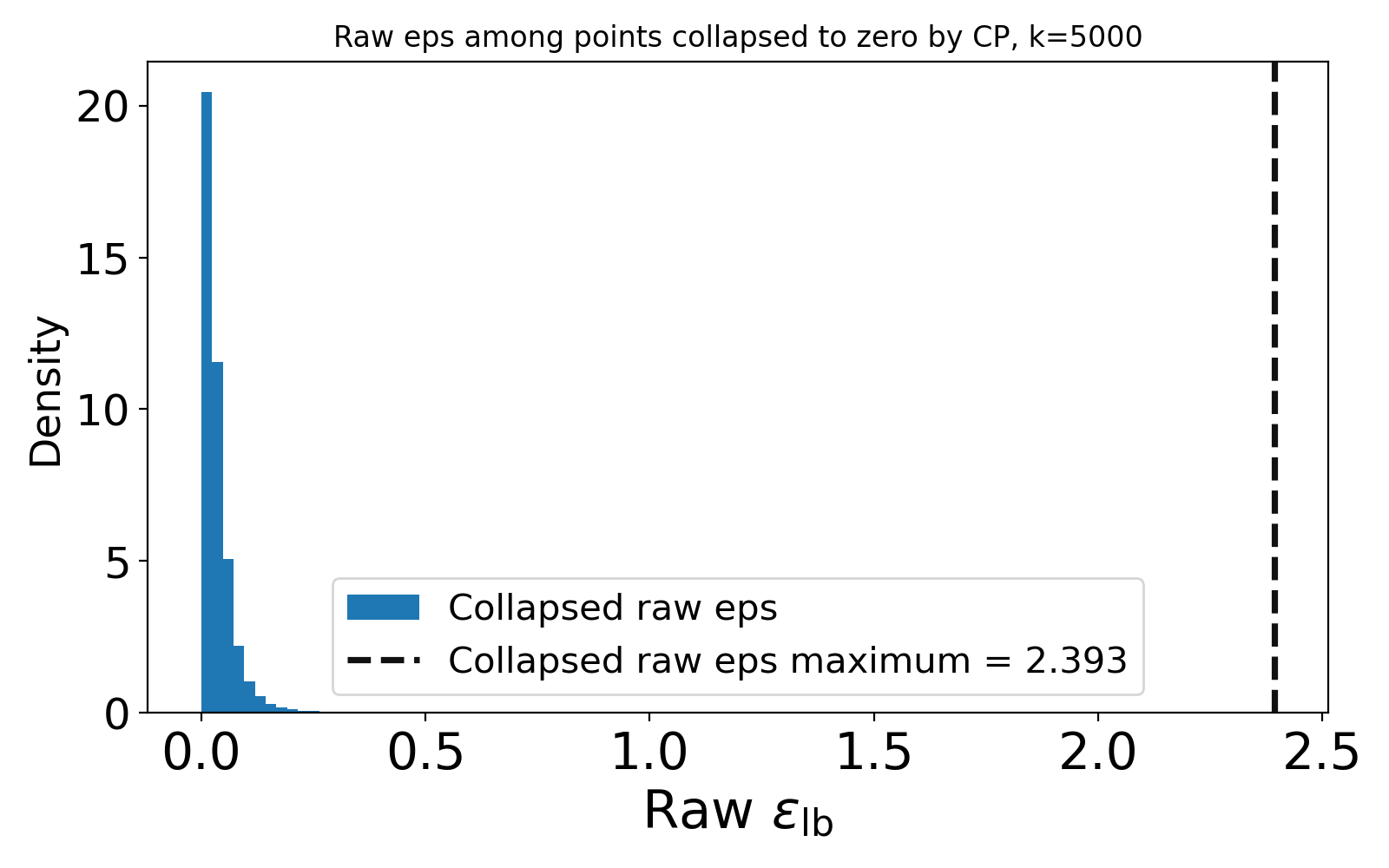}
        \caption*{\(\epsraw\) among collapse points}
    \end{subfigure}

    \vspace{0.5em}
    {\small (c) \(k=5000\)}

    \vspace{1em}

    % Row 4: k = 10000
    \begin{subfigure}[t]{0.32\textwidth}
        \centering
        \includegraphics[width=\linewidth]{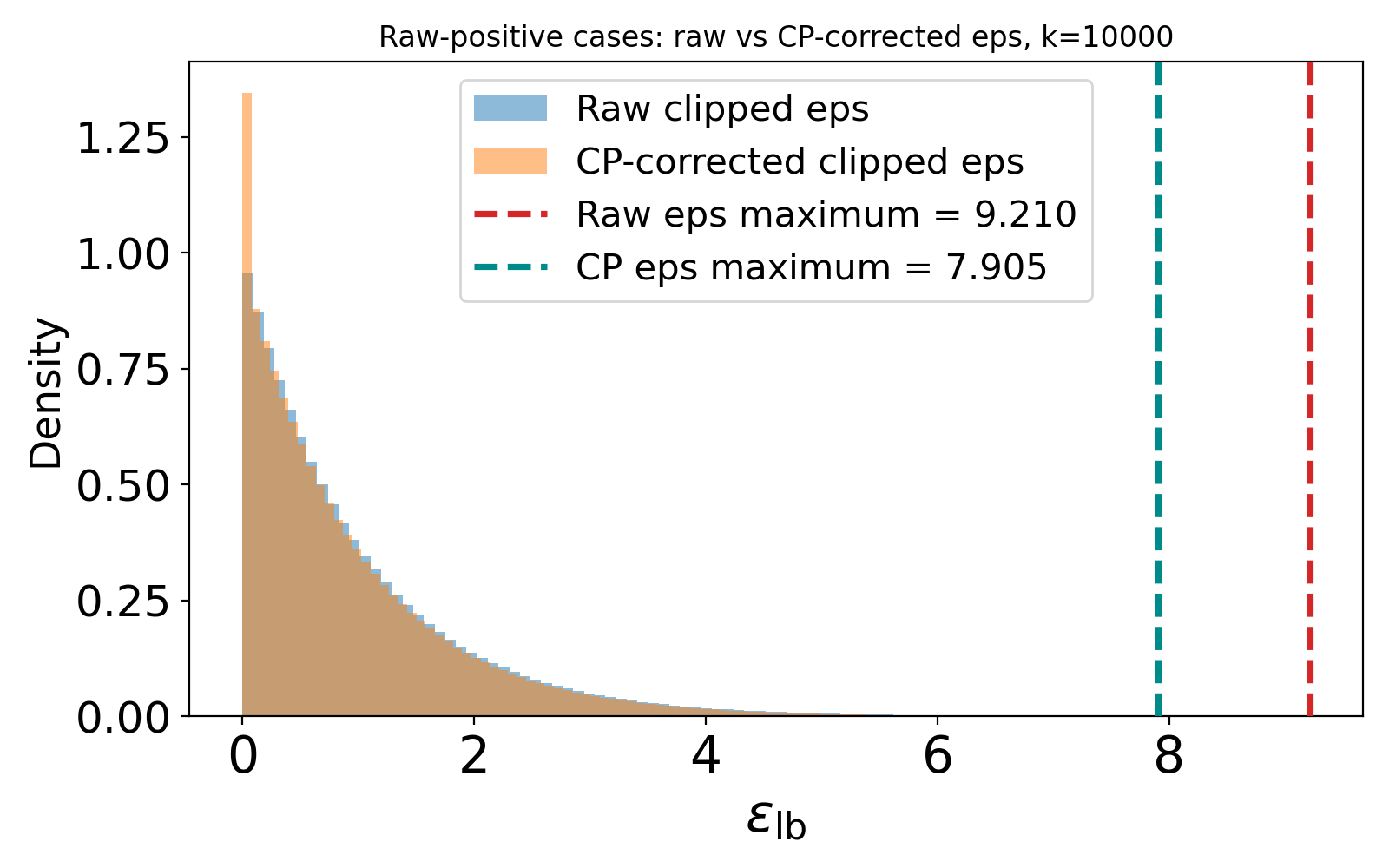}
        \caption*{\(\epsraw\) v.s. \(\epscp\)}
    \end{subfigure}
    \hfill
    \begin{subfigure}[t]{0.32\textwidth}
        \centering
        \includegraphics[width=\linewidth]{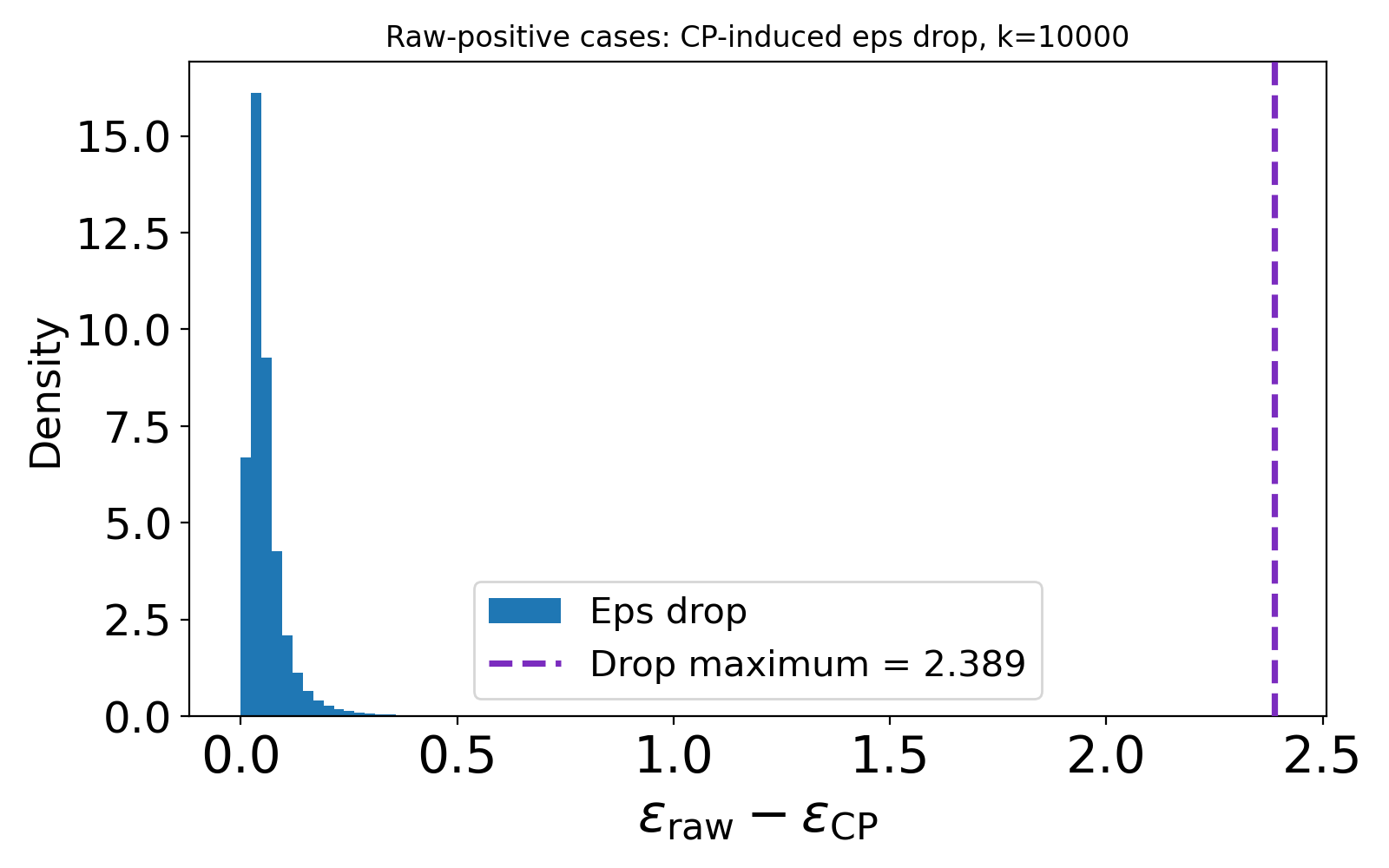}
        \caption*{\(\epsilon_{\mathrm{raw}} -\epsilon_{\mathrm{CP}} \)}
    \end{subfigure}
    \hfill
    \begin{subfigure}[t]{0.32\textwidth}
        \centering
        \includegraphics[width=\linewidth]{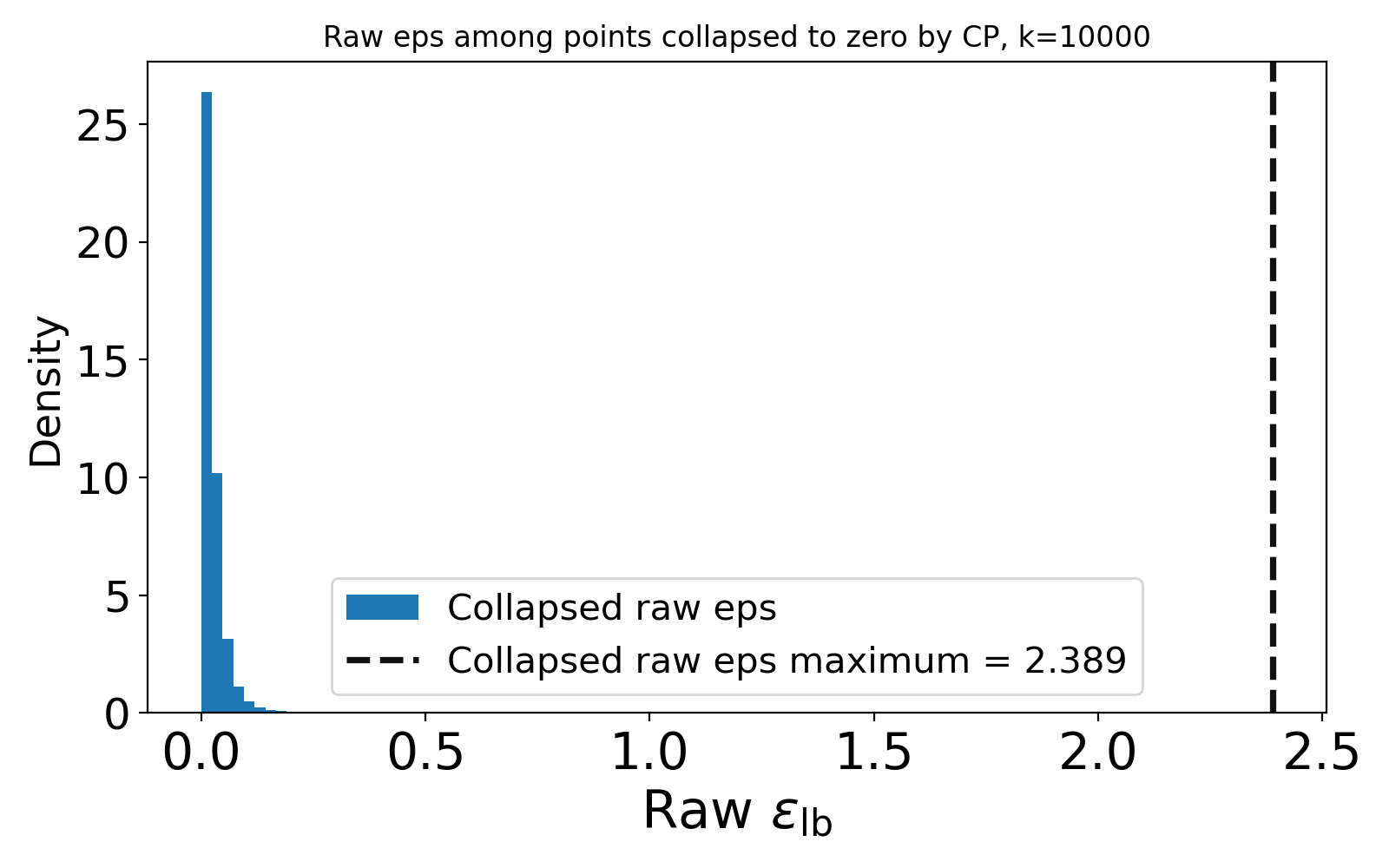}
        \caption*{\(\epsraw\) among collapse points}
    \end{subfigure}

    \vspace{0.5em}
    {\small (d) \(k=10000\)}

    \caption{
    Finite-sample effect of Clopper--Pearson (CP) correction under different sample sizes \(k\), with \(\gamma=0.05\) and \(\delta=10^{-5}\). Each row corresponds to a fixed \(k\). The left column compares the distributions of the raw lower bound \(\epsraw\) and the CP-corrected lower bound \(\epscp\); the middle column shows the distribution of the correction gap \(\epsraw-\epscp\); and the right column shows the distribution of raw lower bounds that collapse to \(0\) after CP correction.
    }
    \label{fig:cp-multik-grid}
\end{figure*}
\newpage

%% file: appendix_intro.tex
%\textbf{Appendix Description}

Our remaining appendices provide additional experimental results and ablation studies that complement our main findings. We summarize the different appendices here. 

\begin{itemize}
    \item \textbf{\ref{app:disparate_impact}: Disparate Impact.} Exploring how defense targets minority vs majority subgroups for sample removal, for varying minority/majority ratios.
%    \item \textbf{\ref{sec:nonprivate}: Empirically Private Learning.} Defense performance on non-private optimizers, and comparison against HAMP \cite{ChenP24}, a state-of-the-art empirical MIA defense.
    \item \textbf{\ref{app:alt_audits}: Alternative Input Space Audits.} Evaluations of input space audit constructions in the multiple canary setting, and the randomized canary setting of \cite{SteinkeNJ23}.
    \item \textbf{\ref{app:score}: Sample Signatures.} Ablation study comparing different scoring functions for filtering.
    \item \textbf{\ref{app:poisson}: Poisson vs.\ Shuffled Subsampling.} Comparison of DP-SGD subsampling variants, and the effect of DP-SGD implementation changes on our estimated $\epslb$.
    \item \textbf{\ref{app:scaling}: Interpolating Canaries.} Testing defense effectiveness across different scalings of canary strengths on a mixture of a blank canary and a natural image.
    \item \textbf{\ref{app:non_canary_audit}: Auditing via Non-Canary Samples.} Analyzing privacy loss on non-canary samples, motivated by the ``privacy onion effect'' observed by \cite{CarliniJZPTT22}.
    \item \textbf{\ref{app:trial_count}: Increasing Trial Count.} Validating audit methodology with 1000 shadow models.
    \item \textbf{\ref{app:params}: Additional Hyperparameter Tuning.} Ablations on global filtering, local filtering bandwidth, and filter frequency.
    \item \textbf{\ref{app:other_eps}: Empirical Privacy Results Across Audit Methods.} Fully specifying $\epslb$ computed under GDP and CP methodology with different holdout ratios, as described in Appendix~\ref{app:lbmethod}.
\end{itemize}

%% file: disparate_impact.tex
\section{Disparate Impact}\label{app:disparate_impact}

A common concern for any filtering-based algorithm is disparate impact. We evaluate our filtering defense on the ColoredMNIST dataset \cite{ArjovskyBGL19}, a variant of MNIST with 2 classes (class 0 has even digits and class 1 has odd digits) where $p\%$ of evens and $(1-p)\%$ of odds are colored red; $(1-p)\%$ of evens and $p\%$ of odds are colored blue. This creates symmetric majority and minority subgroups within both classes. Due to its distribution of relevant (shape) and irrelevant (color) features, balanced by group ratios, ColoredMNIST is a popular synthetic benchmark in the algorithmic fairness and spurious correlations literature \cite{SagawaKHL20, ZhangSZFR22}.

We consider ColoredMNIST under the following values of $p$: $99.5\%, 95\%, 90\%, 85\%, 80\%, 75\%$. For each instantiation of ColoredMNIST, we evaluate how many samples are filtered by our defense from each of the 4 subgroups, and the resulting impact on groupwise model utility. 

Table \ref{tab:defense_removal} indicates that our defense does target minority groups over majority groups when marking samples for removal, yielding a disparate reduction in minority group utility. This targeting effect becomes more benign as the fraction $p$ becomes more mild, as expected. 

Moreover, 
we find that the observed bias towards minority group removal is in some sense definitional from the perspective of sample memorization, which any method with a low auditable $\epslb$ must combat. This finding is thematically consistent with other works investigating tradeoffs between privacy and fairness \cite{BagdasaryanPS19, UniyalNKSKMT21}. To formalize this correlation, we audit ColoredMNIST (at $p=75\%$) without our defense twice, once with the canary as a member of a minority subgroup and once again with the canary as a member of a majority subgroup. Under the minority-group canary, we observe $\epslb=1.145$ whereas under the majority-group canary, we observe $\epslb=0.043$, indicating that minority-group canaries are indeed more prone to memorization.

\begin{table}[ht]
\centering
\caption{Defense Removal Rate (\%): mean $\pm$ std over repetitions}
\label{tab:defense_removal}
\begin{tabular}{c|cccc}
\hline
majority\_pct & c0\_red (maj) & c0\_blue (min) & c1\_red (min) & c1\_blue (maj) \\
\hline
0.995 & $1.2 \pm 0.0\%$ & $100.0 \pm 0.0\%$ & $100.0 \pm 0.0\%$ & $1.1 \pm 0.0\%$ \\
0.950 & $0.3 \pm 0.0\%$ & $27.6 \pm 0.2\%$  & $21.4 \pm 0.2\%$  & $0.6 \pm 0.0\%$ \\
0.900 & $0.5 \pm 0.0\%$ & $12.2 \pm 0.1\%$  & $8.6 \pm 0.1\%$   & $0.8 \pm 0.0\%$ \\
0.850 & $0.6 \pm 0.0\%$ & $7.9 \pm 0.1\%$   & $5.4 \pm 0.1\%$   & $0.9 \pm 0.0\%$ \\
0.800 & $0.7 \pm 0.0\%$ & $5.7 \pm 0.1\%$   & $4.1 \pm 0.1\%$   & $1.0 \pm 0.0\%$ \\
0.750 & $0.8 \pm 0.0\%$ & $4.3 \pm 0.1\%$   & $3.2 \pm 0.1\%$   & $1.1 \pm 0.0\%$ \\
\hline
\end{tabular}
\end{table}

%% file: ablations.tex
\section{Alternative Input Space Audits}\label{app:alt_audits}

\subsection{Single run audit}\label{app:single_run_audit}

A recent work \cite{SteinkeNJ23} provides a computationally cheaper alternative to LiRA-style audits by requiring only a single model training run instead of $2T$ shadow models ($T$ being the number of repetitions). This audit measures individual privacy by generating $N$ canaries and independently selecting $M$ of them at random for inclusion in the training dataset. After training, we compute a score (cross entropy loss) for each of the original $N$ canaries. We then make $k_{+}$ guesses about which canaries were included in the training and $k_{-}$ guesses about which were excluded, abstaining from guessing on the remaining canaries. The TPR and FPR from these guesses are used to compute an $\epslb$ via the $f$-DP framework. We evaluate our defense under the multi-canary individual privacy audit using 500 randomly mislabeled canaries in CIFAR-10. We perform a grid search over all pairs ($k_{+}$, $k_{-}$) and report the maximum $\epslb$ achieved (Table~\ref{tab:individual_group_audits}). With our defense, we filter out approximately $75\%$ of the canaries, substantially reducing the $\epslb$.

\subsection{Group privacy audit}\label{app:group_privacy_audit}

We extended LiRA to group privacy by using the maximum loss across all canaries as the audit score for distinguishing between models trained versus without the canary group. We evaluate our defense under this threat model using $k = 500$ mislabeled canaries on CIFAR-10 (Table~\ref{tab:individual_group_audits}). We evaluate privacy using $k$-group privacy $\epslb$; as discussed in Section~\ref{ssec:setup}, this value does not formally imply a $1$-group privacy $\epslb$ that is $k\times$ smaller, but this is a reasonable scaling to expect. We find that even in the group privacy setting, our defense substantially reduces $\epslb$.

\begin{table}[h]
    \centering
    \begin{tabular}{lcccc}
    \hline
    \textbf{Experiment} & \textbf{Train Acc (\%)} & \textbf{Test Acc (\%)} & \textbf{$\varepsilon_{\mathrm{lb}}$} \\
    \hline
    Mislabeled Single Run Audit, No Defense    & 75.62 & 70.81 & 0.20 \\
    Mislabeled Single Run Audit, Defense       & 74.73 & 70.42 & 0.10 \\
    \hline
    Mislabeled Group Privacy Audit, No Defense & 76.13 & 71.65 & 3.54 \\
    Mislabeled Group Privacy Audit, Defense    & 74.69 & 70.62 & 2.67 \\
    \hline
    \end{tabular}
    \caption{Empirical privacy loss under individual and group privacy audits.}
    \label{tab:individual_group_audits}
\end{table}

\section{Sample Signatures}\label{app:score}

This experiment tested the sensitivity of the local filtering defense to the choice of scoring function used to identify high-risk samples. 
Using these scoring functions, we can detect backdoor canaries that spectral methods \cite{TranLM18} may fail to identify. For each attack, the scoring function should identify and filter out canaries as early as possible in training. Filtering canaries too late is ineffective since they have already been memorized. Future work could address late-stage canary detection through privatized gradient ascent.

On CIFAR-10, we considered the following scoring functions.

\begin{enumerate}
    \item Gradient cosine similarity: how aligned sample gradient direction is with the overall parameter update direction from initialization.
    \item $L_2$ gradient norm: magnitude of their clipped per-sample gradient under the $L_2$ norm. This is computed using both clipped and unclipped gradients.
    \item $L_\infty$ gradient norm: magnitude of the clipped per-sample gradient under the $L_\infty$ norm. This is computed using both clipped and unclipped gradients.
    \item Gradient kurtosis unclipped: kurtosis of the gradient distribution, emphasizing heavy-tailed or outlier-like gradients.
    \item Directional uniqueness: combines gradient magnitude with a proxy for directional “uniqueness” based on how atypical a sample’s gradient direction is relative to a running history.
    \item Prediction entropy: entropy of the predicted probability distribution, emphasizing high-uncertainty predictions.
    \item Prediction margin: inverted gap between the predicted probability of the true class and the strongest competing class.
    \item Random projection variability: projects gradients onto random directions and measures variability in those projections, capturing atypical gradient structure.
\end{enumerate}

All other defense parameters were kept fixed (filtering bandwidth, gradient ascent step, etc.). Across the scoring function choices, the defense consistently detected and removed the canary, with low empirical privacy loss and utility metrics similarly high across configurations (Table~\ref{tab:scoring_function_ablation}). This allows practitioners to select or tune the scoring function that best matches their constraints and threat model without changing the defense behavior. 

\begin{table}[t]
    \centering
    \small
    \begin{tabular}{lccc}
        \toprule
        Scoring Function & Train Acc (\%) & Test Acc (\%) & $\varepsilon_{\mathrm{lb}}$ \\
        \midrule
        Cosine similarity ($\theta_0$)         & 75.85 & 71.05 & 1.11 \\
        Grad norm clipped ($L_{2}$)                    & 75.14 & 70.47 & 1.37 \\
        Grad norm clipped ($L_{\infty}$)               & 75.28 & 70.54 & 0 \\
        Grad norm unclipped ($L_{2}$)          & 75.25 & 70.73 & 1.19 \\
        Grad norm unclipped ($L_{\infty}$)     & 75.20 & 70.58 & 0 \\
        Gradient kurtosis unclipped            & 75.49 & 70.62 & 0 \\
        Directional uniqueness       & 75.26 & 70.75 & 1.20 \\
        Pred.\ entropy               & 75.11 & 70.71 & 0.25 \\
        Pred.\ margin                & 74.97 & 70.49 & 1.11 \\
        Rand.\ proj.\ variability     & 75.44 & 70.83 & 1.18 \\
        \bottomrule
    \end{tabular}
    \caption{Scoring function ablation on CIFAR-10 (CNN).}
    \label{tab:scoring_function_ablation}
\end{table}

\section{Poisson vs.\ Shuffled Subsampling}\label{app:poisson}

As discussed in Section~\ref{ssec:variant}, in practice, many large-scale implementations of DP-SGD favor shuffled mini-batching over Poisson sampling for efficiency reasons; we therefore verify that our defense remains effective under this more practical sampling scheme. Concretely, we compared the results of estimating $\epslb$ after running our defense versus no defense on a blank canary, for the MNIST dataset, with both Poisson sampling and shuffled subsampling (Figure~\ref{fig:new_poisson_vs_shuffled_epsilon}). The results show that the gap in the estimated $\epslb$ persists regardless of sampling strategy.

\begin{figure}
    \centering
    \includegraphics[width=0.75\linewidth]{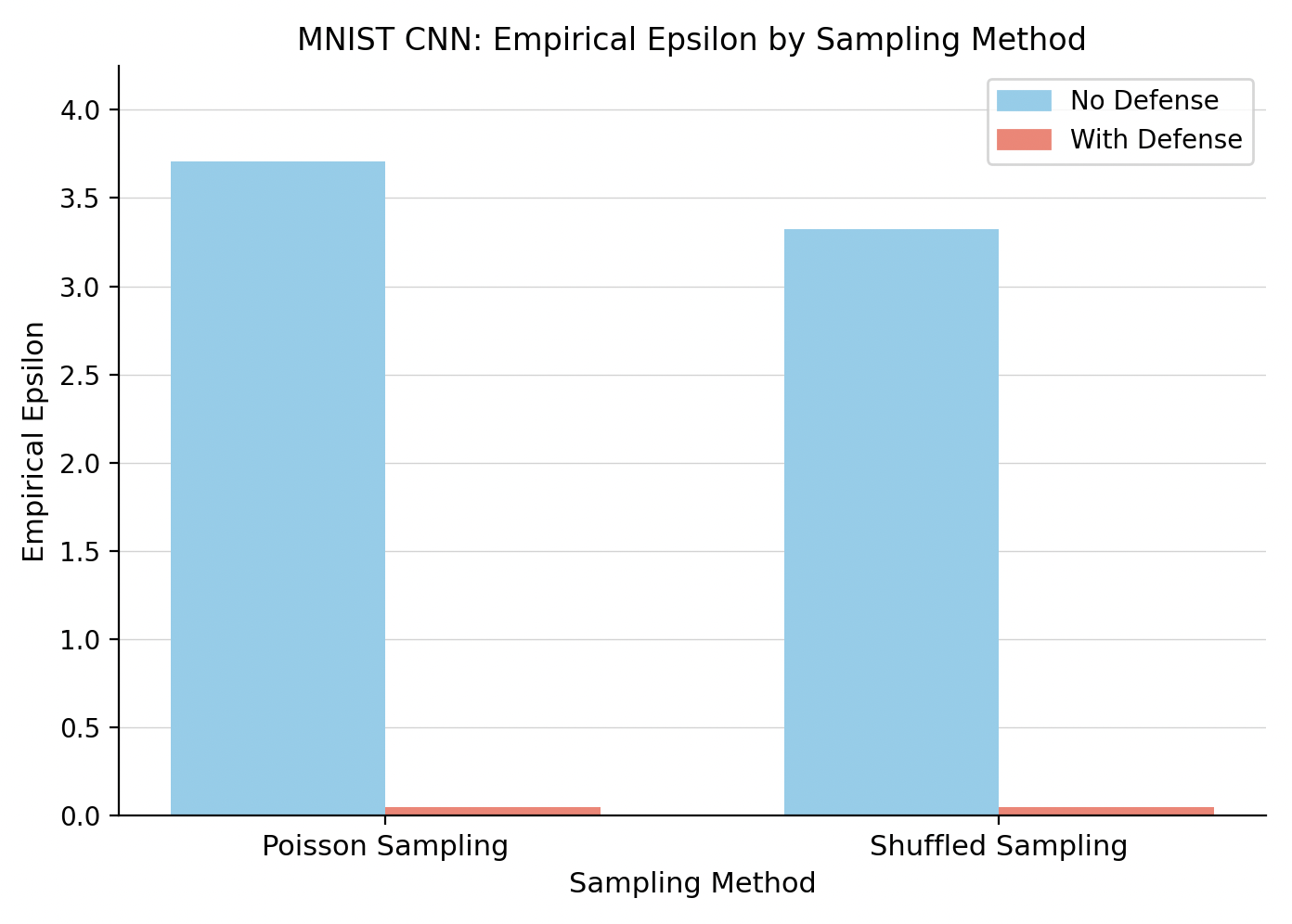}
    \caption{$\epslb$ with/without defense on MNIST (CNN) across Poisson and shuffled sampling methods.}
    \label{fig:new_poisson_vs_shuffled_epsilon}
\end{figure}

\section{Interpolating Canaries}\label{app:scaling}

We linearly interpolated a blank canary with a mislabeled canary with varying degrees of interpolation to test how the effectiveness of the defense scales with canary strength. We constructed canaries as ($\alpha \cdot \text{mislabeled sample}$). When $\alpha$ is 0 the canary is a pure blank image; when $\alpha$ is 1, it is a fully mislabeled image. We evaluated the setup on both MNIST and CIFAR-10 (Table~\ref{tab:alpha_ablation}).

We found that varying $\alpha$ did not produce a clear change in privacy loss. This suggests that privacy loss is not necessarily proportional to the ``strength'' of the signal, and the defense remained effective across the different canary constructions.

\begin{table}[h]
    \centering
    \small
    \begin{tabular}{llccccc}
        \toprule
        Dataset & Defense Setting & $\alpha=0$ & $\alpha=0.25$ & $\alpha=0.5$ & $\alpha=0.75$ & $\alpha=1.0$ \\
        \midrule
        CIFAR-10 & No Defense & 0.46 & 0.42 & 0.30 & 0.35 & 0.73 \\
        CIFAR-10 & Defense    & 0    & 0    & 0    & 0.11 & 0    \\
        \midrule
        MNIST    & No Defense & 2.45 & 2.45 & 2.50 & 2.45 & 2.50 \\
        MNIST    & Defense    & 0.59 & 0.58 & 0.62 & 0.71 & 0.69 \\
        \bottomrule
    \end{tabular}
    \caption{Empirical $\varepsilon_{\mathrm{lb}}$ across interpolated canary strengths on CIFAR-10 and MNIST (CNN).}
    \label{tab:alpha_ablation}
\end{table}

\section{The Privacy Onion Effect: Do We Expose Anyone Else?}\label{app:non_canary_audit}

Prior work \cite{CarliniJZPTT22} finds that removing the most privacy vulnerable samples from a dataset may increase the vulnerability of retained samples rather than yielding an aggregate privacy improvement. This ``privacy onion'' effect suggests that empirical privacy can offer a false sense of security and practitioners ought to employ algorithms with formal privacy guarantees rather than optimizing on empirical privacy. This phenomenon is algorithm-specific, so we analyze its presence for our specific filtering-based defense in this section, by auditing \emph{non-canary samples}.

To empirically test our defense, we trained 1000 shadow models (trained with DP-SGD $\epsilon=10$; 500 with the blank canary, 500 without) and recorded per-sample loss under each final model for every sample in the MNIST training dataset. For each non-canary sample, we compute the $\epslb$ post-defense and compare this to the maximum $\epslb$ pre-defense (corresponding to the canary's $\epslb$).
Our results are striking: the (adaptively-chosen) strongest non-canary audit produces a much weaker $\epslb$ than a fixed canary audit. This is true whether or not the formal holdout (false discovery correction) method in Appendix~\ref{app:lbmethod} is used, on both the canary audit and the non-canary audit; in the holdout case, the auditable $\epslb$ of the best non-canary audit drops to $0$.

Our audited $\epslb$ comparing the canary with the strongest non-canary sample are displayed in Figures~\ref{fig:ultra_new_do_we_expose_1}~and~\ref{fig:ultra_new_do_we_expose_2}. The experiments in these figures respectively apply the multiple hypothesis heuristic from Appendix~\ref{app:lbmethod}, and use a holdout set to estimate $\epslb$, for all audits.

\begin{figure}
    \centering
    \includegraphics[width=0.75\linewidth]{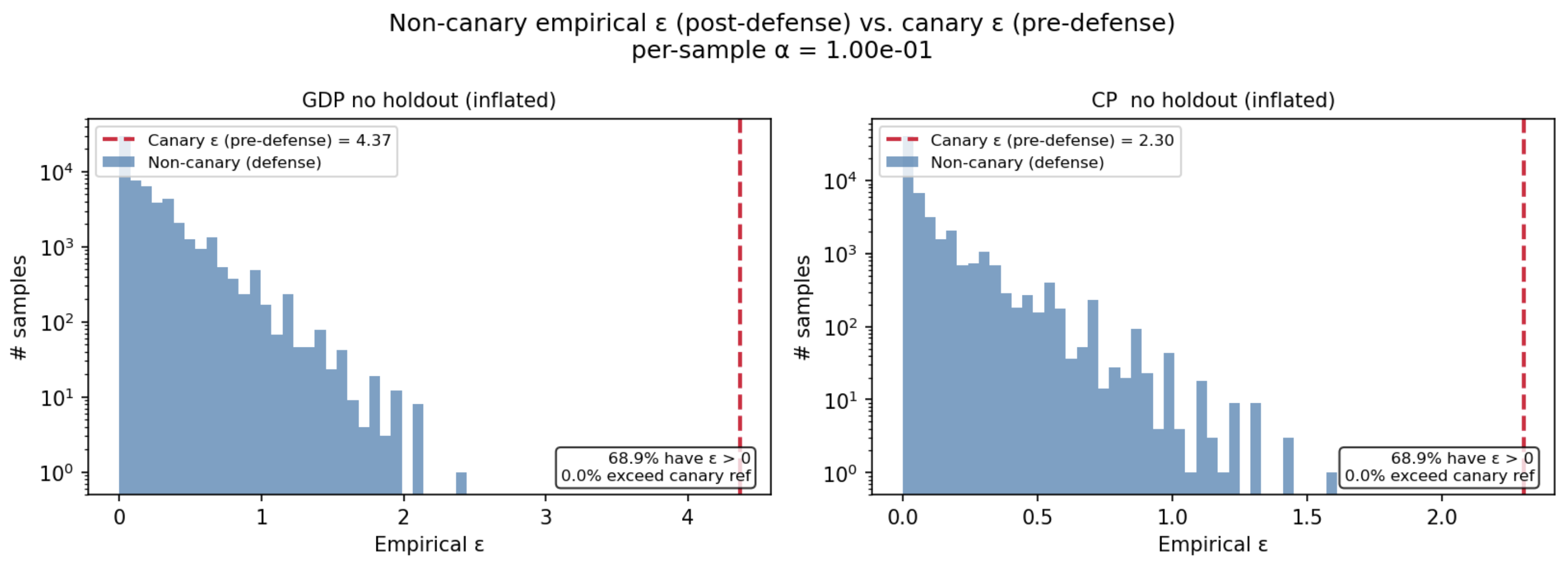}
    \caption{Distribution of empirical $\epslb$ for non-canary training samples 
(post-defense) vs.\ canary $\epslb$ (pre-defense) across two auditing methods. 
In both cases, 0.0\% of non-canary samples exceed the canary reference.}
    \label{fig:ultra_new_do_we_expose_1}
\end{figure}

\begin{figure}
    \centering
    \includegraphics[width=0.75\linewidth]{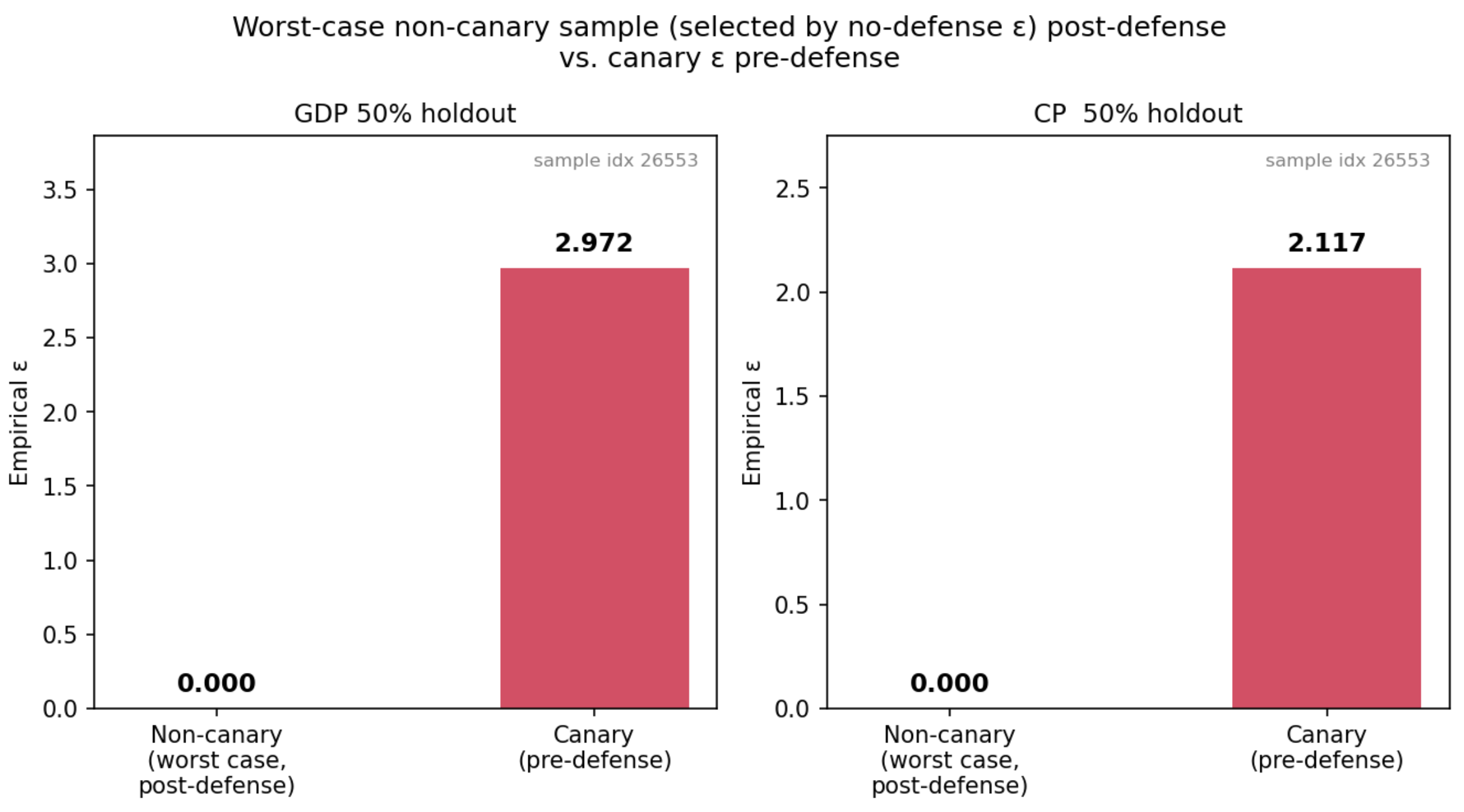}
    \caption{Empirical $\epslb$ of the worst-case non-canary sample post-defense 
vs.\ canary $\epslb$ pre-defense, under GDP 50\% holdout and CP 50\% holdout. 
The worst-case non-canary sample is selected as the sample achieving the highest empirical 
$\epslb$ under no defense; its empirical $\epslb$ 
drops to 0.000 post-defense under both auditing methods.}
    \label{fig:ultra_new_do_we_expose_2}
\end{figure}

\section{Increasing the Trial Count}\label{app:trial_count}

Our primary experiments used 400 shadow models (200 with the canary included, 200 without) to balance computational constraints. However, the Clopper-Pearson confidence intervals used to compute empirical privacy bounds are looser with fewer trials. We conducted a scaled-up experiment on MNIST with a blank canary using 1000 shadow models (500 per condition) instead of the standard 400. We compared the $\epslb$ and model utility between defense and non-defense settings under the higher sample regime (Table~\ref{tab:1000_reps}). The $\epslb$ gap with and without the defense remained consistent with the 400 trial experiments.

\begin{table}[h]
    \centering
    \small
    \begin{tabular}{lcccc}
        \toprule
        Defense Setting & Train Acc (\%) & Test Acc (\%) & $\epslb$ (400 Trials) & $\epslb$ (1000 Trials) \\
        \midrule
        No Defense & 98.80 & 98.64 & 3.71 & 4.01 \\
        Defense    & 97.60 & 97.89 & 0 & 0.09 \\
        \bottomrule
    \end{tabular}
    \caption{Empirical privacy loss on MNIST (CNN) with 1000 shadow models.}
    \label{tab:1000_reps}
\end{table}

\section{Additional Hyperparameter Tuning}\label{app:params}

\subsection{Global Filtering Ablation}\label{app:global_filter}

We conducted experiments on MNIST using a blank input space canary to evaluate how global filtering (selecting the top-$k$ samples across the entire dataset) compares with local filtering (selecting the top-$k$ samples within each class). Our experiments bias towards local filtering to align with prior work \cite{TranLM18}, which found that backdoored data poisons exhibit a stronger signature when compared relative to other samples in the same class rather than the entire dataset. We tested three global filtering bandwidths: $k \in \{10, 25, 50\}$ samples filtered per epoch (Table~\ref{tab:global_filter_ablation}). For comparison, our default per-class filtering with $k=5$ filters, for 50 samples per epoch in total (5 samples per class across 10 classes). 

Global filtering achieved similar empirical privacy and test accuracy compared to per-class filtering across all bandwidth configurations. These results suggest that practitioners can choose between filtering strategies based on implementation convenience or specific dataset characteristics rather than privacy-utility tradeoffs alone. For instance, a model trainer may prefer global filtering to limit the total number of discarded samples, or local filtering when class imbalance is a concern. Future work could explore adaptive scheduling approaches that monitor training dynamics and switch between global and local filtering on a per-epoch basis.

\begin{table}[t]
    \centering
    \small
    \begin{tabular}{lccc}
        \toprule
        Global Filtering $k$ Value & Train Acc (\%) & Test Acc (\%) & $\epslb$ \\
        \midrule
        10 & 0.984 & 0.984 & 0 \\
        25 & 0.980 & 0.981 & 0.015 \\
        50 & 0.975 & 0.977 & 0.090 \\
        \bottomrule
    \end{tabular}
    \caption{Global filtering ablation on MNIST (CNN).}
    \label{tab:global_filter_ablation}
\end{table}

\subsection{Varying Local Filtering Defense Bandwidth Ablation}\label{app:bandwidth_ablation}

To empirically evaluate how the bandwidth parameter $k$ (the number of samples filtered per class per epoch) affects the defense efficacy, we conducted experiments on CIFAR-10 using a blank canary with gradient-norm based local filtering. We tested bandwidth values $k \in \{1, 2, 3, 4, 5\}$ (Table~\ref{tab:local_filter_ablation}). The defense reduced empirical privacy loss across all bandwidth configurations, though the degree of reduction varied with $k$. These results suggest that the choice of k does have some effect on empirical privacy, and practitioners should consider tuning it alongside computational and utility constraints.

\begin{table}[t]
    \centering
    \small
    \begin{tabular}{lccc}
        \toprule
        Local Filtering $k$ Value & Train Acc (\%) & Test Acc (\%) & $\varepsilon_{\mathrm{lb}}$ \\
        \midrule
        1 & 76.01 & 71.07 & 0.44 \\
        2 & 75.92 & 71.11 & 0.52 \\
        3 & 75.52 & 70.62 & 0.43 \\
        4 & 75.28 & 70.60 & 0.12 \\
        5 & 75.27 & 70.77 & 0.36 \\
        \bottomrule
    \end{tabular}
    \caption{Local filtering bandwidth ablation on CIFAR-10 (CNN).}
    \label{tab:local_filter_ablation}
\end{table}

\subsection{Varying Filter Frequency Ablation}\label{app:filter_frequency_ablation}

In this experiment, we varied the filtering frequency during training, applying the defense every $k$ epochs for $k \in \{1, 5, 10, 20\}$. Experiments used the CIFAR-10/CNN architecture with a blank input space canary. 
We observed that filtering frequency has a meaningful impact on empirical privacy outcomes, with more frequent filtering achieving the lowest $\epsilon$lb of 0.43, while less frequent filtering led to substantially higher loss, peaking at 1.39 when filtering every 10 epochs (Table~\ref{tab:filter_frequency_ablation}). These results suggest that practitioners should prefer more frequent filtering when empirical privacy is a priority, though utility metrics remained stable across all configurations.
Filtering frequency may become a more critical hyperparameter for complex canaries or auditing threat models, where model trainers seek to limit the number of discarded samples during training. Future work could explore adaptive scheduling approaches that monitor training dynamics and apply heuristics to determine when filtering is necessary on a per-epoch basis.

\begin{table}[t]
    \centering
    \small
    \begin{tabular}{lccc}
        \toprule
        Epochs Before Filtering & Train Acc (\%) & Test Acc (\%) & $\epslb$ \\
        \midrule
        1  & 75.21 & 70.77 & 0.43 \\
        5  & 75.74 & 70.73 & 1.21 \\
        10 & 76.19 & 71.13 & 1.39 \\
        20 & 76.06 & 70.89 & 1.19 \\
        \bottomrule
    \end{tabular}
    \caption{Effect of filtering frequency on CIFAR-10 (CNN).}
    \label{tab:filter_frequency_ablation}
\end{table}

\section{Full Empirical Privacy Results Across Audit Methods}\label{app:other_eps}

In all experiments throughout the paper, we report $\epslb$ using the GDP no holdout auditing method as our primary metric. For completeness, we report full results across all auditing configurations. For GDP, we vary the fraction of shadow models held out for threshold selection at $25\%$, $50\%$, and $75\%$. For CP, we include results with no holdout as well as the same three holdout levels. These additional metrics provide a more comprehensive picture of the empirical privacy landscape across all experiments, and allow the reader to assess the robustness of our conclusions to the choice of auditing method.

% ============================================================
% Tradeoff Curves
% ============================================================
\begin{table}[t]
    \centering
    \small
    \resizebox{\textwidth}{!}{
    \begin{tabular}{llcccccccc}
        \toprule
        Dataset (Model) & $\varepsilon_{\mathrm{ub}}$ & Setting & GDP 25\% & GDP 50\% & GDP 75\% & CP no holdout & CP 25\% & CP 50\% & CP 75\% \\
        \midrule
        MNIST (CNN)             & 2  & No Defense & 0.000 & 0.000 & 0.000 & 0.000  & 0.000 & 0.000 & 0.000 \\
                                &    & Defense    & 0.000 & 0.000 & 0.000 & 0.015  & 0.000 & 0.000 & 0.000 \\
                                & 4  & No Defense & 0.000 & 0.000 & 0.011 & 0.398  & 0.000 & 0.000 & 0.000 \\
                                &    & Defense    & 0.000 & 0.000 & 0.000 & 0.000  & 0.000 & 0.000 & 0.000 \\
                                & 6  & No Defense & 0.575 & 0.000 & 0.688 & 0.600  & 0.000 & 0.000 & 0.138 \\
                                &    & Defense    & 0.000 & 0.000 & 0.000 & 0.000  & 0.000 & 0.000 & 0.000 \\
                                & 8  & No Defense & 1.054 & 0.424 & 1.851 & 0.840  & 0.000 & 0.000 & 0.544 \\
                                &    & Defense    & 0.000 & 0.000 & 0.000 & 0.000  & 0.000 & 0.000 & 0.000 \\
                                & 10 & No Defense & 0.773 & 1.580 & 2.708 & 1.066  & 0.071 & 0.184 & 0.669 \\
                                &    & Defense    & 0.000 & 0.000 & 0.000 & 0.000  & 0.000 & 0.000 & 0.000 \\
        \midrule
        Purchase (MLP)          & 2  & No Defense & 0.569 & 0.183 & 0.000 & 0.164  & 0.028 & 0.000 & 0.000 \\
                                &    & Defense    & 0.000 & 0.000 & 0.000 & 0.000  & 0.000 & 0.000 & 0.000 \\
                                & 4  & No Defense & 0.000 & 0.984 & 0.493 & 0.739  & 0.002 & 0.155 & 0.000 \\
                                &    & Defense    & 0.000 & 0.000 & 0.000 & 0.045  & 0.000 & 0.000 & 0.000 \\
                                & 6  & No Defense & 0.008 & 1.330 & 0.000 & 0.465  & 0.000 & 0.000 & 0.000 \\
                                &    & Defense    & 0.000 & 0.000 & 0.000 & 0.025  & 0.000 & 0.000 & 0.000 \\
                                & 8  & No Defense & 0.000 & 0.000 & 0.496 & 0.961  & 0.000 & 0.000 & 0.000 \\
                                &    & Defense    & 0.000 & 0.197 & 0.000 & 0.132  & 0.000 & 0.000 & 0.000 \\
                                & 10 & No Defense & 0.000 & 2.047 & 1.205 & 0.753  & 0.000 & 0.397 & 0.314 \\
                                &    & Defense    & 1.281 & 1.278 & 0.902 & 0.503  & 0.205 & 0.263 & 0.000 \\
        \midrule
        CIFAR-10 (CNN)          & 2  & No Defense & 0.000 & 0.000 & 0.000 & 0.017  & 0.000 & 0.000 & 0.000 \\
                                &    & Defense    & 0.000 & 0.000 & 0.000 & 0.000  & 0.000 & 0.000 & 0.000 \\
                                & 4  & No Defense & 0.000 & 0.000 & 0.000 & 0.000  & 0.000 & 0.000 & 0.000 \\
                                &    & Defense    & 0.000 & 0.000 & 0.000 & 0.000  & 0.000 & 0.000 & 0.000 \\
                                & 6  & No Defense & 0.000 & 0.000 & 0.233 & 0.477  & 0.000 & 0.000 & 0.000 \\
                                &    & Defense    & 0.000 & 0.000 & 0.055 & 0.232  & 0.000 & 0.000 & 0.000 \\
                                & 8  & No Defense & 0.000 & 0.523 & 0.931 & 0.570  & 0.000 & 0.000 & 0.194 \\
                                &    & Defense    & 0.000 & 0.000 & 0.074 & 0.243  & 0.000 & 0.000 & 0.000 \\
                                & 10 & No Defense & 0.000 & 0.000 & 0.000 & 0.570  & 0.000 & 0.000 & 0.000 \\
                                &    & Defense    & 0.000 & 0.000 & 0.000 & 0.000  & 0.000 & 0.000 & 0.000 \\
        \midrule
        CIFAR-10 (WideResNet-16)& 2  & No Defense & 0.000 & 0.000 & 0.000 & 0.000  & 0.000 & 0.000 & 0.000 \\
                                &    & Defense    & 0.000 & 0.000 & 0.000 & 0.000  & 0.000 & 0.000 & 0.000 \\
                                & 4  & No Defense & 0.000 & 0.000 & 0.000 & 0.145  & 0.000 & 0.000 & 0.000 \\
                                &    & Defense    & 0.000 & 0.000 & 0.000 & 0.000  & 0.000 & 0.000 & 0.000 \\
                                & 6  & No Defense & 0.000 & 0.014 & 0.479 & 0.407  & 0.000 & 0.000 & 0.071 \\
                                &    & Defense    & 0.000 & 0.000 & 0.000 & 0.000  & 0.000 & 0.000 & 0.000 \\
                                & 8  & No Defense & 0.150 & 0.277 & 0.625 & 0.619  & 0.000 & 0.000 & 0.084 \\
                                &    & Defense    & 0.000 & 0.000 & 0.000 & 0.000  & 0.000 & 0.000 & 0.000 \\
                                & 10 & No Defense & 0.516 & 0.613 & 0.767 & 0.671  & 0.000 & 0.078 & 0.119 \\
                                &    & Defense    & 0.000 & 0.000 & 0.000 & 0.152  & 0.000 & 0.000 & 0.000 \\
        \midrule
        CIFAR-10 (CNN+AM4) & 2  & No Defense & 0.000 & 0.000 & 0.000 & 0.000  & 0.000 & 0.000 & 0.000 \\
                                &    & Defense    & 0.000 & 0.000 & 0.000 & 0.006  & 0.000 & 0.000 & 0.000 \\
                                & 4  & No Defense & 0.000 & 0.000 & 0.000 & 0.000  & 0.000 & 0.000 & 0.000 \\
                                &    & Defense    & 0.000 & 0.000 & 0.000 & 0.000  & 0.000 & 0.000 & 0.000 \\
                                & 6  & No Defense & 0.000 & 0.000 & 0.000 & 0.026  & 0.000 & 0.000 & 0.000 \\
                                &    & Defense    & 0.000 & 0.000 & 0.000 & 0.000  & 0.000 & 0.000 & 0.000 \\
                                & 8  & No Defense & 0.000 & 0.034 & 0.000 & 0.600  & 0.000 & 0.000 & 0.000 \\
                                &    & Defense    & 0.000 & 0.000 & 0.000 & 0.000  & 0.000 & 0.000 & 0.000 \\
                                & 10 & No Defense & 0.000 & 0.000 & 0.000 & 0.000  & 0.000 & 0.000 & 0.000 \\
                                &    & Defense    & 0.000 & 0.000 & 0.000 & 0.000  & 0.000 & 0.000 & 0.000 \\
        \bottomrule
    \end{tabular}
    }
    \caption{Full empirical privacy results across privacy budgets $\varepsilon_{\mathrm{ub}} \in \{2,4,6,8,10\}$, datasets, and model architectures. See Section~\ref{ssec:input}.}
    \label{tab:tradeoff_curves_full}
\end{table}

% ============================================================
% Different Scoring Functions (CIFAR-10)
% ============================================================
\begin{table}[t]
    \centering
    \small
    \resizebox{\textwidth}{!}{
    \begin{tabular}{llccccccc}
        \toprule
        Scoring Function & Setting & GDP 25\% & GDP 50\% & GDP 75\% & CP no holdout & CP 25\% & CP 50\% & CP 75\% \\
        \midrule
        Cosine similarity ($\theta_0$)    & Defense & 0.000 & 0.000 & 0.063 & 0.357 & 0.000 & 0.000 & 0.000 \\
        Grad norm clipped ($L_2$)                 & Defense & 0.000 & 0.394 & 0.000 & 0.477 & 0.000 & 0.000 & 0.000 \\
        Grad norm clipped ($L_\infty$)            & Defense & 0.000 & 0.000 & 0.000 & 0.000 & 0.000 & 0.000 & 0.000 \\
        Grad norm unclipped ($L_2$)       & Defense & 0.000 & 0.000 & 0.023 & 0.391 & 0.000 & 0.000 & 0.000 \\
        Grad norm unclipped ($L_\infty$)  & Defense & 0.000 & 0.000 & 0.000 & 0.000 & 0.000 & 0.000 & 0.000 \\
        Gradient kurtosis unclipped       & Defense & 0.000 & 0.000 & 0.000 & 0.000 & 0.000 & 0.000 & 0.000 \\
        Directional uniqueness & Defense & 0.000 & 0.394 & 0.000 & 0.415 & 0.000 & 0.000 & 0.000 \\
        Pred.\ entropy          & Defense & 0.000 & 0.000 & 0.000 & 0.000 & 0.000 & 0.000 & 0.000 \\
        Pred.\ margin           & Defense & 0.000 & 0.137 & 0.000 & 0.378 & 0.000 & 0.000 & 0.000 \\
        Rand.\ proj.\ variability & Defense & 0.000 & 0.101 & 0.210 & 0.356 & 0.000 & 0.000 & 0.000 \\
        \bottomrule
    \end{tabular}
    }
    \caption{Scoring function ablation on CIFAR-10 (CNN): full empirical privacy results across auditing methods. See Section~\ref{app:score}.}
    \label{tab:scoring_fn_full}
\end{table}
 
% ============================================================
% 1000 Reps
% ============================================================
\begin{table}[t]
    \centering
    \small
    \resizebox{\textwidth}{!}{
    \begin{tabular}{llccccccc}
        \toprule
        Dataset & Setting & GDP 25\% & GDP 50\% & GDP 75\% & CP no holdout & CP 25\% & CP 50\% & CP 75\% \\
        \midrule
        MNIST (CNN) & No Defense & 1.856 & 2.313 & 2.867 & 1.437 & 0.133 & 0.891 & 1.157 \\
                    & Defense    & 0.000 & 0.000 & 0.000 & 0.000 & 0.000 & 0.000 & 0.000 \\
        \bottomrule
    \end{tabular}
    }
    \caption{Empirical privacy loss on MNIST (CNN) with 1000 shadow models, full results across auditing methods. See Section~\ref{app:trial_count}.}
    \label{tab:1000reps_full}
\end{table}
 
% ============================================================
% Shuffled Sampling vs. Poisson Sampling
% ============================================================
\begin{table}[t]
    \centering
    \small
    \resizebox{\textwidth}{!}{
    \begin{tabular}{llccccccc}
        \toprule
        Sampling Method & Setting & GDP 25\% & GDP 50\% & GDP 75\% & CP no holdout & CP 25\% & CP 50\% & CP 75\% \\
        \midrule
        Shuffled & No Defense & 0.000 & 0.617 & 1.335 & 1.372 & 0.000 & 0.025 & 0.368 \\
                 & Defense    & 0.000 & 0.000 & 0.000 & 0.000 & 0.000 & 0.000 & 0.000 \\
        \midrule
        Poisson  & No Defense & 0.773 & 1.580 & 2.708 & 1.066 & 0.071 & 0.184 & 0.669 \\
                 & Defense    & 0.000 & 0.000 & 0.000 & 0.000 & 0.000 & 0.000 & 0.000 \\
        \bottomrule
    \end{tabular}
    }
    \caption{$\varepsilon_{\mathrm{lb}}$ with/without defense on MNIST (CNN) across Poisson and shuffled sampling methods, full results across auditing methods. See Section~\ref{app:poisson}.}
    \label{tab:sampling_full}
\end{table}
 
% ============================================================
% Mislabeled O(1) Audit / Mislabeled Multi-Canary Attack
% ============================================================
\begin{table}[t]
    \centering
    \small
    \resizebox{\textwidth}{!}{
    \begin{tabular}{lllccccccc}
        \toprule
        Audit Type & Dataset & Setting & GDP 25\% & GDP 50\% & GDP 75\% & CP no holdout & CP 25\% & CP 50\% & CP 75\% \\
        \midrule
        Group privacy (multi-canary)    & CIFAR-10 & No Defense & 0.453 & 1.180 & 2.563 & 1.376 & 0.000 & 0.449 & 0.701 \\
                                        &          & Defense    & 0.000 & 1.425 & 1.415 & 1.146 & 0.062 & 0.394 & 0.409 \\
        \bottomrule
    \end{tabular}
    }
    \caption{Full empirical privacy results under group privacy (multi-canary) mislabeled audits. See Section~\ref{app:group_privacy_audit}.}
    \label{tab:mislabeled_audits_full}
\end{table}
 
% ============================================================
% ============================================================
\begin{table}[t]
    \centering
    \small
    \resizebox{\textwidth}{!}{
    \begin{tabular}{llccccccc}
        \toprule
        Canary Type & Setting & GDP 25\% & GDP 50\% & GDP 75\% & CP no holdout & CP 25\% & CP 50\% & CP 75\% \\
        \midrule
        Majority canary & No Defense & 0.000 & 0.000 & 0.000 & 0.000  & 0.000 & 0.000 & 0.000 \\
        \midrule
        Minority canary & No Defense & 0.000 & 0.000 & 0.000 & 0.296  & 0.000 & 0.000 & 0.000 \\
        \bottomrule
    \end{tabular}
    }
    \caption{Full empirical privacy results for the fairness audit on Colored MNIST, under the no-defense setting. See Section~\ref{app:disparate_impact}.}
    \label{tab:fairness_full}
\end{table}
 
% ============================================================
% ClipBKD (MNIST and CIFAR-10)
% ============================================================
\begin{table}[t]
    \centering
    \small
    \resizebox{\textwidth}{!}{
    \begin{tabular}{llccccccc}
        \toprule
        Dataset & Setting & GDP 25\% & GDP 50\% & GDP 75\% & CP no holdout & CP 25\% & CP 50\% & CP 75\% \\
        \midrule
        MNIST (CNN)    & No Defense & 0.153 & 1.732 & 2.663 & 1.052 & 0.199 & 0.000 & 0.152 \\
                       & Defense    & 0.000 & 0.000 & 0.000 & 0.000 & 0.000 & 0.000 & 0.000 \\
        \midrule
        CIFAR-10 (CNN) & No Defense & 2.845 & 1.519 & 1.809 & 0.767 & 0.000 & 0.253 & 0.491 \\
                       & Defense    & 0.000 & 0.000 & 0.000 & 0.000 & 0.000 & 0.000 & 0.000 \\
        \bottomrule
    \end{tabular}
    }
    \caption{Full empirical privacy results under ClipBKD attacks on MNIST and CIFAR-10 (CNN). See Section~\ref{ssec:input}.}
    \label{tab:clipbkd_full}
\end{table}
 
% ============================================================
% FGSM (MNIST and CIFAR-10)
% ============================================================
\begin{table}[t]
    \centering
    \small
    \resizebox{\textwidth}{!}{
    \begin{tabular}{llccccccc}
        \toprule
        Dataset & Setting & GDP 25\% & GDP 50\% & GDP 75\% & CP no holdout & CP 25\% & CP 50\% & CP 75\% \\
        \midrule
        MNIST (CNN)    & No Defense & 0.000 & 1.452 & 2.404 & 1.306 & 0.000 & 0.340 & 0.543 \\
                       & Defense    & 0.000 & 0.000 & 0.000 & 0.000 & 0.000 & 0.000 & 0.000 \\
        \midrule
        CIFAR-10 (CNN) & No Defense & 0.018 & 0.269 & 0.000 & 0.362 & 0.000 & 0.000 & 0.000 \\
                       & Defense    & 0.000 & 0.000 & 0.000 & 0.000 & 0.000 & 0.000 & 0.000 \\
        \bottomrule
    \end{tabular}
    }
    \caption{Full empirical privacy results under FGSM attacks on MNIST and CIFAR-10 (CNN). See Section~\ref{ssec:input}.}
    \label{tab:fgsm_full}
\end{table}
 
% ============================================================
% Mislabeled MNIST and CIFAR-10
% ============================================================
\begin{table}[t]
    \centering
    \small
    \resizebox{\textwidth}{!}{
    \begin{tabular}{llccccccc}
        \toprule
        Dataset & Setting & GDP 25\% & GDP 50\% & GDP 75\% & CP no holdout & CP 25\% & CP 50\% & CP 75\% \\
        \midrule
        MNIST (CNN)    & No Defense & 0.701 & 1.576 & 1.318 & 0.806 & 0.065 & 0.152 & 0.223 \\
                       & Defense    & 0.829 & 0.000 & 0.000 & 0.124 & 0.000 & 0.000 & 0.000 \\
        \midrule
        CIFAR-10 (CNN) & No Defense & 0.008 & 0.000 & 0.000 & 0.068 & 0.000 & 0.000 & 0.000 \\
                       & Defense    & 0.000 & 0.000 & 0.000 & 0.000 & 0.000 & 0.000 & 0.000 \\
        \bottomrule
    \end{tabular}
    }
    \caption{Full empirical privacy results under mislabeled input space attacks on MNIST and CIFAR-10 (CNN). See Section~\ref{ssec:input}.}
    \label{tab:mislabeled_full}
\end{table}
 
% ============================================================
% Gradient Canceling Attack
% ============================================================
\begin{table}[t]
    \centering
    \small
    \begin{tabular}{lccccccc}
        \toprule
        Setting & GDP 25\% & GDP 50\% & GDP 75\% & CP no holdout & CP 25\% & CP 50\% & CP 75\% \\
        \midrule
        No Defense & 1.081 & 2.476 & 2.410 & 1.760 & 0.150 & 0.872 & 0.861 \\
        Defense    & 19.564 & 24.378 & 27.130 & 4.185 & 2.785 & 3.493 & 3.903 \\
        \bottomrule
    \end{tabular}
    \caption{Full empirical privacy results under the gradient canceling attack on MNIST (CNN). See Section~\ref{ssec:gradient}.}
    \label{tab:grad_cancel_full}
\end{table}
 
% ============================================================
% Blank Alpha Ablations (CIFAR-10 and MNIST)
% ============================================================
\begin{table}[t]
    \centering
    \small
    \resizebox{\textwidth}{!}{
    \begin{tabular}{lllccccccc}
        \toprule
        Dataset & $\alpha$ & Setting & GDP 25\% & GDP 50\% & GDP 75\% & CP no holdout & CP 25\% & CP 50\% & CP 75\% \\
        \midrule
        CIFAR-10 & 0.00 & No Defense & 0.000 & 0.000 & 0.000 & 0.062 & 0.000 & 0.000 & 0.000 \\
                 &      & Defense    & 0.000 & 0.000 & 0.000 & 0.000 & 0.000 & 0.000 & 0.000 \\
                 & 0.25 & No Defense & 0.000 & 0.000 & 0.000 & 0.060 & 0.000 & 0.000 & 0.000 \\
                 &      & Defense    & 0.000 & 0.000 & 0.000 & 0.000 & 0.000 & 0.000 & 0.000 \\
                 & 0.50 & No Defense & 0.000 & 0.000 & 0.000 & 0.024 & 0.000 & 0.000 & 0.000 \\
                 &      & Defense    & 0.000 & 0.000 & 0.000 & 0.000 & 0.000 & 0.000 & 0.000 \\
                 & 0.75 & No Defense & 0.000 & 0.000 & 0.000 & 0.035 & 0.000 & 0.000 & 0.000 \\
                 &      & Defense    & 0.000 & 0.000 & 0.000 & 0.000 & 0.000 & 0.000 & 0.000 \\
                 & 1.00 & No Defense & 0.000 & 0.000 & 0.000 & 0.159 & 0.000 & 0.000 & 0.000 \\
                 &      & Defense    & 0.000 & 0.113 & 0.000 & 0.000 & 0.000 & 0.000 & 0.000 \\
        \midrule
        MNIST    & 0.00 & No Defense & 0.701 & 1.576 & 1.318 & 0.806 & 0.065 & 0.152 & 0.223 \\
                 &      & Defense    & 0.354 & 0.000 & 0.000 & 0.102 & 0.000 & 0.000 & 0.000 \\
                 & 0.25 & No Defense & 0.701 & 1.576 & 1.247 & 0.806 & 0.065 & 0.403 & 0.210 \\
                 &      & Defense    & 0.109 & 0.000 & 0.000 & 0.102 & 0.000 & 0.000 & 0.000 \\
                 & 0.50 & No Defense & 0.701 & 1.590 & 1.318 & 0.806 & 0.065 & 0.403 & 0.223 \\
                 &      & Defense    & 0.000 & 0.185 & 0.000 & 0.145 & 0.000 & 0.000 & 0.000 \\
                 & 0.75 & No Defense & 0.701 & 1.576 & 1.318 & 0.806 & 0.065 & 0.403 & 0.223 \\
                 &      & Defense    & 0.354 & 0.000 & 0.000 & 0.183 & 0.000 & 0.000 & 0.000 \\
                 & 1.00 & No Defense & 0.701 & 1.590 & 1.247 & 0.806 & 0.065 & 0.403 & 0.210 \\
                 &      & Defense    & 0.109 & 0.000 & 0.000 & 0.169 & 0.000 & 0.000 & 0.000 \\
        \bottomrule
    \end{tabular}
    }
    \caption{Full empirical privacy results across interpolated canary strengths ($\alpha$) on CIFAR-10 and MNIST (CNN). See Section~\ref{app:scaling}.}
    \label{tab:alpha_ablation_full}
\end{table}
 
% ============================================================
% Varying Filter Epochs Ablation
% ============================================================
\begin{table}[t]
    \centering
    \small
    \resizebox{\textwidth}{!}{
    \begin{tabular}{lccccccc}
        \toprule
        Epochs Between Filtering & GDP 25\% & GDP 50\% & GDP 75\% & CP no holdout & CP 25\% & CP 50\% & CP 75\% \\
        \midrule
        1  & 0.000 & 0.000 & 0.000 & 0.066 & 0.000 & 0.000 & 0.000 \\
        5  & 0.000 & 0.086 & 0.302 & 0.391 & 0.000 & 0.000 & 0.013 \\
        10 & 0.000 & 0.266 & 0.000 & 0.449 & 0.000 & 0.000 & 0.000 \\
        20 & 0.000 & 0.000 & 0.000 & 0.391 & 0.000 & 0.000 & 0.000 \\
        \bottomrule
    \end{tabular}
    }
    \caption{Full empirical privacy results across filter frequencies on CIFAR-10 (CNN). See Section~\ref{app:filter_frequency_ablation}.}
    \label{tab:filter_freq_full}
\end{table}
 
% ============================================================
% Varying Defense Bandwidth Ablation
% ============================================================
\begin{table}[t]
    \centering
    \small
    \resizebox{\textwidth}{!}{
    \begin{tabular}{lccccccc}
        \toprule
        Local Filtering $k$ Value & GDP 25\% & GDP 50\% & GDP 75\% & CP no holdout & CP 25\% & CP 50\% & CP 75\% \\
        \midrule
        1 & 0.000 & 0.000 & 0.000 & 0.055 & 0.000 & 0.000 & 0.000 \\
        2 & 0.000 & 0.000 & 0.000 & 0.102 & 0.000 & 0.000 & 0.000 \\
        3 & 0.000 & 0.000 & 0.000 & 0.065 & 0.000 & 0.000 & 0.000 \\
        4 & 0.000 & 0.000 & 0.000 & 0.000 & 0.000 & 0.000 & 0.000 \\
        5 & 0.000 & 0.000 & 0.000 & 0.035 & 0.000 & 0.000 & 0.000 \\
        \bottomrule
    \end{tabular}
    }
    \caption{Full empirical privacy results across local filtering bandwidth values on CIFAR-10 (CNN). See Section~\ref{app:bandwidth_ablation}.}
    \label{tab:bandwidth_full}
\end{table}
 
% ============================================================
% Defense-Aware Audit
% ============================================================
\begin{table}[t]
    \centering
    \small
    \resizebox{\textwidth}{!}{
    \begin{tabular}{llccccccc}
        \toprule
        Attack & Setting & GDP 25\% & GDP 50\% & GDP 75\% & CP no holdout & CP 25\% & CP 50\% & CP 75\% \\
        \midrule
        Mislabeled          & No Defense & 0.701 & 1.576 & 1.318 & 0.806 & 0.065 & 0.152 & 0.223 \\
                            & Defense    & 0.849 & 0.000 & 0.000 & 0.153 & 0.000 & 0.000 & 0.000 \\
        \midrule
        Defense-aware       & No Defense & 0.000 & 0.672 & 1.346 & 0.531 & 0.000 & 0.106 & 0.333 \\
                            & Defense    & 0.000 & 0.000 & 0.000 & 0.000 & 0.000 & 0.000 & 0.000 \\
        \bottomrule
    \end{tabular}
    }
    \caption{Full empirical privacy results for the defense-aware audit on MNIST (CNN), comparing standard mislabeled canaries against the fixed-point attack. See Section~\ref{ssec:defense_aware}.}
    \label{tab:defense_aware_full}
\end{table}
\iffalse
% ============================================================
% HAMP
% ============================================================

% ============================================================
% HAMP
% ============================================================
\begin{table}[t]
    \centering
    \small
    \begin{tabular}{llcccc}
        \toprule
        Method & Setting & 25\% & 50\% & 75\% & no holdout \\
        \midrule
            SGD (no defense)
            & (GDP) & 0.000 & 0.000 & 0.000 & --- \\
            & \textcolor{gray}{(CP)} & \textcolor{gray}{0.000} & \textcolor{gray}{0.000} & \textcolor{gray}{0.000} & \textcolor{gray}{0.008} \\
        \midrule
            SGD + HAMP
            & (GDP) & 0.000 & 0.000 & 0.000 & --- \\
            & \textcolor{gray}{(CP)} & \textcolor{gray}{0.000} & \textcolor{gray}{0.000} & \textcolor{gray}{0.000} & \textcolor{gray}{0.610} \\
        \midrule
            SGD + our defense
            & (GDP) & 0.000 & 0.000 & 0.000 & --- \\
            & \textcolor{gray}{(CP)} & \textcolor{gray}{0.000} & \textcolor{gray}{0.000} & \textcolor{gray}{0.000} & \textcolor{gray}{0.000} \\
        \bottomrule
    \end{tabular}
    \caption{Full empirical privacy results comparing our defense against HAMP on MNIST (CNN). See Section~\ref{sec:nonprivate}.}
    \label{tab:hamp_full}
\end{table}
 \fi
% ============================================================
% Global Filter Ablation
% ============================================================
\begin{table}[t]
    \centering
    \small
    \begin{tabular}{lccccccc}
        \toprule
        Global Filtering $k$ Value & GDP 25\% & GDP 50\% & GDP 75\% & CP no holdout & CP 25\% & CP 50\% & CP 75\% \\
        \midrule
        10 & 0.000 & 0.000 & 0.000 & 0.000 & 0.000 & 0.000 & 0.000 \\
        25 & 0.000 & 0.000 & 0.000 & 0.000 & 0.000 & 0.000 & 0.000 \\
        50 & 0.000 & 0.000 & 0.000 & 0.000 & 0.000 & 0.000 & 0.000 \\
        \bottomrule
    \end{tabular}
    \caption{Full empirical privacy results for global filtering ablation on MNIST (CNN). See Section~\ref{app:global_filter}.}
    \label{tab:global_filter_full}
\end{table}

% ============================================================
% Gradient Bandwidth
% ============================================================

\begin{table}[t]
    \centering
    \small
    \begin{tabular}{lccccccc}
        \toprule
        Setting & GDP 25\% & GDP 50\% & GDP 75\% & CP no holdout & CP 25\% & CP 50\% & CP 75\% \\
        \midrule
        No Defense & 17.625 & 24.378 & 27.130 & 4.185 & 2.749 & 3.493 & 3.903 \\
        Defense    & 17.625 & 24.378 & 25.311 & 4.185 & 2.749 & 3.493 & 3.892 \\
        \bottomrule
    \end{tabular}
    \caption{Full empirical privacy results under the gradient bandwidth attack on MNIST (CNN). See Section~\ref{ssec:gradient}.}
    \label{tab:grad_bandwidth_full}
\end{table}